\documentclass[11pt]{article}
\pdfoutput=1

\usepackage{microtype}
\usepackage{graphicx}
\usepackage{subfigure}
\usepackage{booktabs} 
\usepackage{amssymb}
\usepackage{bm}
\usepackage{bbm}
\usepackage{amsmath}  
\usepackage{amsthm}
\usepackage{xcolor}
\usepackage{textcomp}
\usepackage{multirow}
\usepackage{mathtools}
\usepackage[margin=0.75in]{geometry}
\usepackage{authblk}
\usepackage[numbers]{natbib}
\usepackage{xspace}
\usepackage{comment}
\usepackage{thmtools, thm-restate}  
\usepackage{verbatim}
\usepackage{times}

\usepackage{tikz}
\usetikzlibrary{arrows}
\usetikzlibrary{positioning}

\tikzset{
  treenode/.style = {align=center, inner sep=0pt, text centered,
    font=\sffamily},
  arn_n/.style = {treenode, circle, black, font=\sffamily\bfseries, draw=black,
    fill=white, text width=1.5em},
  arn_r/.style = {treenode, circle, black, font=\sffamily\bfseries, draw=black,
    fill=white, text width=1.0em},
  arn_x/.style = {treenode, rectangle, draw=black,
    minimum width=0.5em, minimum height=0.5em}
}

\usepackage{hyperref}

\usepackage{bm}
\usepackage{enumitem}

\newtheorem{theorem}{Theorem}
\newtheorem{lemma}{Lemma}
\newtheorem{corollary}{Corollary}

\newtheorem{definition}{Definition}

\newcommand{\argmin}[2]{\textrm{argmin}_{#1}~#2}

\newcommand{\E}[2]{\mathbb{E}_{#1}\left[#2\right]}

\newcommand{\Var}[2]{\textrm{Var}_{#1}\left[#2\right]}

\newcommand{\ind}[1]{\mathbbm{1}\left\{#1\right\}}

\newcommand{\rademacher}[0]{\mathfrak{R}}

\newcommand{\D}[0]{\mathcal{D}}

\newcommand{\X}[0]{\mathcal{X}}
\newcommand{\Y}[0]{\mathcal{Y}}
\newcommand{\p}[0]{\mathcal{P}}

\newcommand{\A}[0]{\mathcal{A}}
\newcommand{\Z}[0]{\mathcal{Z}}
\newcommand{\F}[0]{\mathcal{F}}

\newcommand{\phat}{\hat{p}}

\newcommand{\G}[0]{\mathcal{G}}

\ifdefined\ShowNotes
  \newcommand{\colornote}[3]{{\color{#1}\bf{#2 #3}\normalfont}}
\else
  \newcommand{\colornote}[3]{}
\fi

\definecolor{darkred}{rgb}{0.7,0.1,0.1}
\definecolor{darkgreen}{rgb}{0.1,0.5,0.1}
\definecolor{cyan}{rgb}{0.7,0.0,0.7}
\definecolor{dblue}{rgb}{0.2,0.2,0.8}
\definecolor{maroon}{rgb}{0.76,.13,.28}
\definecolor{burntorange}{rgb}{0.81,.33,0}
\definecolor{royalpurple}{rgb}{0.47,.31,0.66}

\newcommand{\sysname}{\textsc{Thanos}\xspace}
\newcommand{\lspread}{L_{\text{spread}}}
\newcommand{\supcon}{L_{\text{sup}}}
\newcommand{\cnce}{L_{\text{cNCE}}}

\newcommand{\lspreadhat}{\hat{L}_{\text{spread}}}
\newcommand{\supconhat}{\hat{L}_{\text{sup}}}
\newcommand{\cncehat}{\hat{L}_{\text{cNCE}}}

\usepackage{pifont}
\newcommand{\cmark}{\ding{51}}
\newcommand{\xmark}{\ding{55}}

\ifdefined\ShowNotes
  \newcommand{\num}[1]{{\color{red}\bf{#1}\normalfont}}
\else
  \newcommand{\num}[1]{#1}
\fi

\newif\ifarxiv
\arxivtrue

\begin{document}

\title{Perfectly Balanced: Improving Transfer and Robustness of Supervised Contrastive Learning}

\author[1]{Mayee~F.~Chen$^*$}
\author[1]{Daniel~Y.~Fu\thanks{Equal Contribution.}}
\author[1]{Avanika~Narayan}
\author[1]{Michael~Zhang}
\author[2]{Zhao~Song}
\author[1]{Kayvon~Fatahalian}
\author[1]{Christopher~R\'e}

\date{}

\affil[1]{Department of Computer Science, Stanford University}
\affil[2]{Adobe Research}
\affil[1]{\footnotesize{\texttt{\{mfchen, danfu, avanikan, mzhang, kayvonf, chrismre\}@cs.stanford.edu}}}
\affil[2]{\footnotesize{\texttt{zsong@adobe.com}}}

\maketitle


\begin{abstract}

An ideal learned representation should display transferability and robustness.
Supervised contrastive learning (SupCon) is a promising method for training accurate models, but produces representations that do not capture these properties due to class collapse---when all points in a class map to the same representation.
Recent work suggests that ``spreading out'' these representations improves them, but the precise mechanism is poorly understood.
We argue that creating spread alone is insufficient for better representations, since spread is invariant to permutations within classes.
Instead, both the correct degree of spread and a mechanism for breaking this invariance are necessary.
We first prove that adding a weighted class-conditional InfoNCE loss to SupCon controls the degree of spread.
Next, we study three mechanisms to break permutation invariance: using a constrained encoder, adding a class-conditional autoencoder, and using data augmentation.
We show that the latter two encourage clustering of latent subclasses under more realistic conditions than the former.
Using these insights, we show that adding a properly-weighted class-conditional InfoNCE loss and a class-conditional autoencoder to SupCon achieves \num{11.1} points of lift on coarse-to-fine transfer across \num{5} standard datasets and \num{4.7} points on worst-group robustness on \num{3} datasets, setting state-of-the-art on CelebA by \num{11.5} points.

\end{abstract}


\section{Introduction}
\label{sec:intro}

\begin{figure*}[t]
  \includegraphics[width=6.75in]{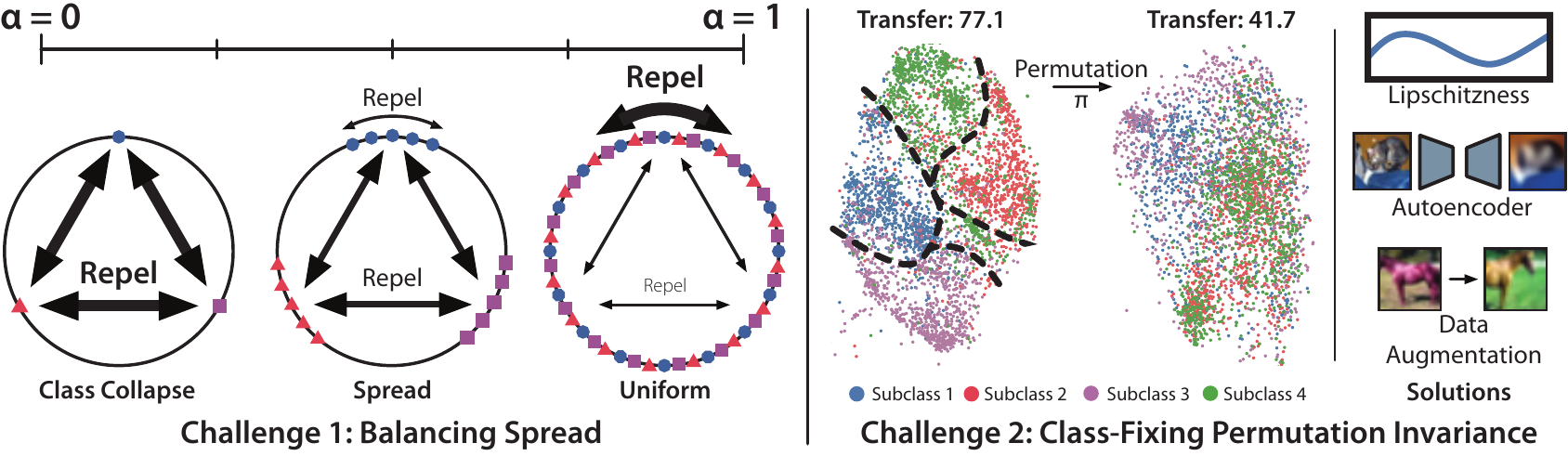}
  \centering
  \caption{
    There are two key challenges to achieving better representations with SupCon.
    Left: The first challenge is balancing multiple contrastive loss terms with competing geometries.
    We show that adding a weighted class-conditional InfoNCE term can balance the geometries and induce spread in the representation geometry.
    Right: The second challenge is that representation geometries may be invariant to class-fixing permutations.
    The two example geometries shown both have spread, but vastly different coarse-to-fine transfer performance.
    We analyze three mechanisms for addressing this challenge: constraining the encoder, adding a class-conditional autoencoder, and using data augmentation.
    Best viewed in color.
  }
  \vspace{-1.0em}
  \label{fig:banner}
\end{figure*}

Learning a representation with a favorable geometry is a critical challenge for modern machine learning.
Good geometries can engender strong downstream transfer performance and robustness to subgroup imbalances, whereas poor geometries may have low transferability and be brittle~\citep{islam2021broad, sohoni2020george}.
However, producing---or even characterizing---a good geometry can be difficult.

We focus on the challenges in doing so with supervised contrastive learning (SupCon).
SupCon is a promising method for training accurate machine learning models~\citep{khosla2020supervised}, but suffers from class collapse---wherein each point in the same class has the same representation, as in Figure~\ref{fig:banner} far left~\cite{graf2021dissecting}.
Collapsed representations cannot distinguish fine-grained details within classes---in particular \emph{latent subclasses}---resulting in poor transferability and robustness.
Modifications to SupCon that heuristically ``spread out'' its representations have shown empirical promise~\citep{islam2021broad}, but a precise understanding of spread---how separated individual points are in representation space---and how to control it is lacking.

Furthermore, spread alone is not sufficient to explain improved representations.
We observe that modifications to SupCon that increase spread are invariant to \textit{class-fixing permutations}.
That is, the loss value does not change when points of the same class are arbitrarily permuted in representation space.
For example, Figure~\ref{fig:banner} right visualizes two geometries that both have spread but differ in representation quality, as suggested by the significant gap in transfer learning performance (\num{35.4} points).
Thus, while spread may be important, another mechanism is needed to break class-fixing permutation invariance for good performance.

We argue that these are the two key challenges to improving SupCon's representations: creating the correct degree of spread, and breaking class-fixing permutation invariance.
This paper makes progress on these challenges.

\textbf{Challenge 1: Balancing Spread.}
We first prove a simple result that a class-collapsed representation cannot have good transfer performance, which motivates spread.
We then analyze whether $\lspread$, a loss function that combines SupCon with a class-conditional InfoNCE loss, can induce spread.

We find that previous approaches for analyzing contrastive losses encounter a technical challenge because SupCon and InfoNCE have incompatible optimal geometries (class collapse and uniformity on the hypersphere, respectively).
For example, \citet{wang2020understanding} analyze individual loss components in isolation, but doing so risks drawing misleading conclusions when the loss components are incompatible.
Further, finding exact solutions to optimization problems on the hypersphere is fundamentally difficult;
a classic example is the Thomson problem~\citep{thomson1897cathode}, which has evaded an exact solution after a century of study.

We bypass these problems by constructing a distribution that is neither collapsed nor uniform and analyzing its loss.
We introduce $s_f(y)$, a notion of class variance, to measure spread on this distribution. We show that this distribution has an intermediate degree of spread by deriving bounds for the weight $\alpha$ on the class-conditional InfoNCE loss within which this distribution attains lower loss than either extreme.
While this result does not fully characterize the geometry, it suggests that setting $\alpha$ properly can induce an optimal distribution with appropriate spread---which we validate with measurements on CIFAR10.

\textbf{Challenge 2: Breaking Permutation Invariance.}
Our first result demonstrates that $\lspread$ can induce spread but does not give insight on class-fixing permutation invariance.
We formally define class-fixing permutation invariance and prove that $\lspread$ is subject to it absent other interventions. 

This motivates the question: how should we break class-fixing permutation invariance?
We show that inducing an inductive bias towards clustering of latent subclasses can break permutation invariance---and more importantly, can result in good coarse-to-fine transfer performance.
We introduce $\sigma_f(z)$, a measure of subclass clustering in representation space, and show that coarse-to-fine generalization error scales with $\sigma_f(z)/s_f(y)$.

A standard approach to controlling $\sigma_f(z)$ is assuming Lipschitzness of the model.
However, Lipschitzness is a strong assumption for modern deep networks, which are powerful enough to memorize random noise~\citep{zhang2016understanding}.
In empirical measurements, we find that modern deep networks display poor Lipschitzness, and thus the Lipschitzness assumption is insufficient for inducing clustered subclass representations. 

We thus propose two alternatives that can bound $\sigma_f(z)$ under more realistic assumptions: directly encoding fine-grained details by concatenating the representations from a class-conditional autoencoder, and using data augmentation in the class-conditional InfoNCE loss.
The former only requires a ``reverse Lipschitz'' decoder to upper bound $\sigma_f(z)$, and can do so by a constant factor tighter than a general (non-conditional) autoencoder.
The latter only requires the encoder to be Lipschitz over data augmentations to induce subclass clustering---and can also explain observations from prior work~\citep{islam2021broad}.
We validate these findings by measuring Lipschitzness constants and $\sigma_f(z)/s_f(y)$ on real data; we find that these alternate assumptions are more realistic than overall Lipschitzness, and that data augmentation and autoencoders help induce subclass clustering.

\textbf{Empirical Validation}
Using our theoretical insights, we propose \sysname: adding a class-conditional InfoNCE loss and a class-conditional autoencoder to SupCon.
We evaluate \sysname\ on two tasks designed to evaluate how well it preserves subclasses:
\ifarxiv
\begin{itemize}
\else
\begin{itemize}[itemsep=0.5pt,topsep=0pt,leftmargin=*]
\fi
    \item \textbf{Coarse-to-fine transfer learning} trains a model to classify superclasses but use the representations to distinguish subclasses.
    \sysname\ outperforms SupCon by \num{11.1} points on average across \num{5} standard datasets.
    \item \textbf{Worst-group robustness} evaluates how well a model can identify underperforming sub-groups and maintain high performance on them.
    \sysname\ identifies underperforming sub-groups \num{7.7} points better than previous work~\citep{sohoni2020george} and achieves \num{4.7} points of lift on worst-group robustness across \num{3} datasets, setting state-of-the-art on CelebA by \num{11.5} points.
    \sysname\ can even outperform GroupDRO~\citep{sagawa2019groupdro}, a state-of-the-art robustness algorithm that uses \textit{ground-truth} sub-group labels.
\end{itemize}

\section{Background}
\label{sec:background}

Section~\ref{subsec:setup} presents our data model and the coarse-to-fine transfer task.
Section~\ref{subsec:contrastive_loss} presents $\lspread$, a simple variant of SupCon that adds a weighted class-conditional InfoNCE loss.
Section~\ref{subsec:background_geometry} discusses geometry of contrastive losses.

\subsection{Data Setup}\label{subsec:setup}

Input data $x \in \X$ are drawn from a distribution $\p$ with deterministic class $y = h(x)$, where $y \in \Y = \{0, \dots, K - 1\}$. We assume that the data is class-balanced such that $\Pr(y = i) = \frac{1}{K}$ for all $i \in \Y$.

Data points also belong to \emph{latent subclasses}. 
Following~\citet{sohoni2020george}, we denote a subclass as a latent discrete variable $z \in \Z$. 
$\Z$ can be partitioned into disjoint subsets $S_0, \dots, S_{K - 1}$ such that if $z \in S_k$, then its corresponding $y$ label is equal to $k$. 
For simplicity, we assume that there are two subclasses for each label $k$, e.g. $|S_k| = 2$. 
The data generating process proceeds as follows: first, the latent subclass $z$ is sampled with proportion $p(z)$. Then, $x$ is sampled from the distribution $\p_z = p(\cdot | z)$, and its corresponding deterministic label is denoted $y = S(z)$. 
Let $h_s(x): \X \rightarrow \Z$ denote $x$'s subclass.

We have a class-balanced labeled training dataset $\D = \{(x_i, y_i)\}_{i = 1}^n$ where points $(x_i, z_i, y_i)$ are drawn i.i.d, and the value of each $z_i$ is unknown during training time. Denote $\D_y = \{x \in \D: h(x) = y \}$ and $\D_z = \{x \in \D: h_s(x) = z\}$, and denote their sizes by $n_y = |\D_y|$ and $n_z = |\D_z|$.

Contrastive learning trains an encoder $f: \X \rightarrow \mathbb{R}^d$ on $\D$ that maps inputs to representations in an embedding space $\mathbb{R}^d$.

\paragraph{Coarse-to-Fine Transfer}
Coarse-to-fine transfer evaluates how well an embedding trained on coarse classes $\Y$ distinguishes fine classes (subclasses) $\Z$.
Fix a $y$ and smakeuppose that $S_y = \{z, z'\}$. 
The task is to classify $z$ versus $z'$ using the encoder $f$ learned on $\D$. 
We are given a dataset of subclass labels, $\D_{s} = \{(x_i, z_i)\}_{i = 1}^m$. Denote $\D_{s, z} = \{x \in \D_s: h_s(x) = z \}$ and $m_z = |\D_{s, z}|$.
We learn linear weights $W_z, W_{z'} \in \mathbb{R}^d$ and construct an estimate $\phat(z | f(x))$ by using softmax scores $\phat(z | f(x)) = \frac{\exp(f(x)^\top W_{h_s(x)})}{\exp(f(x)^\top W_{z}) + \exp(f(x)^\top W_{z'})}$, where $f$ is fixed. We use the mean classifier to construct $W$, following prior work~\citep{arora2019theoretical}. That is, $W_z = \frac{1}{m_z} \sum_{x \in \D_{s, z}} f(x)$, and $W_{z'}$ is similarly defined.

We evaluate the performance of coarse-to-fine transfer with a $\gamma$-margin loss, defined on a point $(x, z)$ as 
\begin{align}
    \ell_{\gamma, f}(x, z) = 1 - \ind{\phat(z | f(x)) \ge \gamma \phat(z' | f(x))}
\end{align}
for $\gamma > 1$. That is, we want the model to output the correct subclass label at least $\gamma$ times more likely than the incorrect one. Define the $\gamma$-margin generalization error on $z$ as $L_{\gamma, f}(z) = \E{x \sim \p_z}{\ell_{\gamma, f}(x, z)}$.

\subsection{A Modified Supervised Contrastive Loss}\label{subsec:contrastive_loss}

Contrastive learning trains an encoder to produce representations of the data by pushing together similar points (positive pairs) and pulling apart different points (negative pairs).
We consider $\lspread$, a weighted sum of a supervised contrastive loss $\supcon$~\cite{khosla2020supervised} and a class-conditional InfoNCE loss $\cnce$.

Let $B$ be a batch of data from $\D$.
Define $P(i, B) \texttt{=} \{x^+ \in B \backslash i: h(x^+) = h(x_i) \}$ as the points with the same label as $x_i$ and $N(i, B) = \{x^- \in B \backslash i: h(x^-) \neq h(x_i) \}$ as points with a different label. Let $a(x_i)$ be an augmentation of $x_i$, and assume that augmentations of each sample are disjoint. 
Denote $\sigma_f(x, x') \texttt{=} \exp(f(x)^\top f(x') / \tau)$ with temperature hyperparameter $\tau$.
For $\alpha \in [0, 1]$, $\lspreadhat(f, x, B)$ on $x$ belonging to $B$ is:
\begin{align*}
\lspreadhat(f, x, B) &\texttt{=} (1 - \alpha) \supconhat(f, x, B) + \alpha \cncehat(f, x, B),
\end{align*}
where
\ifarxiv
\begin{align}
\supconhat(f, x_i, B) &= -\frac{1}{|P(i, B)|} \;\;\; \sum_{\mathclap{x^+ \in P(i, B)}}\;\; \log  \frac{\sigma_f(x_i, x^+)}{ \sigma_f(x_i, x^+) + \sum_{x^- \in N(i, B)} \sigma_f(x_i, x^-)}, \label{eq:sup}  \\
\cncehat(f, x_i, B) &= -\log \frac{\sigma_f(x_i, a(x_i))}{ \sum_{x^+ \in P(i, B)} \sigma_f(x_i, x^+)}. \label{eq:cnce}
\end{align}
\else
\begin{align}
&\supconhat(f, x_i, B) = -1/|P(i, B)| \times \label{eq:sup}\\
& \quad \sum_{\mathclap{x^+ \in P(i, B)}} \log \frac{\sigma_f(x_i, x^+)}{ \sigma_f(x_i, x^+) + \sum_{x^- \in N(i, B)} \sigma_f(x_i, x^-)}, \nonumber  \\
&\cncehat(f, x_i, B) = -\log \frac{\sigma_f(x_i, a(x_i))}{ \sum_{x^+ \in P(i, B)} \sigma_f(x_i, x^+)}. \label{eq:cnce}
\end{align}
\fi
The overall loss $\lspreadhat(f, B)$ is averaged over all points in $B$. $\supcon$ is a variant of the SupCon loss~\citep{khosla2020supervised}.
$\cnce$ is a class-conditional version of the InfoNCE loss, where the positive distribution consists of augmentations and the negative distribution consists points from the same class, intuitively encouraging them to be spread apart.

\subsection{Geometries of Contrastive Losses}\label{subsec:background_geometry}

We present a series of standard theoretical assumptions for analyzing contrastive geometry, and define two important distributions---class collapse and class uniformity.

\textbf{Assumptions}
We make several standard theoretical assumptions~\citep{graf2021dissecting, wang2020understanding, robinson2020contrastive}: 1) restrict the encoder $f$'s output space to be $\mathcal{S}^{d-1}$, the unit hypersphere (i.e. normalized outputs); 2) assume that $K \le d + 1$, such that a $K-$regular simplex inscribed in $\mathcal{S}^{d-1}$ exists; 3) assume that the encoder is \emph{infinitely powerful}, meaning that any distribution on $\mathcal{S}^{d-1}$ is realizable by $f(x)$.
We define the pushforward measure of the class-conditional distribution of $p(\cdot | h(x) = y)$ via $f$ as $\mu_y$ for $y \in \Y$, where $\mu_y \in \mathcal{M}(\mathcal{S}^{d-1})$ is over all Borel probability measures on the hypersphere. 
Define $\bm{\mu} = \{ \mu_y\}_{y \in \Y}$ as the overall pushforward measure corresponding to $\p \circ f^{-1} \in \mathcal{M}(\mathcal{S}^{d-1})$.

\textbf{Class Collapse Distribution}
Define $\bm{v} = \{v_y\}_{y \in \Y} \in \mathcal{S}^{d-1}$ as the set of vectors forming the regular simplex inscribed in the hypersphere, satisfying: a) $\sum_y v_y =\vec{0}$; b) $\|v_y\|_2 = 1 \; \forall y$; and c) $\exists \; c_K \in \mathbb{R}$ s.t. $v_y^\top v_{y'} = c_K$ for $y \neq y'$. Let $\delta_{v_y}$ be the probability measure on $\mathcal{S}^{d-1}$ with all mass on $v_y$, and let $\bm{\delta_v} = \{\delta_{v_y}\}_{y \in \Y}$ be the \emph{class-collapsed measure} such that $\mu_y = \delta_{v_y}$ and $f(x) = v_y$ almost surely whenever $h(x) = y$.
\citet{graf2021dissecting} show that $\bm{\mu} = \bm{\delta_v}$ minimizes the SupCon loss.

\textbf{Class Uniform Distribution}
Denote $\sigma_{d-1}$ as the normalized surface area measure on $\mathcal{S}^{d-1}$. $\bm{\mu} = \bm{\sigma_{d-1}}$ is the \emph{class-uniform measure} when $\mu_y = \sigma_{d-1}$ for all $y \in \Y$.
\citet{wang2020understanding} show that $\sigma_{d-1}$ minimizes the InfoNCE loss.


\section{Controlling Spread}
\label{sec:spread}

In Section~\ref{subsec:supconc2f}, we demonstrate the importance of spread---having distinguishable representations of points in a class---by showing that SupCon results in poor coarse-to-fine transfer.
In Section~\ref{subsec:asymptotic}, we begin to explore whether $\lspread$ can result in more spread out geometries. We define the asymptotic form of $\lspread$ and apply the approach from~\citet{wang2020understanding} to analyze individual loss terms.
We find that the optimal geometries of the individual terms are incompatible.
In Section~\ref{subsec:spread_balance}, we analyze the asymptotic $\lspread$ as a whole using a nuanced approach that compares the loss over different geometries. We conclude that the optimal geometry is neither class-collapsed nor class-uniform for a range of $\alpha$. This result suggests that spread can be carefully controlled, and we capture this property by introducing a notion of intra-class variance, $s_f(y)$. All proofs for the paper are in Appendix~\ref{sec:app_proofs}.

\subsection{The Importance of Spread}\label{subsec:supconc2f}

SupCon exhibits class collapse on the training data and does not spread out representations in a class. We use standard generalization bounds and show that this geometry results in poor coarse-to-fine generalization error: asymptotically, the error obtains its maximum possible value.

Define $f_{SC} \in \F$ to be the encoder trained with SupCon satisfying class collapse, $f_{SC}(x) = v_y$ for all $x \in \D$ where $h(x) = y$.  
Let $f(x)[j]$ be the $j$th entry of $f(x)$. For function class $\F$, let $\F_j = \{f(\cdot)[j]: f\in \F\}$ be the elementwise class. Let $\rademacher_n(\F_j)$ denote $\F_j$'s Rademacher complexity on $n$ samples, and define $\rademacher_n(\F) = \sum_{j = 1}^d \rademacher_n(\F_j)$. 
\begin{theorem}\label{thm:supcongenerr}
For $\gamma$ where $\log \gamma \ge 8 \max\{\rademacher_{n_z}(\F), \rademacher_{n_{z'}}(\F)\}$, SupCon's coarse-to-fine error is at least 
\begin{align*}
L_{\gamma, f_{SC}}(z) &\ge 1 - \delta(n_z, \F, \gamma) - \delta(n_{z'}, \F, \gamma) - \xi(m_z \wedge m_{z'}, \gamma ), \nonumber 
\end{align*}
where $\delta(n_z, \F, \gamma) = d \exp \Big(-\frac{n_z}{32d^2} (\log \gamma - 8 \rademacher_{n_z}(\F))^2 \Big)$ bounds generalization error of $f_{SC}$ and $\xi(m_z \wedge m_{z'}, \gamma) = 4d \exp\Big(-\frac{(m_z \wedge m_{z'}) \log^2 \gamma}{32d} \Big)$ bounds the noise from $\D_s$.
\end{theorem}

As $n\wedge m$ increases, error approaches $1$---its maximum value---and the model will almost surely predict the correct subclass $\gamma$ times less often than the incorrect one.
This result motivates studying whether $\lspread$ can encourage spread.

\subsection{Asymptotic $\lspread$}\label{subsec:asymptotic}

We present the asymptotic version of $\lspread$.
For a given anchor $x \sim \p$, define a positive pair $x^+ \sim p(\cdot | h(x^+) = h(x))$ from the same class and a negative pair using $x^- \sim p(\cdot | h(x^-) \neq h(x))$ from a different class. 
Let $a(x)$ be an augmentation of $x$ drawn from a distribution $p_a(\cdot | x)$, where each $p_a(\cdot | x)$ has disjoint support.

\begin{definition}
Define $\lspread(f, \alpha)$ as
\ifarxiv
\begin{align}
\lspread(f, \alpha) &= (1 - \alpha) L_{\text{align}}(f) + \alpha L_{\text{aug}}(f) + (1 - \alpha) L_{\text{diff}}(f) + \alpha L_{\text{same}}(f), \nonumber
\end{align}
\else
\begin{align}
\lspread(f, \alpha) &= (1 - \alpha) L_{\text{align}}(f) + \alpha L_{\text{aug}}(f) \\
&+ (1 - \alpha) L_{\text{diff}}(f) + \alpha L_{\text{same}}(f), \nonumber
\end{align}
\fi

where 
\begin{align*}
L_{\text{align}}(f) &=  \E{x, x^+}{\|f(x) - f(x^+) \|^2 / 2\tau} \\
L_{\text{aug}}(f) &= \E{x, a(x)}{\|f(x) - f(a(x))\|^2 / 2\tau} \\
L_{\text{diff}}(f) &=  \E{x}{\log \E{x^-}{ \exp(-\|f(x) - f(x^-)\|^2 /2\tau)}} \\
L_{\text{same}}(f) &= \E{x}{\log \E{x^+}{\exp(-\|f(x) - f(x^+)\|^2/ 2\tau)}}
\end{align*}

\label{def:asymptotic}
\end{definition}

The derivation of $\lspread(f, \alpha)$ is in Appendix~\ref{subsec:app_spread_proofs}. Next, we analyze individual terms, similar to~\citet{wang2020understanding}'s approach.
For simplicity, we present the binary setting $K = 2$. We abuse notation and use $f$ and $\bm{\mu}$, the pushforward measure of $x$ on the hypersphere, interchangeably in $\lspread(f, \alpha)$ as well as in the loss components in Definition~\ref{def:asymptotic}.

\begin{restatable}[Individual losses]{proposition}{termwise}
$L_{\text{align}}(f)$ and $L_{\text{aug}}(f)$ are minimized when $f(x) = f(x^+)$ and $f(x) = f(a(x))$ almost surely, respectively. $L_{\text{diff}}(\bm{\mu})$ is minimized when $\bm{\mu} = \bm{\delta_v}$. $L_{\text{same}}(\bm{\mu})$ is minimized when $\bm{\mu} = \bm{\sigma_{d-1}}$.

\label{prop:termwise}
\end{restatable}

When $\alpha = 0$, the ``active'' loss terms are $L_{\text{align}}$ and $L_{\text{diff}}$, whose optima are jointly realizable and yield $\bm{\mu} = \bm{\delta_v}$ overall. When $\alpha = 1$, the terms $L_{\text{aug}}$ and $L_{\text{same}}$ are also compatible, yielding $\bm{\mu} = \bm{\sigma_{d-1}}$ and augmentations with the same embedding as their original point.

Neither of these distributions has good coarse-to-fine transfer performance on its own: $\bm{\delta_v}$ loses information within classes, and $\bm{\sigma_{d-1}}$ allows points of different classes to be close together (Figure~\ref{fig:banner} left).
To avoid both $\bm{\delta_v}$ and $\bm{\sigma_{d-1}}$, $\alpha \in (0, 1)$ must achieve a balance between the two loss terms.
But the behavior of the weighted loss overall is unclear from the result in Proposition~\ref{prop:termwise}.
It is also unclear whether there even exists an intermediate distribution that minimizes $\lspread(\bm{\mu}, \alpha)$.

\subsection{Our Spread Result}
\label{subsec:spread_balance}

We seek to analyze the geometry of the overall loss.
Explicitly characterizing the optimal geometry is challenging, so we design a family of measures on the hypersphere and examine when such measures obtain lower loss than collapsed or uniform measures. We perform analysis for $K = 2$ and consider $K = 3$ in Appendix~\ref{sec:supp_additional_theory}. Synthetic experiments are in Appendix~\ref{sec:supp_synthetics}.

The measure we study, $\bm{\mu_{\theta}}$, assigns mass evenly on two points that are close to $v_y$, a vertex of the regular simplex, but separated by some angle $\theta$ for each $\mu_y$ (see Figure~\ref{fig:synthetic_fig} in Appendix~\ref{sec:supp_synthetics}).
Formally, define a block-diagonal rotation matrix $R_{\theta} \in \mathbb{R}^{d \times d}$ consisting of submatrices $\begin{bsmallmatrix} \cos \theta & - \sin \theta \\ \sin \theta & \cos \theta\end{bsmallmatrix}$ and $I_{d-2}$ on the diagonal. For $\theta \in (0, \pi/2]$, define the measure $\bm{\mu_{\theta}} = \{\mu_{0, \theta}, \mu_{1, \theta}\}$, where $\mu_{0, \theta} = \frac{1}{2} \delta_{R_{\theta} v_0} + \frac{1}{2} \delta_{R_{-\theta} v_0}$, and similarly  $\mu_{1, \theta} = \frac{1}{2} \delta_{R_{\theta} v_1} + \frac{1}{2} \delta_{R_{-\theta} v_1}$. We present a technical result on the range of $\alpha$ for which $\bm{\mu_{\theta}}$ attains lower loss than class-collapsed or class-uniform measures.

\begin{figure}[t]
    \centering
    \ifarxiv
        \includegraphics[width=5in]{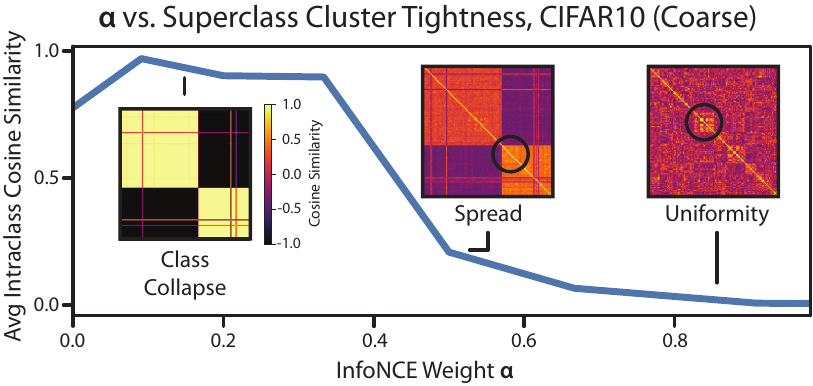}
    \else
        \includegraphics[width=3.25in]{figs/alpha_sensitivity_graph}
    \fi
    \caption{
        Measure of cluster tightness vs. $\alpha$.
        Clusters are collapsed for low values of $\alpha$, display spread for a small region, and then dissolve into uniformity for high values of $\alpha$.
        Inserts: heatmaps of cosine similarity between points, sorted by class and subclass.
        Circles: apparent subclass clusters within spread.
    }
    \label{fig:alpha_sensitivity}
\end{figure}

\begin{theorem}
 Let $c_{\tau, d} \texttt{=}  \frac{2 + \frac{1}{\tau} - \sqrt{\frac{1}{\tau}(-2 + \frac{1}{\tau}) - 2 \log W_{1/2\tau}(\mathcal{S}^{d-1})}}{3}$, where $W_{1/2\tau}(\mathcal{S}^{d-1})$ is a constant depending on $\tau$ and $d$ (see Appendix~\ref{subsec:app_spread_proofs} for exact value). Then, when $\alpha \in (2/3, c_{\tau, d})$, $\theta^\star = \arcsin \sqrt{\frac{\tau}{2} \log \frac{3\alpha - 1}{3 - 3\alpha}}$ minimizes $\lspread(\bm{\mu_{\theta}}, \alpha)$ and satisfies $\lspread(\bm{\mu_{\theta^\star}}, \alpha) \le \min_{\bm{\mu} \in \{\bm{\delta_v}, \bm{\sigma_{d-1}} \}} \lspread(\bm{\mu}, \alpha)$.
\label{thm:geometry}
\end{theorem}

Our result does not define the exact optimal geometry since it constrains the measures we optimize to be over $\bm{\mu_{\theta}}$.
For $\alpha \notin (2/3, c_{\tau, d})$, it also does not specify the optimal geometry---we only know that the optimal geometry is not  of form
$\bm{\mu_{\theta}}$.

However, our result yields a high-level insight: there exists a range of $\alpha$ for which the optimal geometry that minimizes $\lspread(\bm{\mu}, \alpha)$
spreads out points on the hypersphere. 
Concretely, define the spread of class $y$ under $f$ as $s_f(y) = \E{h(x) = y}{\|f(x) - \E{h(x) = y}{f(x)} \|}$. 
\begin{corollary}
If $\alpha \in (2/3, c_{\tau, d})$ and $f(x)$ has measure $\mu_{\theta^\star}$, the spread for $y$ under $f$ is $s_f(y) = \sqrt{\frac{\tau}{2} \log \frac{3\alpha - 1}{3 - 3\alpha}} \sim \omega(1)$.
\label{cor:mu_theta_spread}
\end{corollary}
In other words, $\lspread$ can yield an extent of spread $s_f(y)$ that is controlled by $\alpha$.
Experiments on CIFAR10 support our result (Figure~\ref{fig:alpha_sensitivity}); the geometry is collapsed for low values of $\alpha$, followed by a region of spread, followed by uniformity.

Finally, we remark on two deliberate aspects of our analysis. 
First, to avoid issues of non-convexity, we directly compare the overall loss of our measures with those of the two extrema $\bm{\delta_v}$ and $\bm{\sigma_{d-1}}$.
Second, general distributions beyond the regular simplex and normalized surface measure are hard to compute contrastive losses over, and such computations are not largely studied to the extent of our knowledge. 
This inherently restricts analysis to simple distributions like $\bm{\mu_{\theta}}$.


\section{Breaking Permutation Invariance}
\label{sec:theory}

Our analysis in the previous section shows that $\lspread$ can obtain an optimal geometry that is neither collapsed nor uniform.
However, this result does not completely explain improved transfer performance because $\lspread$ under the previous setup is \emph{class-fixing permutation invariant}, a property we define in Section~\ref{subsec:permutation}.
Inducing an inductive bias can break such an invariance.
We argue that an inductive bias that encourages clustering of latent subclasses can be particularly useful.
In Section~\ref{subsec:ib_c2f}, we show that generalization error on coarse-to-fine transfer learning depends on both $s_f(y)$ and a notion of subclass clustering $\sigma_f(z)$.
We thus discuss three approaches for controlling $\sigma_f(z)$: one standard, and two alternatives with more realistic assumptions (Section~\ref{subsec:ib}).

\subsection{Class-Fixing Permutation Invariance}\label{subsec:permutation}

First, we define class-fixing permutation invariance.

\begin{definition}[Class-Fixing Permutation Invariance]
Let $\mathcal{F}$ be a class of encoders.
Let $L(f, B)$ be a loss function over an encoder $f \in \mathcal{F}$ and a set of $n$ points $B = \{x_1, \dots, x_n\}$.
Define $S_{h,B}$ as the set of class-fixing permutations such that $\pi \in S_{h,B}: [n] \rightarrow [n]$ satisfies $h(x_{\pi(i)}) = h(x_i)$ for all $i \in [n]$.
Then, $L$ is invariant on class-fixing permutations under $\mathcal{F}$ if, for any batch $B$, permutation $\pi \in S_{h, B}$, and encoder $f \in \mathcal{F}$, there exists another encoder $f^{\pi} \in \mathcal{F}$ such that $f^\pi(x_i) = f(x_{\pi(i)})$ for all $i \in [n]$ and $L(f, B) = L(f^{\pi}, B)$.

\end{definition}

We find that $\lspread$ is invariant on class-fixing permutations under the infinite encoder assumption from Section~\ref{subsec:background_geometry}.
\begin{restatable}[]{proposition}{infiniteencoder}
Let $\mathcal{F}$ be the set of infinite encoders. Then $\lspread$ is invariant on class-fixing permutations under $\mathcal{F}$.
\end{restatable}

Under class-fixing permutation invariance, data points can be arbitrarily mapped to representations within their classes, suggesting that the mapping that minimizes $\lspread$ is not unique.
However, not all these mappings achieve the same performance on downstream tasks.
Therefore, while our result from Section~\ref{sec:spread} provides insight into $\lspread$'s geometry under an infinitely powerful encoder, it cannot completely explain representation quality.

\subsection{Inductive Bias for Improved Coarse-to-fine Transfer}\label{subsec:ib_c2f}

Inducing an inductive bias can break permutation invariance (see Lemma~\ref{lemma:perm_invariance} in Appendix~\ref{sec:supp_additional_theory} for a simple proof of how smoothness of $f$ is a sufficient condition for breaking invariance).
We argue that inducing subclass clustering can be particularly helpful for transfer performance.
We measure subclass clustering in embedding space via the expected distance to the center of the subclass, $\sigma_f(z) = \E{x \sim \p_z}{\|f(x) - \E{x \sim \p_z}{f(x)}\|}$. We show that this quantity $\sigma_f(z)$, along with degree of spread $s_f(y)$, is critical for the generalization error of coarse-to-fine transfer.

To present our result on coarse-to-fine generalization error, we define some additional terms. Let $y$ denote the class label corresponding to $z, z'$. Define the quantity $\delta_f(z, z') = \frac{1}{p(z | y) p(z' | y)} \big(s_f(y) - p(z | y)^2 \sigma_f(z) - p(z' | y)^2\sigma_f(z') \big)$ as a notion of separation between $z$ and $z'$. $\delta_f(z, z')$ is large when there is spread (large $s_f(y)$) and sufficient subclass clustering (low $\sigma_f(z), \sigma_f(z')$). Define the variance of a subclass as $\Var{f}{z} = \E{x \sim \p_z}{\|f(x) - \E{x \sim \p_z}{f(x)} \|^2}$. 
We assume that for all $x \sim \p_z$, there exists a $c > 0$ such that $\|f(x) - \E{x\sim \p_{z'}}{f(x)} \| \ge c \cdot \mathbb{E}_{x \sim \p_z, x' \sim \p_{z'}}[\|f(x) - f(x') \|]$ (i.e., no point from $z$ is equal to the center of $z'$).
\begin{theorem}\label{thm:coarsetofine}
Denote $r_{f}(z, z') = c^2 \delta_f(z, z')^2 - |\Var{f}{z} - \Var{f}{z'}|$. With probability $1 - \delta$, the coarse-to-fine error is at most
\begin{align*}
L_{\gamma, f}(z) &\le \frac{\sigma_f(z)}{ \sqrt{r_f(z, z') - 2 \log \gamma}} +  \mathcal{O} \Big(\Big(\frac{d \log(d / \delta)}{m_z \wedge m_{z'}}\Big)^{1/4} \Big).
\end{align*}

under the boundary condition that $r_f(z, z') - 2 \log \gamma \ge 16 \sqrt{\frac{2d \log(8d/\delta)}{m_z \wedge m_{z'}}} + \frac{2d \log(8d/\delta)}{m_z}$.

\end{theorem}

The generalization error depends on the sampling error, $\gamma$, and three quantities intrinsic to the distribution of $f(x)$:
\ifarxiv 
\begin{itemize}
\else
\begin{itemize}[itemsep=0.5pt,topsep=0pt,leftmargin=*]
\fi
\item $s_f(y)$: the bound scales inversely in $s_f(y)$; points in a class must be spread out in order for subclasses to be distinguishable. Corollary~\ref{cor:mu_theta_spread} and empirical measurements (Figure~\ref{fig:alpha_sensitivity}) suggest that spread is non-zero when using $\lspread$ with $\alpha$ set properly. Note that under SupCon, $s_f(y)$ is asymptotically equal to $0$ and this bound is vacuous (refer to Theorem~\ref{thm:supcongenerr} for SupCon's generalization error).
\item $\sigma_f(z)$: the bound scales in $\sigma_f(z)$, confirming that spread alone is insufficient. Subclasses also need to be clustered tightly to achieve good transfer performance.
\item $|\Var{f}{z} - \Var{f}{z'}|$: distinguishing $z$ versus $z'$ may be difficult when only one subclass is clustered. When both $\sigma_f(z)$ and $\sigma_f(z')$ are small, this quantity is negligible.
\end{itemize}

Altogether, the generalization error scales in $\frac{\sigma_f(z)}{s_f(y)}$. 
Therefore, in addition to having sufficient spread $s_f(y)$, it is critical that $\sigma_f(z)$ is bounded.
We thus explore techniques for inducing an inductive bias that can control this quantity.

\subsection{Techniques for Inducing Subclass Clustering}\label{subsec:ib}


\begin{table*}[t]
    \centering
    \caption{
        Three mechanisms for inductive bias and empirical measurements of their associated Lipschitzness constants.
        Higher $K$ is a worse Lipschitzness constant, which suggests the assumptions are less realistic
    }
    \begin{tabular}{llcllccc}
        \toprule
        \textbf{Mechanism}         & \textbf{Assumptions} & \textbf{Lipschitzness Constant} \\
        \midrule
        \textbf{Encoder (Constrained)} & Lipschitz $f$ & $K_L=0.058$ \\
        \textbf{Autoencoder}       & Decoder $g$ reverse Lipschitz & $K_g=0.041$ \\
        \textbf{Augmentations}     & Lipschitz $f$ on augmentations & $K_{aug}=0.040$ \\
        \bottomrule
    \end{tabular}
    \label{table:inducing_bias}
\end{table*}

We analyze three mechanisms on $f$ for inducing an inductive bias that can cluster subclasses: a constrained encoder, a class-conditional autoencoder, and data augmentations.
These three mechanisms use Lipschitzness assumptions of varying strength to bound $\sigma_f(z)$.
We assume that subclasses are ``clustered'' in input space; i.e. there exists some $\sigma_z$ such that $\E{x, x' \sim \p_z}{\|x - x' \|} \le \sigma_z$, and we study how these mechanisms on $f$ allow us to control $\sigma_f(z)$ in terms of $\sigma_z$.
For each mechanism, we show that $\sigma_f(z) \sim K \sigma_z$ for some particular Lipschitzness constant $K$.
The lower the Lipschitzness constant, the better each mechanism can induce subclass clustering.

We summarize the assumptions of each mechanism in Section~\ref{subsec:takeaways} and report empirical estimates of Lipschitz constants in Table~\ref{table:inducing_bias}.
Section~\ref{subsec:takeaways} also reports estimates of $\frac{\sigma_f(z)}{s_f(y)}$, the ratio that governs the generalization error in Theorem~\ref{thm:coarsetofine}, showing how these mechanisms impact this quantity.

\subsubsection{Lipschitz Encoder}
One common method for incorporating inductive bias is to suppose that the class of encoders $\mathcal{F}$ is Lipschitz smooth.
We show that assuming $\mathcal{F}$ to be the class of $K_L-$Lipschitz encoders can explain subclass clustering of representations.
\begin{restatable}[]{lemma}{lipschitz}\label{lemma:lipschitz_encoder}
    Let $\mathcal{F}_{K_L}$ be the class of $K_L-$Lipschitz encoders.
    Then for any $f_{K_L} \in \mathcal{F}_{K_L}$, $\sigma_{f_{K_L}}(z) \le K_L \sigma_z.$
\end{restatable}

Lipschitzness with a sufficiently low constant $K_L$ is realistic for simple function classes, such as MLPs with bounded norms. However, modern deep networks are not Lipschitz, as they are powerful enough to memorize random noise~\citep{zhang2016understanding}.
In Table~\ref{table:inducing_bias} we confirm that the Lipschitz constant estimated empirically from our model's encoder on real data is relatively high. Therefore, since encoders with deep architectures are not Lipschitz, we consider other more realistic setups that can encourage subclass clustering.

\subsubsection{Class-Conditional Autoencoder}

To encourage embeddings to preserve properties of the input space without assuming Lipschitzness over the encoder, we propose concatenating embeddings from separate ``class-conditional'' autoencoders, each consisting of an encoder $f_{AE} \in \F_{AE}$ and a decoder $g \in \G$, to the embeddings learned from $\lspread$. An autoencoder for class $y$ aims to minimize the class reconstruction loss $\hat{L}_{AE}(\D_y) = \frac{1}{n_y} \sum_{x \in \D_y} \|g(f_{AE}(x)) - x \|^2$.
These $K$ per-class autoencoders thus intuitively learn distinctions within classes.

Define a notion of Rademacher complexity $\rademacher_n^p(\F_1, \F_2) = \E{\bm{\sigma}}{\sup_{f_1, f_2 \in \F_1, \F_2} \frac{1}{n} \sum_{i = 1}^n \sigma_i \|f_1(x_i) - f_2(x_i) \|^p}$ for Rademacher random variables $\bm{\sigma} = \{\sigma_1, \dots, \sigma_n \}$.
\begin{lemma}\label{lemma:autoencoder}
For any $g \in \G$, suppose there exists a $K_g > 0$ such that $g$ is ``reverse Lipschitz'', satisfying $\|f_{AE}(x) - f_{AE}(x') \| \le K_g \|g(f_{AE}(x)) - g(f_{AE}(x'))\|$, and there exists finite $b$ such that the reconstruction loss satisfies $\max_x \|g(f_{AE}(x)) - x \|^2 \le b$. 

Then with probability at least $1 - \delta$,
\ifarxiv
\begin{align}
    \sigma_{f_{AE}}(z)  &\le \frac{2K_g}{p(z | y)}\bigg(\hat{L}_{AE}(\D_y) + 2\rademacher_{n_y}^2(\G \circ \F_{AE}, \text{id}_{\X}) + b \sqrt{\log (1/ \delta)/2n_y} \bigg)^{1/2}  + K_g \sigma_z,
    \end{align}
\else
\begin{align}
\sigma_{f_{AE}}(z)  &\le \frac{2K_g}{p(z | y)}\bigg(\hat{L}_{AE}(\D_y) + 2\rademacher_{n_y}^2(\G \circ \F_{AE}, \text{id}_{\X}) \nonumber \\
&+ b \sqrt{\log (1/ \delta)/2n_y} \bigg)^{1/2}  + K_g \sigma_z,
\end{align}
\fi
where $\text{id}_{\X}$ is the identity function on $\X$, and $p(z | y) = \frac{p(z)}{p(y)}$ is the probability that $x$ drawn from $p(\cdot | y)$ has label $z$.
\end{lemma}

There are no explicit assumptions on $f_{AE}$; instead, a condition on the decoder is used for clustering subclasses. In Appendix~\ref{subsec:supp_theory_theory}, we show that for an autoencoder trained on $\D$ instead of $\D_y$, $p(z)$ replaces $p(z| y)$, and $n$ replaces $n_y$. That is, while a general autoencoder is learned on more data, individual subclasses comprise a smaller proportion of the data and thus could be harder to learn meaningful representations of.
This suggests that $\sigma_{f_{AE}}(z)$ is roughly a constant factor larger with a general autoencoder when $n_y$ and $n$ are both large, and thus a class-conditional autoencoder can better cluster subclasses. 

\subsubsection{Data Augmentation}

Another way of inducing inductive bias for subclass clustering is data augmentations, which we use in $\lspread$ and which play a prominent role in contrastive learning overall.
Define $\A: \X \rightarrow \X$ as the function class of augmentations and $\F_{aug}$ as the class of encoders.
\begin{lemma}\label{lemma:augmentation}
For $a \in \A$ and any $x, x' \in \X$, suppose that $f_{aug} \in \F_{aug}$ satisfies $\|f_{aug}(a(x)) - f_{aug}(a(x'))\| \le K_{aug} \|a(x) - a(x') \|$ for some $K_{aug}$ and that $f(a(x)) = f(x)$ for $x \in \D$. Denote $\sigma_z^{aug} = \E{x, x' \sim \p_z}{\|a(x) - a(x') \|}$. Then 
with probability at least $1 - \delta$,
\ifarxiv
\begin{align*}
    \sigma_{f_{aug}}(z) &\le \frac{2}{p(z)} \Big(2 \rademacher_{n}^1(\F_{aug}, \F_{aug} \circ \A) + \sqrt{2 \log(1 /\delta)/n} \Big) + K_{aug} \sigma_z^{aug}.
\end{align*}
\else
\begin{align*}
    \sigma_{f_{aug}}(z) &\le \frac{2}{p(z)} \Big(2 \rademacher_{n}^1(\F_{aug}, \F_{aug} \circ \A) + \sqrt{2 \log(1 /\delta)/n} \Big) \nonumber \\
    &+ K_{aug} \sigma_z^{aug}.
\end{align*}
\fi

\end{lemma}
Our result assumes Lipschitzness \emph{only on the augmentations}, which is consistent with literature such as~\citet{dao2019kernel}, and that the model can align augmented and original training data pairs.
$\sigma_f(z)$ scales with how close augmentations of points within a subclass are, $\sigma_z^{aug}$. 
This quantity can actually be less than $\sigma_z$ under assumptions in prior work on characterizing augmentations~\citep{huang2021generalization}, which results in tighter embedding clusters. 
Our result can also explain why prior works~\citep{islam2021broad} observe that modified losses (that include augmentations) result in better transfer. 
The empirical findings in Figure~\ref{fig:banner} (right), where the subclass embedding visualization are $\lspread$ with and without augmentations, support this result.

\subsubsection{Overall Takeaways}\label{subsec:takeaways}
Our results from Lemmas~\ref{lemma:lipschitz_encoder},~\ref{lemma:autoencoder}, and~\ref{lemma:augmentation} show that $\sigma_f(z)$, which is critical for transfer performance as demonstrated in Theorem~\ref{thm:coarsetofine}, can be controlled.
Table~\ref{table:inducing_bias} summarizes our results on how a standard encoder, an autoencoder, and data augmentations can encourage subclass clustering under various assumptions. 
We report empirical measures of $K_L, K_g$ and $K_{aug}$ on real datasets, and find that the autoencoder and data augmentation assumptions are more realistic (lower values of $K$).

Figure~\ref{fig:alpha_sensitivity} demonstrates these effects on real data; apparent clusters begin forming under $\lspread$ (which is trained with data augmentation).
We also measure the ratio $\frac{\sigma_f(z)}{s_f(y)}$ and find that it can range as high as \num{1.94} for SupCon.
For $\lspread$ with augmentations and the autoencoder, the maximum values are \num{1.05} and \num{1.03}, respectively---suggesting that these modifications control subclass clustering better, and should result in better coarse-to-fine transfer.


\section{Experiments}
\label{sec:validation}


\begin{table}[t]
     \centering
     \caption{
        Summary of the datasets we use for evaluation.
    }
    \ifarxiv
        \normalsize
    \else
        \scriptsize
    \fi
     \begin{tabular}{lccllccc}
         \toprule
         \textbf{Dataset}           & $K_{\text{coarse}}$ & $K_{\text{fine}}$ & \textbf{Notes}       \\ \midrule
         \textbf{CIFAR10}           & 2 & 10 & Coarse labels are animal vs. vehicle \\
         \textbf{CIFAR100}          & 20 & 100 & Standard coarse labels \\
         \textbf{CIFAR100-U}        & 20 & 100 & CIFAR100, imbalanced fine classes \\
         \textbf{MNIST}             & 2 & 10 & Coarse labels are $<5$ and $\geq 5$  \\
         \textbf{TinyImageNet}      & 67  & 200  & Coarse labels from ImageNet hierarchy \\
         \textbf{Waterbirds}        & 2 & 3 & Bird images~\cite{sagawa2019groupdro} \\
         \textbf{ISIC}              & 2 & 3 & Skin lesions~\cite{codella2019skin} \\
         \textbf{CelebA}            & 2 & 3 & Celebrity faces~\cite{liu2015faceattributes} \\
         \bottomrule
     \end{tabular}
\label{table:datasets}
\end{table}


\begin{table*}[t]
  \centering
  \caption{Coarse-to-fine transfer learning performance. Best in bold.}
  \small
  \begin{tabular}{llcccccccccc}
    \toprule
     & \textbf{Method} & \textbf{CIFAR10} & \textbf{CIFAR100} & \textbf{CIFAR100-U} & \textbf{MNIST} & \textbf{TinyImageNet} \\
    \midrule
    \multirow{3}{*}{\rotatebox[origin=c]{90}{\scriptsize{Baselines}}}
    & InfoNCE~\citep{chen2020simple} 
                       & 77.6 $\pm$ 0.1   & 60.5 $\pm$ 0.1    & 56.4 $\pm$ 0.3      & 98.4 $\pm$ 0.1  & 44.9 $\pm$ 0.1 \\
    & SupCon~\citep{khosla2020supervised} 
                       & 51.8 $\pm$ 1.2   & 56.1 $\pm$ 0.1    & 49.8 $\pm$ 0.3      & 95.4 $\pm$ 0.1  & 43.9 $\pm$ 0.1 \\
    & SupCon + InfoNCE~\citep{islam2021broad}
                       & 77.6 $\pm$ 0.1   & 55.7 $\pm$ 0.1    & 48.0 $\pm$ 0.2      & 98.6 $\pm$ 0.1  & 46.1 $\pm$ 0.1 \\
    \cmidrule(lr){2-7}
    \multirow{4}{*}{\rotatebox[origin=c]{90}{\scriptsize{Ours}}}
    & cAuto            & 71.4 $\pm$ 0.1   & 62.9 $\pm$ 0.1    & 58.7 $\pm$ 0.5      & 98.7 $\pm$ 0.1  & 47.1 $\pm$ 0.1 \\
    & SupCon + cNCE ($\lspread$)
                       & 77.1 $\pm$ 0.1   & 58.7 $\pm$ 0.2    & 53.5 $\pm$ 0.4      & 98.5 $\pm$ 0.1  & 45.8 $\pm$ 0.1 \\
    & SupCon + cAuto  & 71.7 $\pm$ 0.1   & 63.8 $\pm$ 0.6    & \textbf{59.8 $\pm$ 0.3} & 98.7 $\pm$ 0.1 & 49.3 $\pm$ 0.1  \\
    & SupCon + cNCE + cAuto (\textbf{\sysname})
                    & \textbf{79.1 $\pm$ 0.2}  & \textbf{65.0 $\pm$ 0.2} & 59.7 $\pm$ 0.3 & \textbf{99.0 $\pm$ 0.1} & \textbf{49.6 $\pm$ 0.1} \\                     
    \bottomrule
  \end{tabular}
  \label{table:transfer}
\end{table*}

In this section, we evaluate how well adding a class-conditional InfoNCE loss and a class-conditional autoencoder improves the representations produced by supervised contrastive learning.
We call our overall method \sysname.
This section is primarily designed to evaluate two claims:
\ifarxiv
\begin{itemize}
\else
\begin{itemize}[itemsep=0.5pt,topsep=0pt,leftmargin=*]
\fi
    \item We use coarse-to-fine transfer learning to evaluate how well the representations maintain subclass information.
    \sysname\ achieves \num{11.1} lift on average across \num{five} datasets.

    \item We evaluate how well \sysname\ can improve worst-group robustness in the unlabeled setting.
    \sysname\ detects low-performing sub-groups \num{6.2} points better than SupCon across \num{three} datasets.
    \sysname\ sets state-of-the-art worst-group robustness without sub-group labels by \num{11.5} points on CelebA---and even outperforms an algorithm that has access to ground-truth sub-group labels in some cases.
\end{itemize}
We also present ablations.
Additional experiments on overall model quality, additional baselines, and additional datasets are in Appendix~\ref{sec:supp_additional_results}.
Although we focus on coarse-to-fine transfer and robustness here, we note that our method also produces lift on overall model quality.

\paragraph{\sysname Method}
We summarize the \sysname method.\footnote{Our code is available at \url{https://github.com/HazyResearch/thanos-code/}.}
\sysname consists of adding a class-conditional InfoNCE loss and a class-conditional autoencoder to the supervised contrastive loss with standard data augmentations used in~\citet{chen2020simple}.
We implement the former by training an encoder with $\lspread$.
To implement the latter, we train a single autoencoder with a joint MSE reconstruction loss and a cross entropy loss.
We then concatenate the autoencoder representation to the representation of the encoder trained with $\lspread$.
Details on architectures and hyperparameters in Appendix~\ref{sec:supp_details}.

\paragraph{Datasets}
Table~\ref{table:datasets} lists the datasets we use in our evaluation.
We use coarse versions of CIFAR10, CIFAR100, MNIST, and TinyImageNet to study coarse-to-fine transfer.
We use Waterbirds, ISIC and CelebA for robustness~\citep{sagawa2019groupdro,codella2019skin,liu2015faceattributes,sohoni2020george}.

\paragraph{Coarse-to-Fine Transfer}
We use coarse-to-fine transfer learning to isolate how well representations separate subclasses in an ideal setting.
In coarse-to-fine transfer, we train models on coarse labels, freeze the weights, and then train a linear probe over the final layer on fine labels.
Note that this setting is more challenging than the self-supervised setting, since it requires maintaining high performance on the coarse classes while also being transferrable to the fine classes.
We focus on transfer numbers in this section, but Table~\ref{table:coarse-quality} in the Appendix presents results on coarse accuracy.

For the autoencoder experiments, we train an autoencoder separately and concatenate its embedding layer with the contrastive embedding for the linear probe.
We jointly optimize all contrastive losses and the class-conditional autoencoder with a cross-entropy loss head.
We train all models with dropout as well as label smoothing on the cross-entropy loss heads.

We report four variants of \sysname: the class-conditional autoencoder on its own, SupCon modified with a class-conditional InfoNCE loss, SupCon modified with a class-conditional autoencoder, and SupCon with both modifications.
We report \num{3} baselines from previous work on the transferability of SupCon~\citep{islam2021broad}: SupCon, SupCon plus an InfoNCE loss, and the InfoNCE loss on its own.

\sysname\ significantly outperforms SupCon on coarse-to-fine transfer learning---by an average of \num{11.1} points across all tasks.
\num{7.3} points can be attributed to the class-conditional InfoNCE loss on average, but mileage varies between tasks (\num{25.3} points of lift for CIFAR10, vs. \num{2.6} for CIFAR100).
The difference is the number of coarse classes: CIFAR10 only has two coarse classes, whereas CIFAR100 has 20.
Fewer coarse classes makes it easier to achieve class collapse, so spread is more necessary.
Finally, we also note that combining the autoencoder with the other components outperforms using the autoencoder on its own, by \num{2.7} points on average.
This suggests that each component is helpful.


\begin{table}[t]
    \caption{
    Unsupervised subclass recovery (top, F1), and worst-group
    performance (AUROC for ISIC, Acc for others).
    Best in bold.
    }
    \centering
    \scriptsize
    \begin{tabular}{lccccccccccc}
        \toprule
                                 & \textbf{Group} \\
        \textbf{Method}          & \textbf{Labels} & \textbf{Waterbirds} & \textbf{ISIC}   & \textbf{CelebA} \\ \midrule
        & & \multicolumn{3}{c}{\textbf{Sub-Group Recovery}} \\ \cmidrule(lr){3-5}
        \citet{sohoni2020george} & \xmark & 56.3                & 74.0            & 24.2  \\
        SupCon                   & \xmark & 47.1                & 92.5            & 19.4  \\
        \sysname                 & \xmark & \textbf{59.0}       & \textbf{93.8}   & \textbf{24.8}  \\
        \midrule
        & & \multicolumn{3}{c}{\textbf{Worst-Group Robustness}}  \\ \cmidrule(lr){3-5}
        \citet{sohoni2020george} & \xmark & 88.4                & 92.0            & 55.0  \\
        JTT~\citep{liu2021just}  & \xmark & 83.8                & 91.8            & 77.9 \\
        SupCon                   & \xmark & 86.8                & \textbf{93.3}   & 66.1  \\
        \sysname                 & \xmark & \textbf{88.6}       & 92.6            & \textbf{89.4} \\ \cmidrule(lr){1-5}
        GroupDRO 
                                 & \cmark & 90.7 & 92.3 & 88.9 \\
        \bottomrule
    \end{tabular}

\label{table:robustness}
\end{table}

\paragraph{Worst-Group Robustness}
We use robustness to measure how well \sysname\ can recover hidden subgroups in an unsupervised setting.
For these models, we train contrastive losses on their own.
We follow the methodology from~\citet{sohoni2020george}.
We first train a model with class labels.
We then cluster the embeddings to produce pseudolabels for subclasses, which we use as input to the GroupDRO algorithm to optimize worst-group robustness~\citep{sagawa2019groupdro}.

Our primary evaluation metric is robustness, but we also evaluate a subgroup recovery metric since prior work has suggested that it is important for robustness.
Subgroup recovery also acts as a proxy for \textit{unsupervised} group recovery.
We compare subgroup recovery against SupCon and~\citet{sohoni2020george}.
We compare worst-group robustness against \citet{sohoni2020george} and JTT~\citep{liu2021just}, as well as using sub-group labels from SupCon.
We also report the performance of GroupDRO with ground-truth subclass labels.

Table~\ref{table:robustness} shows the results.
\sysname\ outperforms both SupCon and \citet{sohoni2020george} on subgroup
recovery.
\sysname\ further achieves state-of-the-art worst-group robustness, outperforming JTT by \num{4.7} points and \citet{sohoni2020george} by \num{11.7} points on average---and setting state-of-the-art on CelebA by \num{11.5} points.
Surprisingly, \sysname\ can even outperform GroupDRO with ground-truth subgroup labels in two cases.

Subgroup recovery and worst-group robustness are correlated but not causal:~\citep{sohoni2020george} observed inconsistencies between them, and so do we (i.e., our method outperforms GroupDRO, an approach with ``perfect'' subgroup labels).
This phenomenon deserves further exploration.

\paragraph{Ablations}
We summarize two ablations (Appendix~\ref{subsec:supp_ablations}).
First, we validate Lemma~\ref{lemma:autoencoder} and find that using a generic autoencoder underperforms a class-conditional autoencoder by \num{30.0} points on average---and furthermore does not improve the performance of SupCon as well (\num{2.0} points of lift compared to \num{11.0} points).
Second, we validate Lemma~\ref{lemma:augmentation} and confirm that data augmentation is crucial; removing data augmentation degrades performance by \num{35.4} points.


\section{Related Work and Discussion}
\label{sec:related}

We present an abbreviated related work.
A full treatment can be found in Appendix~\ref{sec:supp_related}.
Our theoretical work relates to theory on the geometry of contrastive learning~\citep{wang2020understanding,graf2021dissecting,robinson2020contrastive,zimmermann2021contrastive}, collapsed representations~\citep{galanti2021role,jing2021understanding}, autoencoders~\citep{epstein2019generalization,le2018supervised}, data augmentation~\citep{haochen2021provable,guo2018deep,abavisani2020deep}, and robustness~\citep{sohoni2020george}.
Our use of $\lspread$ and an autoencoder draws from a wave of empirical work on contrastive learning~\citep{chen2020simple,khosla2020supervised}, and its properties~\citep{islam2021broad,bukchin2021fine}.

In aggregate, we study how to improve the quality of representations trained with supervised contrastive learning.
We identify controlling spread and inducing subclass clustering as two key challenges and show how two modifications to supervised contrastive learning improve transfer and robustness.

\ifarxiv
\paragraph{Authors' Note}
The first two authors contributed equally.
Co-first authors can prioritize their names when adding this paper's reference
to their resumes.

\section*{Acknowledgments}

We thank Beidi Chen, Tri Dao, Karan Goel, and Albert Gu for their helpful comments on early drafts of this paper.
We gratefully acknowledge the support of NIH under No. U54EB020405 (Mobilize), NSF under Nos. CCF1763315 (Beyond Sparsity), CCF1563078 (Volume to Velocity), and 1937301 (RTML); ONR under No. N000141712266 (Unifying Weak Supervision); ONR N00014-20-1-2480: Understanding and Applying Non-Euclidean Geometry in Machine Learning; N000142012275 (NEPTUNE); the Moore Foundation, NXP, Xilinx, LETI-CEA, Intel, IBM, Microsoft, NEC, Toshiba, TSMC, ARM, Hitachi, BASF, Accenture, Ericsson, Qualcomm, Analog Devices, the Okawa Foundation, American Family Insurance, Google Cloud, Salesforce, Total, the HAI-GCP Cloud Credits for Research program, the Stanford Data Science Initiative (SDSI), 
Department of Defense (DoD) through the National Defense Science and
Engineering Graduate Fellowship (NDSEG) Program, 
and members of the Stanford DAWN project: Facebook, Google, and VMWare. The Mobilize Center is a Biomedical Technology Resource Center, funded by the NIH National Institute of Biomedical Imaging and Bioengineering through Grant P41EB027060. The U.S. Government is authorized to reproduce and distribute reprints for Governmental purposes notwithstanding any copyright notation thereon. Any opinions, findings, and conclusions or recommendations expressed in this material are those of the authors and do not necessarily reflect the views, policies, or endorsements, either expressed or implied, of NIH, ONR, or the U.S. Government.

\fi

\bibliography{main}
\bibliographystyle{plainnat}

\appendix

\newpage

\onecolumn

We present a full treatment of related work in Appendix~\ref{sec:supp_related}.
We present a glossary in Appendix~\ref{sec:supp_glossary}.
We present proofs in Appendix~\ref{sec:app_proofs}, additional theoretical results in Appendix~\ref{sec:supp_additional_theory}, and auxiliary lemmas in Appendix~\ref{sec:supp_aux_lemmas}.
We present additional experimental details in Appendix~\ref{sec:supp_details}, additional results in Appendix~\ref{sec:supp_additional_results}, and synthetics in Appendix~\ref{sec:supp_synthetics}.

\section{Related Work}
\label{sec:supp_related}

We presented an extended treatment of related work.

From work in contrastive learning,
we take inspiration from \citet{arora2019theoretical}, who use a latent
classes view to study self-supervised contrastive learning. 
Similarly, \citet{zimmermann2021contrastive} considers how minimizing the InfoNCE loss recovers a latent data generating model.
Recent work has also analyzed contrastive learning from the
information-theoretic perspective~\citep{oord2018representation, tian2020makes,
tsai2020self}, but does not fully explain practical behavior~\citep{tschannen2020on}.
On the geometric side, we are inspired by the theoretical tools from 
\citet{wang2020understanding} and \citet{graf2021dissecting},
who study representations on the hypersphere along with~\citet{robinson2020contrastive}.
There has been work studying other notions of collapsed repesentations. 
\citet{jing2021understanding} examines dimension collapse in contrastive learning, which occurs when the learned representations span a subspace of the representation space. Our definition of class collapse can also be viewed as Neural Collapse~\citep{papyan2020prevalence}, which started as an empirical observation about when models are trained beyond $0$ training error using cross-entropy or MSE loss~\citep{lu2020neural, han2021neural}. Recent works on neural collapse have studied the transferrability of collapsed representations~\citep{galanti2021role}, and~\citet{hui2022limitations, kothapalli2022neural} have identified its limitations in this setting. We offer another perspective on the relationship between collapse and embedding quality, and offer techniques to mitigate the effects of collapse in transfer learning.

Our work builds on the recent wave of empirical interest in
contrastive learning~\citep{chen2020simple, he2019moco, chen2020mocov2,
goyal2021self, caron2020swav} and supervised contrastive
learning~\citep{khosla2020supervised}.
There has also been empirical work analyzing the transfer performance of contrastive
representations and the role of intra-class variability in transfer learning.
\citet{islam2021broad} find that combining supervised and self-supervised contrastive loss 
improves transfer learning performance, and they hypothesize that this is due to both inter-class separation 
and intra-class variability.
\citet{bukchin2021fine} find that combining cross entropy
and a class-conditional self-supervised contrastive loss improves coarse-to-fine transfer, also
motivated by preserving intra-class variability.

Our use of $\lspread$ and a class-conditional autoencoder arises from similar motivations to losses proposed in these works,
and we futher theoretically study their implications for spread.
Our theoretical analysis of autoencoders draws from previous work~\citep{epstein2019generalization,le2018supervised}.
Our study of data augmentation similarly builds on recent theoretical analysis of the role of data augmentation in contrastive learning~\citep{haochen2021provable, huang2021generalization} and clustering~\citep{guo2018deep,abavisani2020deep}.

Our treatment of subclasses is strongly inspired by~\citet{sohoni2020george}
and \citet{oakden2020hidden}, who document empirical consequences of hidden strata.
We are inspired by empirical work that has demonstrated that detecting subclasses
can be important for performance~\citep{hoffmann2001using,d2021spotlight} and
robustness~\citep{duchi2020distributionally,sagawa2019groupdro,goel2020model,liu2021just}.

\section{Glossary}
\label{sec:supp_glossary}

The glossary is given in Table~\ref{table:glossary} below.
\begin{table*}[bp!]
\centering
\ifarxiv
\scriptsize
\else
\small
\fi
\begin{tabular}{l l}
\toprule
Symbol & Used for \\
\midrule
$x$ & Input data $x \in \X$ with distribution $\p$.\\
$y$ & Class label $y \in \Y = \{0, \dots K - 1\}$, where $h(x)$ is $x$'s class label.\\
$z$ & Latent subclass $z \in \Z$. \\
$S_y$ & The set of all subclasses corresponding to class label $y$. \\
$p(z)$ & The proportion of subclass $z$ over $\Z$.\\
$\p_z$ & The distribution of input data belonging to subclass $z$, i.e. $\p_z = p(\cdot | z)$. \\
$S(z)$ & The label corresponding to subclass $z$. \\
$h_s(x)$ & The subclass that $x$ belongs to. \\
$\D$ & Training dataset of $n$ points $\{(x_i, y_i)\}_{i = 1}^n$. \\
$\D_y$ & Training data with label $y$, $\D_y = \{x \in \D: h(x) = y\}$ of size $n_y$. \\
$\D_z$ & Training data with latent subclass $z$, $\D_z= \{x \in \D: h_s(x) = z\}$ of size $n_z$. \\
$f$ & The encoder $f: \X \rightarrow \mathbb{R}^d$ that maps input data to an embedding space with dimension $d$. \\
$\D_s$ & A dataset of $m$ points with subclass labels $\D_s = \{(x_i, z_i)\}_{i= 1}^m$ used for coarse-to-fine transfer. \\
$\D_{s, z}$ & The subset of $\D_s$ with subclass $z$, $\D_{s, z} = \{x \in \D_s: h_s(x) = z \}$ of size $m_z$. \\
$W_z$ & Linear weight $W_z$ for model used in coarse-to-fine transfer. \\
$\phat(z | f(x))$ & Softmax score output by linear model for coarse-to-fine transfer.  \\
$L_{\gamma, f}(z)$ & The $\gamma$-margin generalization error on subclass $z$ in coarse-to-fine transfer. \\
$B$ & Batch of input data. \\
$P(i, B)$ & Points in $B$ with the same label as $x_i$, $\{x^+ \in B\backslash i: h(x^+) = h(x_i) \}$. \\
$N(i, B)$ & Points in $B$ with a label different from that of $x_i$, $\{x^- \in B \backslash i: h(x^-) \neq h(x_i) \}$. \\
$a(x_i)$ & An augmentation of $x_i$, where $a: \X \rightarrow \X$. \\
$\sigma_f(x, x')$ & Notation for $\exp\Big(\frac{f(x)^\top f(x')}{\tau}\Big)$. \\
$\tau$ & Temperature hyperparameter in contrastive loss. \\
$\lspreadhat(f, B)$ & The contrastive loss we study (on batch $B$ with encoder $f$), a weighted sum of a SupCon and \\
& class-conditional InfoNCE loss. \\
$\alpha$ & Weight parameter for $\lspread$. \\
$L_{\text{sup}}$ & SupCon loss that is used in $\lspread$ that pushes points of a class together (see~\eqref{eq:sup}). \\
$L_{\text{cNCE}}$ & Class-conditional InfoNCE loss that is used in $\lspread$ to pull apart points within a class (see~\eqref{eq:cnce}). \\
$\mathcal{S}^{d-1}$ & The unit hypersphere in $\mathbb{R}^d$. \\
$\mu_y$ & The pushforward measure of the class-conditional distribution $p(\cdot | h(x) = y$ via $f$, where $\mu_y \in \mathcal{M}(\mathcal{S}^{d-1})$, \\
&the set of all Borel probability measures on the hypersphere. \\
$\bm{\mu}$ & $\bm{\mu} = \{\mu_y \}_{y \in \Y}$ is the overall pushforward measure $\p \circ f^{-1} \in \mathcal{M}(\mathcal{S}^{d-1})$. \\
$\bm{v}$ & $\bm{v} = \{v_y\}_{y \in \Y} \in \mathcal{S}^{d-1}$ is the regular simplex inscribed in the hypersphere. \\
$\delta_{v_y}$ & The probability measure on $\mathcal{S}^{d-1}$ with all mass on $v_y$. \\
$\bm{\delta_v}$ & The class-collapsed measure $\bm{\delta_v} = \{\delta_{v_y} \}_{y \in \Y}$ where $f(x) = v_y$ almost surely whenever $h(x) = y$. \\
$\sigma_{d-1}$ & The normalized surface area measure on $\mathcal{S}^{d-1}$. \\
$\bm{\sigma_{d-1}}$ & The class-uniform measure where $\mu_y = \sigma_{d-1}$ for all $y \in \Y$. \\
$f_{SC}$ & The encoder trained with SupCon, satisfying $f_{SC}(x) = v_y$ for all $x \in \D$ where $h(x) = y$. \\
$x^+$ & Point for $x$'s positive pair, drawn from distribution $p(\cdot | h(x^+) = h(x))$. \\
$x^-$ & Point for $x$'s negative pair, drawn from distribution $p(\cdot | h(x^-) \neq h(x))$. \\
$\lspread(f, \alpha)$ & Asymptotic version of $\lspread$ that we analyze (see Definition~\ref{def:asymptotic}, also referred to as $\lspread(\bm{\mu}, \alpha)$. \\
$R_{\theta}$ & Rotation matrix $R_{\theta} \in \mathbb{R}^{d \times d}$ that rotates by angle $\theta$ in two dimensions. \\
$\bm{\mu_{\theta}}$ & A measure $\bm{\mu_{\theta}} = \{\mu_{0, \theta}, \mu_{1, \theta}\}$ on the hypersphere that we compare against $\bm{\delta_v}$ and $\bm{\sigma_{d-1}}$. \\
&In particular, $\mu_{0, \theta} = \frac{1}{2} \delta_{R_{\theta}^\top v_0} + \frac{1}{2}\delta_{R_{-\theta}^\top v_0}$ and similarly for $\mu_{1, \theta}$. \\
$c_{\tau, d}$ & Constant that upper bounds the range of $\alpha$ for which some $\bm{\mu_{\theta}}$ attains lower $\lspread(\bm{\mu}, \alpha)$ than $\bm{\delta_v}$ or $\bm{\sigma_{d-1}}$. \\
$s_f(y)$ & Notion of spread in embedding space, defined as $s_f(y) = \E{h(x) = y}{\|f(x) - \E{h(x) = y}{f(x)}\|}$. \\
$\sigma_f(z)$ & Notion of subclass clustering in embedding space, defined as $\sigma_f(z) = \E{x \sim \p_z}{\|f(x) - \E{x \sim \p_z}{f(x)} \|}$. \\
$\Var{f}{z}$ & Notion of subclass variance, defined as $\Var{f}{z} = \E{x \sim \p_z}{\|f(x) - \E{x \sim \p_z}{f(x)} \|^2}$.  \\
$\sigma_z$ & How clustered a subclass is in input space, defined as $\sigma_z = \E{x, x' \sim \p_z}{\|x - x'\|}$. \\
$K_L$ & The Lipschitzness constant of a Lipschitz encoder from function class $\F_{K_L}$. \\
$f_{AE}, g$ & Autoencoder with encoder $f_{AE} \in \F_{AE}$ and decoder $g \in \G$. \\
$\hat{L}_{AE}$ & The autoencoder's reconstruction loss (mean squared error). \\
$\rademacher_n^p(\F_1, \F_2)$ & Notion of Rademacher complexity defined as $\rademacher_n^p(\F_1, \F_2) = \E{\bm{\sigma}}{\sup_{f_1, f_2 \in \F_1, \F_2} \frac{1}{n} \sum_{i = 1}^n \sigma_i \|f_1(x_i) - f_2(x_i)\|^p}$. \\
$K_g$ & ``Reverse Lipschitzness'' constant of the decoder, e.g. $\|f_{AE}(x) - f_{AE}(x')\|\le K_g \|g(f_{AE}(x)) - g(f_{AE}(x))\|$. \\
$\A$ & Function class of augmentations $\A: \X \rightarrow \X$. \\
$\F_{aug}$ & Function class of encoders that are trained on augmentations. \\
$K_a$ & The Lipschitzness constant on augmentations for $f_{aug} \in \F_{aug}$, e.g. \\
&$\|f_{aug}(a(x)) - f_{aug}(a(x')) \| \le K_a \|a(x) - a(x') \|$ for any $x', x \in \X$ and $a \in \A$. \\
$\sigma_z^{aug}$ & How clustered augmentations of a subclass are in input space.\\
\toprule
\end{tabular}
\caption{
	Glossary of variables and symbols used in this paper.
}
\label{table:glossary}
\end{table*}

\newpage

\section{Proofs}\label{sec:app_proofs}

\subsection{Proofs for Section~\ref{sec:spread}}\label{subsec:app_spread_proofs}

\begingroup
\def\thetheorem{\ref{thm:supcongenerr}}
\begin{theorem}
For $\gamma$ where $\log \gamma \ge 8 \max\{\rademacher_{n_z}(\F), \rademacher_{n_{z'}}(\F)\}$, SupCon's coarse-to-fine error is at least 
\begin{align*}
L_{\gamma, f_{SC}}(z) &\ge 1 - \delta(n_z, \F, \gamma) - \delta(n_{z'}, \F, \gamma) - \xi(m_z \wedge m_{z'}, \gamma ), \nonumber 
\end{align*}
where $\delta(n_z, \F, \gamma) = d \exp \Big(-\frac{n_z}{32d^2} (\log \gamma - 8 \rademacher_{n_z}(\F))^2 \Big)$ bounds generalization error of $f_{SC}$ and $\xi(m_z \wedge m_{z'}, \gamma) = 4d \exp\Big(-\frac{(m_z \wedge m_{z'}) \log^2 \gamma}{32d} \Big)$ bounds the noise from $\D_s$.
\end{theorem}
\addtocounter{theorem}{-1}
\endgroup

\begin{proof}
$L_{\gamma, f}(z) = \Pr_{x \sim \p_z}( \phat(z | f(x)) \le \gamma \phat(z' | f(x)))$, and by definition of the linear softmax classifier, we have that
\begin{align}
L_{\gamma, f}(z) &= \Pr_{x, \sim \p_z}\Big(\frac{\exp(f(x)^\top W_z)}{\exp(f(x)^\top W_z) + \exp(f(x)^\top W_{z'})} \le \gamma \frac{\exp(f(x)^\top W_{z'})}{\exp(f(x)^\top W_z) + \exp(f(x)^\top W_{z'})} \Big) \nonumber \\
& = \Pr_{x \sim \p_z}(f(x)^\top (W_z - W_{z'}) \le \log \gamma). \label{eq:gamma_loss}
\end{align}

To lower bound this quantity, we focus on upper bounding $f(x)^\top (W_z - W_{z'})$.
We can bound $f_{SC}(x)^\top (W_z - W_{z'}) \le \|W_z - W_{z'} \| \le \|\E{x \sim \p_z}{f_{SC}(x)} - \E{x' \sim \p_{z'}}{f_{SC}(x')}\| + \xi_z + \xi_{z'}$, where $\xi_z = \big\|\frac{1}{m_z} \sum_{x \in \D_{s, z}} f_{SC}(x) - \E{x \sim \p_z}{f_{SC}(x)}\big\|$ can be bounded via standard concentration inequalities and $\xi_{z'}$ is similarly constructed. 

Because SupCon yields collapsed training embeddings within any class, we know that $f_{SC}(x) = f_{SC}(x')$ for $x, x' \in \D$ where $h_s(x) = z$ and $h_s(x') = z'$. Therefore, it holds that 
\begin{align*}
&\|\E{x \sim \p_z}{f_{SC}(x)} - \E{x' \sim \p_{z'}}{f_{SC}(x')}\| \\
&= \Big\|\E{x \sim \p_z}{f_{SC}(x)} - \frac{1}{n_z} \sum_{x \in \D_z} f_{SC}(x) + \frac{1}{n_{z'}} \sum_{x' \in \D_{z'}} f_{SC}(x')- \E{x' \sim \p_{z'}}{f_{SC}(x')}\Big\| \\
&\le \Big\|\E{x \sim \p_z}{f_{SC}(x)} - \frac{1}{n_z} \sum_{x \in \D_z} f_{SC}(x) \Big\| + \Big\| \frac{1}{n_{z'}} \sum_{x' \in \D_{z'}} f_{SC}(x')- \E{x' \sim \p_{z'}}{f_{SC}(x')}\Big\| \\
&\le \sup_{f \in \F}\Big\|\E{x \sim \p_z}{f(x)} - \frac{1}{n_z} \sum_{x \in \D_z} f(x) \Big\| + \sup_{f \in \F}\Big\| \frac{1}{n_{z'}} \sum_{x' \in \D_{z'}} f(x')- \E{x' \sim \p_{z'}}{f(x')}\Big\|.
\end{align*}

Define $\epsilon(z, \D, \F) = \sup_{f \in \F} \|\E{x \sim \p_z}{f(x)} - \frac{1}{n_z} \sum_{x \in \D_z} f(x)\| $ and similarly $\epsilon(z', \D, \F)$. Therefore, our loss in~\eqref{eq:gamma_loss} satisfies
\begin{align}
L_{\gamma, f_{SC}}(z) &\ge \Pr(\epsilon(z, \D, \F) + \epsilon(z', \D, \F) + \xi_z + \xi_{z'} \le \log \gamma) \nonumber \\
&\ge \Pr\Big(\epsilon(z, \D, \F) \le \frac{\log \gamma}{4}\Big) \Pr\Big( \epsilon(z', \D, \F) \le \frac{\log \gamma}{4}\Big) \Pr \Big( \xi_z \le \frac{\log \gamma}{4}\Big) \Pr \Big( \xi_{z'} \le \frac{\log \gamma}{4} \Big), \label{eq:gamma_loss_iid}
\end{align}

where independence comes from the fact that data is i.i.d. sampled for each subclass and each $\D$ and $\D_s$, and that we are taking the supremum over $\F$. Next, we bound $\epsilon(z, \D, \F)$. Since $\|f(x)\| \le 1$, we have that by Lemma~\ref{lemma:rademacher} that with probability $1 - \delta$,
\begin{align*}
\epsilon(z, \D, \F) \le 2 \rademacher_{n_z}(\F) + d\sqrt{\frac{2 \log(d /\delta)}{n_z}}.
\end{align*}

Setting $\epsilon := 2 \rademacher_{n_z}(\F) + d\sqrt{\frac{2 \log(d /\delta)}{n_z}}$, we can write $\delta = d \exp \Big(-\frac{n_z}{2d^2} (\epsilon - 2\rademacher_{n_z}(\F))^2 \Big)$. Therefore, for $\log \gamma \ge 8 \rademacher_{n_z}(\F)$, we have that 
\begin{align}
\Pr\Big(\epsilon(z, \D, \F) \le \frac{\log \gamma}{4}\Big) \ge 1 - d\exp \Big(-\frac{n_z}{32d^2} (\log \gamma - 8\rademacher_{n_z}(\F))^2 \Big). \label{eq:epsilon_gen_err}
\end{align}

Next, we bound $\xi_z$. We can write 
\begin{align*}
\bigg\|\frac{1}{m_z} \sum_{x \in \D_{s, z}} f_{SC}(x) - \E{x \sim \p_z}{f_{SC}(x)} \bigg\| = \bigg(\sum_{j = 1}^d  \Big(\frac{1}{m_z} \sum_{x \in \D_{s, z}} f_{SC}(x)[j] - \E{x \sim \p_z}{f_{SC}(x)[j]}\Big)^2 \bigg)^{1/2},
\end{align*}

where $f_{SC}(x)[j]$ is the $j$th element of $f_{SC}(x)$.
\ifarxiv
Using Hoeffding's inequality, we have that $\Pr(\frac{1}{m_z} \sum_{x \in \D_{s, z}} \newline f_{SC}(x)[j] - \E{x \sim \p_z}{f_{SC}(x)[j]})^2 \ge \epsilon) \le 2 \exp(-\frac{m_z \epsilon}{2})$, and therefore 
\begin{align*}
&\Pr \bigg(\sum_{j = 1}^d  \Big(\frac{1}{m_z} \sum_{x \in \D_{s, z}} f_{SC}(x)[j] - \E{x \sim \p_z}{f_{SC}(x)[j]}\Big)^2 \le d \epsilon \bigg) \\
&\ge \Pr \bigg(\bigcap_{j = 1}^d \Big(\frac{1}{m_z} \sum_{x \in \D_{s, z}} f_{SC}(x)[j] - \E{x \sim \p_z}{f_{SC}(x)[j]}\Big)^2 \le \epsilon \bigg) \\
&\ge \Big(1 - 2\exp\Big(-\frac{m_z \epsilon}{2}\Big)\Big)^d \ge 1 - 2d\exp\Big(-\frac{m_z \epsilon}{2}\Big).
\end{align*}
\else
Using Hoeffding's inequality, we have that $\Pr(\frac{1}{m_z} \sum_{x \in \D_{s, z}} f_{SC}(x)[j] - \E{x \sim \p_z}{f_{SC}(x)[j]})^2 \ge \epsilon) \le 2 \exp(-\frac{m_z \epsilon}{2})$, and therefore  
\begin{align*}
&\Pr \bigg(\sum_{j = 1}^d  \Big(\frac{1}{m_z} \sum_{x \in \D_{s, z}} f_{SC}(x)[j] - \E{x \sim \p_z}{f_{SC}(x)[j]}\Big)^2 \le d \epsilon \bigg) \ge \Pr \bigg(\bigcap_{j = 1}^d \Big(\frac{1}{m_z} \sum_{x \in \D_{s, z}} f_{SC}(x)[j] - \E{x \sim \p_z}{f_{SC}(x)[j]}\Big)^2 \le \epsilon \bigg) \\
&\ge \Big(1 - 2\exp\Big(-\frac{m_z \epsilon}{2}\Big)\Big)^d \ge 1 - 2d\exp\Big(-\frac{m_z \epsilon}{2}\Big).
\end{align*}
\fi

That is, $\Pr(\xi_z \le \sqrt{d \epsilon}) \ge 1 - 2d \exp\Big(-\frac{m_z \epsilon}{2} \Big)$.

Setting $\frac{\log \gamma}{4} = \sqrt{d \epsilon}$ gives us $\Pr(\xi_z \le \frac{\log \gamma}{4}) \ge 1 - 2d\exp(- \frac{m_z \cdot \log^2 \gamma / 16d}{2}) = 1 - 2d\exp(- \frac{m_z \cdot \log^2 \gamma}{32d})$.

We put this expression and~\eqref{eq:epsilon_gen_err} back into~\eqref{eq:gamma_loss_iid} and use the fact that $(1 - \delta_1)(1 -\delta_2) \ge 1 - \delta_1 - \delta_2$ for any $\delta_1, \delta_2 > 0$. Therefore, we have for SupCon,
\begin{align*}
L_{\gamma, f_{SC}}(z) &\ge 1 - d\exp \Big(-\frac{n_z}{32d^2} (\log \gamma - 8\rademacher_{n_z}(\F))^2 \Big) - d\exp \Big(-\frac{n_{z'}}{32d^2} (\log \gamma - 8\rademacher_{n_{z'}}(\F))^2 \Big) \\
&- 2d\exp\Big(-\frac{m_z \log^2 \gamma}{32d}\Big) - 2d\exp\Big(-\frac{m_{z'} \log^2 \gamma}{32d}\Big) \\
&\ge 1 - d\exp \Big(-\frac{n_z}{32d^2} (\log \gamma - 8\rademacher_{n_z}(\F))^2 \Big) - d\exp \Big(-\frac{n_{z'}}{32d^2} (\log \gamma - 8\rademacher_{n_{z'}}(\F))^2 \Big) \\
&- 4d\exp\Big(-\frac{(m_z \wedge m_{z'}) \log^2 \gamma}{32d}\Big) \\
&= 1 - \delta(n_z, \F, \gamma) - \delta(n_{z'}, \F, \gamma) - \xi(m_z \wedge m_{z'}, \gamma).
\end{align*}
\end{proof}

\paragraph{Derivation of $\lspread(f, \alpha)$} We explain how we arrive at the asymptotic form of $\lspread$. Let $\lspread(f, n^+, n^-)$ be the population-level version of $\lspreadhat(f, B)$, where $n^+, n^-$ are the number of negatives in the denominators of $\cnce$ and $\supcon$, respectively. 
\begin{align}
\lspread(f, n^+, n^-) &= (1 - \alpha) \supcon(f, n^-) + \alpha \cnce(f, n^+) \\ 
\supcon(f, n^-) &= -\E{}{\log \frac{\sigma_f(x, x^+)}{\sigma_f(x, x^+) + \sum_{i = 1}^{n^-} \sigma_f(x, x_i^-)}} \\
\cnce(f, n^+) &= -\E{}{\log \frac{\sigma_f(x, a(x))}{\sigma_f(x, a(x)) + \sum_{i = 1}^{n^+} \sigma_f(x, x_i^+)}}
\end{align}

We now demonstrate how minimizing $\lim_{n^+, n^- \rightarrow \infty} \lspread(f, n^+, n^-)$ is equivalently to minimizing $\lspread(f, \alpha)$. In $\supcon$, we divide the numerator and denominator by $n^-$, and in $\cnce$ we divide the numerator and denominator by $n^+$:

\begin{align*}
\supcon(f, n^-) &= \E{}{-\log \frac{\sigma_f(x, x^+)}{\frac{1}{n^-}\sigma_f(x, x^+) + \frac{1}{n^-} \sum_{i = 1}^{n^-} \sigma_f(x, x_i^-) }} + \log n^- \\
\cnce(f, n^+) &= \E{}{-\log \frac{\sigma_f(x, a(x))}{\frac{1}{n^+}\sigma_f(x, a(x)) + \frac{1}{n^+} \sum_{i = 1}^{n^+} \sigma_f(x, x_i^+)   }  } + \log n^+
\end{align*}

We can thus write $\lspread(f, n^+, n^-)$ as
\begin{align*}
L_{spread}(f, n^+, n^-) - &\alpha \log n^- - (1 - \alpha) \log n^+ = -\alpha \E{}{\log \sigma_f(x, x^+)} - (1 - \alpha) \E{}{\log \sigma_f(x, a(x))} \\
&+ \alpha \E{}{\log \bigg(\frac{1}{n^-} \sigma_f(x, x^+) + \frac{1}{n^-} \sum_{i = 1}^{n^-} \sigma_f(x, x_i^-)\bigg)} \\
&+ (1 - \alpha) \E{}{\log \bigg(\frac{1}{n^+} \sigma_f(x, a(x)) + \frac{1}{n^+} \sum_{i = 1}^{n^+} \sigma_f(x, x_i^+)\bigg)}.
\end{align*}

Taking the limit $n^+, n^- \rightarrow \infty$ yields
\begin{align*}
\lim_{n^+, n^- \rightarrow \infty} &L_{spread}(f, n^+, n^-) - \alpha \log n^- - (1 - \alpha) \log n^+ \\
&= -\alpha \E{}{\log \sigma_f(x, x^+)} - (1 - \alpha) \E{}{\log \sigma_f(x, a(x))} \\
&+ \alpha \E{x}{\log \E{x^-}{\sigma_f(x, x^-)}} + (1 - \alpha) \E{x}{\log \E{x^+}{\sigma_f(x, x^+)}}.
\end{align*}

Lastly, we use the fact that $\sigma_f(x, x') = \exp(f(x)^\top f(x') / \tau) = \exp(-\|f(x) - f(x') \|^2 / 2 \tau + 1/ \tau)$, since $f(x), f(x') \in \mathcal{S}^{d-1}$ to get that 
\begin{align*}
\lim_{n^+, n^- \rightarrow \infty} L_{spread}(f, n^+, n^-) &- \alpha \log n^- - (1 - \alpha) \log n^+ \\
&= (1 - \alpha) L_{\text{align}}(f) + \alpha L_{\text{aug}}(f) - \frac{1}{\tau} + (1 - \alpha) L_{\text{diff}}(f) + \alpha L_{\text{same}}(f) + \frac{1}{\tau} \\
&= \lspread(f, \alpha).
\end{align*}

\termwise*

\begin{proof}

We analyze each term's optimal measure on the hypersphere.

\paragraph{$L_{\text{align}}(f), L_{\text{aug}}(f)$} For both $L_{\text{align}}(f)$ and $L_{\text{aug}}(f)$, the minimum value of the expression is $0$, which is obtained when $f(x) = f(x^+)$ and $f(x) = f(a(x))$ almost surely, respectively.

\paragraph{$L_{\text{diff}}(f)$} We show in Lemma~\ref{lemma:jensen_swap} that the measure that minimizes $L_{\text{diff}}(f)$ also minimizes
\ifarxiv
 $\log \mathbb{E}_{x, x^-}[-\exp(\|f(x) - \newline f(x^-) \|^2 / 2\tau)]$
\else
 $\log \E{x, x^-}{-\exp(\|f(x) - f(x^-) \|^2 / 2\tau)}$
\fi
 (note this is not identical to the approach in~\citet{wang2020understanding}). We can thus equivalently consider minimizing $\E{x, x^-}{\exp(-\|f(x) - f(x^-) \|^2 / 2\tau)}$. 
Note that $\max_{f(x), f(x') \in \mathcal{S}^{d-1}} \|f(x) - f(x') \| = 2$, and so $\inf \E{x, x^-}{\exp(-\|f(x) - f(x^-) \|^2 / 2\tau)} = \exp(-2/\tau)$.
For $\bm{\mu} = \bm{\delta_v}$,
\begin{align}
\E{x, x^-}{\exp(-\|f(x) - f(x^-) \|^2 / 2\tau)} &= \frac{1}{2} \E{h(x) = 0, h(x') = 1}{\exp(-\|v_0 - v_1 \|^2 / 2\tau)} \\
&+ \frac{1}{2}\E{h(x) = 1, h(x') = 0}{\exp(-\|v_1 - v_0 \|^2 / 2\tau)}  \nonumber \\
&= \exp(-\|v_0 - v_1\|^2 / 2\tau) = \exp(-2/\tau). \nonumber
\end{align}

The first equality follows from class balance, and the third equality follows from the definition of the regular simplex. Therefore, $\bm{\mu} = \bm{\delta_v}$ minimizes $L_{\text{diff}}(f)$.

\paragraph{$L_{\text{same}}(f)$}For $L_{\text{same}}(f)$, we can directly use the proof of Theorem 1 in~\citet{wang2020understanding} to show that the optimal measure that minimizes $L_{\text{same}}$ also minimizes $\log \E{x, x^+}{\exp(-\|f(x) - f(x^+) \|^2 / 2\tau)}$. We equivalently consider minimizing $\E{x, x^+}{\exp(-\|(f(x) - f(x^+)\|^2 / 2\tau)}$. We can write this as
\begin{align}
\E{x, x^+}{\exp(-\|(f(x) - f(x^+)\|^2 / 2\tau)} &= \frac{1}{2} \E{h(x) = h(x') = 0}{\exp(-\|(f(x) - f(x^+)\|^2 / 2\tau)} \\
&+ \frac{1}{2} \E{h(x) = h(x') = 1}{\exp(-\|(f(x) - f(x^+)\|^2 / 2\tau)}. \nonumber
\end{align}

Using the infinite encoder assumption, we can equivalently consider the following minimization problem over the hypersphere, where $u, u' \in \mathcal{S}^{d-1}$:
\begin{align}
\text{minimize}_{\mu_0, \mu_1} \frac{1}{2} \int \int \exp(-\|u - u'\|^2 / 2\tau) d\mu_0(u) d\mu_0(u') + \frac{1}{2}\int \int \exp(-\|u - u'\|^2 / 2\tau) d\mu_1(u) d\mu_1(u')
\end{align}

Each of these integrals can be minimized individually, and the problem becomes equivalent to having both $\mu_0$ and $\mu_1$ minimize the Gaussian $\frac{1}{2\tau}$-energy. Using Proposition 4.4.1 and Theorem 6.2.1 of~\cite{borodachov2019discrete}, the optimal solution is $\mu_0 = \mu_1 = \sigma_{d-1}$, the normalized surface area measure. Therefore, $\bm{\mu} = \bm{\sigma_{d-1}}$.

\end{proof}

\begingroup
\def\thetheorem{\ref{thm:geometry}}
\begin{theorem}
 Let $c_{\tau, d} \texttt{=}  \frac{2 + \frac{1}{\tau} - \sqrt{\frac{1}{\tau}(-2 + \frac{1}{\tau}) - 2 \log W_{1/2\tau}(\mathcal{S}^{d-1})}}{3}$, where $W_{1/2\tau}(\mathcal{S}^{d-1})$ is the Wiener constant of the Gaussian $\frac{1}{2\tau}$-energy on $\mathcal{S}^{d-1}$, which is defined as
\begin{align*}
W_{1/2\tau}(\mathcal{S}^{d-1}) = \frac{2^{d-2} \Gamma(d/2)}{\sqrt{\pi} \Gamma((d-1)/2)} \int_0^1 \exp\Big(-\frac{2u}{\tau}\Big) (u(1 - u))^{(d-3)/2} du,
\end{align*}

where the Gamma function is $\Gamma(z) = \int_0^{\infty} x^{z-1} e^{-x} dx$ for $z > 0$. Then, when $\alpha \in (2/3, c_{\tau, d})$, $\theta^\star = \arcsin \sqrt{\frac{\tau}{2} \log \frac{3\alpha - 1}{3 - 3\alpha}}$ minimizes $\lspread(\bm{\mu_{\theta}}, \alpha)$ and satisfies $\lspread(\bm{\mu_{\theta^\star}}, \alpha) \le \min_{\bm{\mu} \in \{\bm{\delta_v}, \bm{\sigma_{d-1}} \}} \lspread(\bm{\mu}, \alpha)$.
\end{theorem}
\addtocounter{theorem}{-1}
\endgroup

\begin{proof}

Because the augmentations only play a role in $L_{aug}(f)$ and are disjoint, the condition that $f(x) = f(a(x))$ a.s. is compatible with any of the other three losses in Proposition~\ref{prop:termwise}. Therefore, we focus on analyzing the combined weighted loss $(1 - \alpha) L_{\text{align}}(f) + (1 - \alpha)L_{\text{diff}}(f) + \alpha L_{\text{same}}(f)$. We restate the loss:
\ifarxiv
\begin{align*}
L(\bm{\mu}, &\alpha) = (1 - \alpha) \E{x}{\log \E{x^-}{\exp \bigg(-\frac{1}{2\tau} \|f(x) - f(x^-) \|^2 \bigg)}} \\
&+ \alpha \E{x}{\log \E{x^+}{\exp \bigg(-\frac{1}{2\tau} \|f(x) - f(x^+)\|^2 \bigg)}} + (1 - \alpha) \E{x, x^+}{\frac{1}{2\tau} \|f(x) - f(x^+) \|^2}.
\end{align*}
\else
\begin{align*}
L(\bm{\mu}, \alpha) &= (1 - \alpha) \E{x}{\log \E{x^-}{\exp \bigg(-\frac{1}{2\tau} \|f(x) - f(x^-) \|^2 \bigg)}} \\
&+ \alpha \E{x}{\log \E{x^+}{\exp \bigg(-\frac{1}{2\tau} \|f(x) - f(x^+)\|^2 \bigg)}} + (1 - \alpha) \E{x, x^+}{\frac{1}{2\tau} \|f(x) - f(x^+) \|^2}.
\end{align*}
\fi

We restate the definition of the measure $\bm{\mu_{\theta}} = \{\mu_{0, \theta}, \mu_{1, \theta} \}$ for $\theta \in (0, \pi/2]$ that involves ``splitting'' $\bm{\delta_v}$ by angle $\theta$. Without loss of generality, suppose that $v_0 = e_1$ and $v_1 = -e_1$, where $e_1 \in \mathbb{R}^d$ is a standard basis vector $[1, 0, \dots, 0]$ where all but the first two elements are always $0$. For $v_0$, we consider the vectors $v_{0, \theta} = [\cos \theta, \sin \theta, \dots, 0]$ and $v_{0, -\theta} = [\cos \theta, - \sin \theta, \dots, 0]$. For $v_1$, we consider the vectors $v_{1, \theta} = [-\cos \theta, -\sin \theta, \dots, 0]$ and $v_{1, -\theta} = [-\cos \theta, \sin \theta, \dots, 0]$.
Let $\mu_{0, \theta} = \frac{1}{2}\delta_{v_{0, \theta}} + \frac{1}{2} \delta_{v_{0, -\theta}}$ and let $\mu_{1, \theta} = \frac{1}{2} \delta_{v_{1, \theta}} + \frac{1}{2} \delta_{v_{1, -\theta}}$. That is, each class-conditional measure is a mixture on two points separated by $\theta$.

The first step is to show that for some range of $\alpha$, $\min_{\theta}L(\bm{\mu_{\theta}}, \alpha) < L(\bm{\delta_v}, \alpha)$. We have that
\begin{align}
L(\bm{\delta_v}, \alpha) = (1 - \alpha) \bigg(\frac{1}{2}\log \exp \bigg(-\frac{4}{2\tau} \bigg) + \frac{1}{2}\log \exp \bigg(-\frac{4}{2\tau} \bigg)\bigg) + \alpha \log \exp(0) + (1 - \alpha) \cdot 0 = -\frac{2(1 - \alpha)}{\tau}.
\end{align}

and 
\begin{align}
L(\bm{\mu_{\theta}}, \alpha) &= (1 - \alpha) \log \bigg(\frac{1}{2}\exp\bigg(-\frac{4}{2\tau} \bigg) + \frac{1}{2} \exp \bigg(-\frac{4\cos^2 \theta}{2\tau} \bigg) \bigg) + \alpha \log \bigg(\frac{1}{2} + \frac{1}{2} \exp \bigg(-\frac{4\sin^2 \theta}{2\tau} \bigg) \bigg) \label{eq:mu_theta_loss} \\
&+ \frac{1 - \alpha}{2\tau} \cdot \frac{1}{2} (4 \sin^2 \theta) \nonumber \\
&= - \log 2 -\frac{2(1 - \alpha)}{\tau} + (1 - \alpha) \log \bigg(1 + \exp\bigg(\frac{2\sin^2 \theta}{\tau}\bigg) \bigg) + \alpha \log \bigg(1 + \exp\bigg(-\frac{2\sin^2 \theta}{\tau}\bigg)\bigg) \nonumber \\
&+ \frac{(1 - \alpha) \sin^2 \theta}{\tau}. \nonumber
\end{align}

We now compute the derivative $\frac{\partial L(\bm{\mu_{\theta}}, \alpha)}{\partial \theta}$ to find local minima:
\ifarxiv
\begin{align*}
\frac{\partial L(\bm{\mu_{\theta}}, \alpha)}{\partial \theta} &= (1 - \alpha) \frac{\exp\big(\frac{2\sin^2 \theta}{\tau}\big) \cdot \frac{4 \sin \theta \cos \theta}{\tau}}{1 + \exp\big(\frac{2\sin^2 \theta}{\tau}\big)} + \alpha \frac{\exp\big(-\frac{2 \sin^2 \theta}{\tau}) \cdot \frac{-4 \sin \theta \cos \theta}{\tau}}{1 + \exp\big(\frac{-2 \sin^2 \theta}{\tau} \big)} + \frac{(1 - \alpha) 4 \sin \theta \cos \theta}{\tau} \\
&= \frac{4 \sin \theta \cos \theta}{\tau} \bigg((1 - \alpha) \frac{\exp(2\sin^2 \theta / \tau)}{1 + \exp(2\sin^2 \theta / \tau)} - \alpha \frac{\exp(-2 \sin^2 \theta / \tau)}{1 + \exp(-2\sin^2 \theta / \tau)} + \frac{1 - \alpha}{2} \bigg).
\end{align*}
\else
\begin{align*}
\frac{\partial L(\bm{\mu_{\theta}}, \alpha)}{\partial \theta} &= (1 - \alpha) \frac{\exp(2\sin^2 \theta / \tau) \cdot 4 \sin \theta \cos \theta / \tau}{1 + \exp(2\sin^2 \theta / \tau)} + \alpha \frac{\exp(-2 \sin^2 \theta / \tau) \cdot -4 \sin \theta \cos \theta / \tau}{1 + \exp(-2\sin^2 \theta / \tau)} + \frac{(1 - \alpha) 4 \sin \theta \cos \theta}{\tau} \\
&= \frac{4 \sin \theta \cos \theta}{\tau} \bigg((1 - \alpha) \frac{\exp(2\sin^2 \theta / \tau)}{1 + \exp(2\sin^2 \theta / \tau)} - \alpha \frac{\exp(-2 \sin^2 \theta / \tau)}{1 + \exp(-2\sin^2 \theta / \tau)} + \frac{1 - \alpha}{2} \bigg).
\end{align*}
\fi

Note that $\sin \theta$ and $\cos \theta$ are positive for $\theta \in (0, \pi/2]$. Next, for notational simplicity let $x = \frac{2\sin^2 \theta}{\tau}$. Then, we can equivalently evaluate
\begin{align*}
(1 - \alpha) \frac{e^x}{1 + e^x} - \alpha \frac{e^{-x}}{1 + e^{-x}} + \frac{1 - \alpha}{2} &= (1 - \alpha) \frac{e^x}{1 + e^x} - \alpha \frac{1}{1 + e^x} + \frac{1 - \alpha}{2} \\
&= \frac{e^x - \alpha(1 + e^x)}{1 + e^x} + \frac{1 - \alpha}{2} \\
&= \frac{e^x}{1 + e^x} + \frac{1}{2} - \frac{3\alpha}{2}.
\end{align*}

Setting this equal to $0$, we get that $\alpha = \frac{3e^x + 1}{3e^x + 3}$ and $x = \log \frac{3\alpha - 1}{3 - 3\alpha}$. Since $x \in (0, 2/\tau]$, we have that if $\alpha \in \Big( \frac{2}{3}, \frac{3 \exp(2/\tau) + 1}{3 \exp(2 /\tau) + 3}\Big)$, there exists a local optima over $\theta \in (0, \pi/2]$.

Moreover, we observe that when $\alpha \le 2/3$, we have that $\frac{e^x}{1 + e^x} + \frac{1}{2} - \frac{3\alpha}{2} \ge \frac{e^x}{1 + e^x} - \frac{1}{2 } \ge 0$, which means that $L(\bm{\mu_\theta}, \alpha)$ increases in $\theta$ for $\alpha \le 2/3$. Therefore, when $\alpha \le 2/3$, class collapse is \emph{always better} no matter the angle, and $L(\bm{\delta_v}, \alpha) \le \min_{\theta} L(\bm{\mu_\theta}, \alpha)$.

Next, we consider when $\alpha \ge \frac{3\exp(2 /\tau) + 1}{3\exp(2 /\tau) + 3}$. In this case, $\frac{e^x }{1 + e^x} + \frac{1}{2} - \frac{3\alpha}{2} < 0$, so $L(\bm{\mu_\theta}, \alpha)$ is decreasing in $\theta$. This means that any nonzero $\theta$ in this setting is going to result in a smaller loss than the class-collapsed loss.

Lastly, we consider the intermediate case of $\alpha \in \Big( \frac{2}{3}, \frac{3 \exp(2/\tau) + 1}{3 \exp(2 /\tau) + 3}\Big)$. Plugging back in $x = \log \frac{3\alpha - 1}{3 - 3\alpha}$ back into $L(\bm{\mu_\theta}, \alpha)$ in~\eqref{eq:mu_theta_loss}, we have
\begin{align}
L(\bm{\mu_\theta}, \alpha) &= -\log 2 - \frac{2(1 - \alpha)}{\tau} + (1 - \alpha) \log \Big(1 + \frac{3\alpha - 1}{3 - 3\alpha} \Big) + \alpha \log \Big(1 + \frac{3 - 3\alpha}{3\alpha - 1} \Big) + \frac{1- \alpha}{2} \cdot \log \frac{3\alpha - 1}{3 - 3\alpha} \label{eq:mu_theta_intermediate} \\
&=-\log 2 - \frac{2(1 - \alpha)}{\tau} + (1 - \alpha) \log \frac{2}{3 - 3\alpha}  + \alpha \log \frac{2}{3\alpha - 1}  + \frac{1- \alpha}{2} \log(3\alpha - 1) - \frac{1 - \alpha}{2} \log(3 - 3\alpha) \nonumber \\
&= -\frac{2(1 - \alpha)}{\tau} - \frac{3 - 3\alpha}{2}\log (3 - 3\alpha) - \frac{3\alpha - 1}{2}  \log (3\alpha - 1). \nonumber
\end{align}

Note that $(3 - 3\alpha)\log (3 - 3\alpha) +(3\alpha - 1)  \log (3\alpha - 1)$ equals $0$ at $\alpha = 2/3$ and is increasing in $\alpha$. Therefore, we have that $-\frac{2(1 - \alpha)}{\tau} - \frac{3 - 3\alpha}{2}\log (3 - 3\alpha) - \frac{3\alpha - 1}{2}  \log (3\alpha - 1) \le -\frac{2(1 - \alpha)}{\tau} = L(\bm{\delta_v}, \alpha)$. Therefore, for $\alpha \in \Big( \frac{2}{3}, \frac{3 \exp(2/\tau) + 1}{3 \exp(2 /\tau) + 3}\Big)$, the optimal $\theta^\star$ satisfies $L(\bm{\mu_{\theta}}, \alpha) \le L(\bm{\delta_v}, \alpha)$. In particular, solving $\frac{2\sin^2 \theta^\star}{\tau} = \log \frac{3\alpha - 1}{3 - 3\alpha}$ gives us $\theta^\star = \arcsin \sqrt{ \frac{\tau}{2} \log \frac{3\alpha - 1}{3 - 3\alpha}}$.

Therefore, our analysis in these three ranges of $\alpha$ suggest that the optimal embedding geometry is \emph{not collapsed} when $\alpha \ge \frac{2}{3}$.

Next, we want to understand when the optimal embedding geometry is not $\bm{\sigma_{d-1}}$. A sufficient condition for this is to show that there exists an $\alpha \ge \frac{2}{3}$ where $\min_{\theta} L(\bm{\mu_{\theta}}, \alpha) \le L(\bm{\sigma_{d-1}}, \alpha)$. We first compute an upper bound on $\min_{\theta}L(\bm{\mu_{\theta}}, \alpha)$ for $\alpha > \frac{2}{3}$. Recall that our loss from~\eqref{eq:mu_theta_intermediate} can be written as
\begin{align}
L(\bm{\mu_{\theta}}, \alpha) = -\frac{2(1 - \alpha)}{\tau} - \log 2 - \frac{3 - 3\alpha}{2} \log \Big(\frac{3 - 3\alpha}{2} \Big) - \frac{3\alpha - 1}{2} \log \Big(\frac{3\alpha - 1}{2} \Big). \label{eq:mu_theta_alphas_only}
\end{align}

For ease of notation, let $x = \frac{3 - 3\alpha}{2} \in (0, 1)$. We show that $f(x) = x \log x + (1 - x) \log (1-x)$ can be lower bounded quadratically. Performing a Taylor expansion at $x = 0.5$, we have that $x \log x + (1 - x) \log (1 - x) \approx -\log 2 + 2(x-1/2)^2$. We claim that  $x \log x + (1 - x) \log (1 - x) \ge -\log 2 + 2(x-1/2)^2$. Note that the two sides are equal when $x = 1/2$, so proving this inequality is equivalent to showing that $f'(x) \ge 4(x - 1/2)$ for $x \ge 1/2$ and $f'(x) < 4(x - 1/2)$ for $x < 1/2$. $f'(x) = \log \frac{x}{1-x}$ is equal to $4(x - 1/2)$ at $x = 1/2$, so we want to show that $f''(x) \ge 4$ for all $x$. $f''(x) = \frac{1}{x(1 - x)}$ satisfies this inequality. Therefore, \eqref{eq:mu_theta_alphas_only} becomes
\begin{align*}
L(\bm{\mu_{\theta}}, \alpha) \le -\frac{2(1 - \alpha)}{\tau} - 2 \Big(1 - \frac{3\alpha}{2} \Big)^2.
\end{align*}

Next, we compute $L(\bm{\sigma_{d-1}}, \alpha)$. With $\bm{\sigma_{d-1}}$, $f(x), f(x^+),$ and $f(x^-)$ are all uniformly distributed on the hypersphere. Therefore,
\begin{align}
L(\bm{\sigma_{d-1}}, \alpha) = \E{x}{ \log \E{x^+}{\exp(-\|f(x) - f(x^+)\|^2 / 2\tau)}} - \frac{1 - \alpha}{2\tau} \int \int -\|u - u' \|^2 d \sigma_{d-1}(u) d \sigma_{d-1}(u'). \label{eq:sigma0}
\end{align}

From~\citet{wang2020understanding}, we know that when $f(x)$ and $f(x^+)$ are drawn from a distribution with measure $\bm{\sigma_{d-1}}$, it holds that $\E{x}{ \log \E{x^+}{\exp(-\|f(x) - f(x^+)\|^2 / 2\tau)}} = \log \E{x, x^+}{\exp(-\|f(x) - f(x^+)\|^2 / 2\tau)}$ .

Define the Gaussian $\frac{1}{2\tau}$-energy of $\sigma_{d-1}$ on $\mathcal{S}^{d-1}$ as
\ifarxiv
 $I_{\frac{1}{2\tau}}[\sigma_{d-1}] = \int\limits_{\mathcal{S}^{d-1}}\;\int\limits_{\mathclap{\mathcal{S}^{d-1}}}\exp \Big(-\frac{1}{2\tau} \|u - u' \|^2 \Big) d \sigma_{d-1}(u) d \sigma_{d-1}(u')$.
\else
 $I_{1/2\tau}[\sigma_{d-1}] = \int_{\mathcal{S}^{d-1}}\int_{\mathcal{S}^{d-1}}\exp \Big(-\frac{1}{2\tau} \|u - u' \|^2 \Big) d \sigma_{d-1}(u) d \sigma_{d-1}(u')$.
\fi
Then,~\eqref{eq:sigma0} becomes
\begin{align}
L(\bm{\sigma_{d-1}}, \alpha) = (1 - \alpha) \log I_{1/2\tau}[\sigma_{d-1}] + \alpha \log I_{1/2\tau}[\sigma_{d-1}] - \frac{1 - \alpha}{2\tau} \int \int -\|u - u' \|^2 d \sigma_{d-1}(u) d \sigma_{d-1}(u').  \label{eq:sigma}
\end{align}

From Theorem 4.6.5 of~\citet{borodachov2019discrete}, any measure $\mu$ that has mass centered at the origin, i.e. $\int u d\mu(u) = 0$, minimizes the energy $I_{-2}[\mu] = \int \int - \| u - u'\|^2 d\mu(u) d\mu(u')$. Therefore, $ \int \int -\|u - u' \|^2 d \sigma_{d-1}(u) d \sigma_{d-1}(u') $ is equivalent to the energy $I_{-2}[\frac{1}{2} \delta_{v_0} + \frac{1}{2} \delta_{v_1}] = -\frac{1}{2}(2^2) = -2$, since a measure on two points with probability $1/2$ each has mass centered at the origin. Plugging this back into $L(\bm{\sigma_{d-1}}, \alpha)$ in~\eqref{eq:sigma}, we have
\begin{align}
L(\bm{\sigma_{d-1}}, \alpha) = \log I_{1/2\tau}[\sigma_{d-1}] + \frac{1 - \alpha}{\tau}. \label{eq:sigma1}
\end{align}

From Proposition 4.4.1 and Theorem 6.2.1 of~\citet{borodachov2019discrete}, $\sigma_{d-1}$ is the unique equilibrium measure for the Gaussian $\frac{1}{2\tau}$ kernel on $\mathcal{S}^{d-1}$, and as a result $I_{1/2\tau}[\sigma_{d-1}]$ is equal to the Wiener constant $W_{1/2\tau}(\mathcal{S}^{d-1})$. By Proposition A.11.2 of~\cite{borodachov2019discrete}, this has the value
\begin{align}
W_{1/2\tau}(\mathcal{S}^{d-1}) = \frac{2^{d-2} \Gamma(d/2)}{\sqrt{\pi} \Gamma((d-1)/2)} \int_0^1 \exp\Big(-\frac{2u}{\tau}\Big) (u(1 - u))^{(d-3)/2} du, \label{eq:wiener}
\end{align}

where the Gamma function is $\Gamma(z) = \int_0^{\infty} x^{z-1} e^{-x} dx$ for $z > 0$.

Therefore, to prove that there exists a $\bm{\mu_{\theta}}$ that has lower loss than $\bm{\sigma_{d-1}}$, we must find $\alpha > 2/3$ that satisfies
\begin{align*}
2\Big(1 - \frac{3\alpha}{2} \Big)^2 + \frac{3(1 - \alpha)}{\tau} + \log W_{1/2\tau}(\mathcal{S}^{d-1}) \ge 0.
\end{align*}

This expression is quadratic in $\alpha$, and we solve it to get that $c(\tau, d) \le \frac{2 + \frac{1}{\tau} - \sqrt{\frac{1}{\tau}(-2 + \frac{1}{\tau}) - 2 \log W_{1/2\tau}(\mathcal{S}^{d-1})}}{3}$.

\end{proof}

\subsection{Proofs for Section~\ref{sec:theory}}
\label{subsec:supp_theory_theory}

\infiniteencoder*

\begin{proof}

We know that any $f^{\pi} \in \F$ can satisfy $f^{\pi}(x_i) = f(x_{\pi(i)})$, since the infinite encoder assumption means that $f^{\pi}$ can be arbitrarily fit to any data. Therefore, we only need to show that $\lspread$ does not change when $f^{\pi}$, which permutes within classes, is used instead of $f$.

For a given batch $B$ and $f$, $\lspread$ is constructed as defined in Section~\ref{subsec:contrastive_loss}. The numerator of $\supcon$ can be written as $\frac{1}{|B|} \frac{1}{|P(i, B)|} \sum_{i = 1}^{|B|} \sum_{x^+ \in P(i, B)} \log \sigma_f(x_i, x^+)$. This is a summation over the representations of all positive pairs in the batch. Therefore, a permutation $\pi$ within each class that changes the assignments to the representations will not change the value of this quantity, and $\frac{1}{|B|} \frac{1}{|P(i, B)|} \sum_{i = 1}^{|B|} \sum_{x^+ \in P(i, B)} \log \sigma_{f^{\pi}}(x_i, x^+) = \frac{1}{|B|} \frac{1}{|P(i, B)|} \sum_{i = 1}^{|B|} \sum_{x^+ \in P(i, B)} \log \sigma_f(x_i, x^+)$.

\ifarxiv
Next, the denominator of $\supcon$ can be written as $\frac{1}{|P(i, B)} \sum_{i = 1}^{|B|} \sum_{x^+ \in P(i, B)} \log \Big(\sigma_f(x_i, x^+) + \sum_{x^- \in N(i, B)} \newline \sigma_f(x_i, x^-) \Big)$. 
\else
Next, the denominator of $\supcon$ can be written as $\frac{1}{|P(i, B)} \sum_{i = 1}^{|B|} \sum_{x^+ \in P(i, B)} \log \Big(\sigma_f(x_i, x^+) + \sum_{x^- \in N(i, B)}  \sigma_f(x_i, x^-) \Big)$. 
\fi
Every single positive pair and negative pair is included in this expression, so this quantity is class-fixing permutation invariant.

The numerator of $\cnce$ is $\frac{1}{|B|} \sum_{i = 1}^{|B|} \log \sigma_f(x_i, a(x_i))$. Since an augmentation of $x_i$ is a function of $x_i$ and is disjoint from augmentations of other points, this quantity is class-fixing permutation invariant.

Lastly, the denominator of $\cnce$ is $\frac{1}{|B|} \sum_{i = 1}^{|B|} \log \Big( \sum_{x^+ \in p(i, B)} \sigma_f(x_i, x^+)\Big)$. From the same logic as the numerator of $\supcon$, any permutation within the class will still allow each $x_i$ to be compared with all other points in $x_i$'s class, hence being class-fixing permutation invariant.

Therefore, under the infinite encoder assumption where all $f^{\pi}$ are valid, $\lspread$ is permutation invariant.
\end{proof}

\begingroup
\def\thetheorem{\ref{thm:coarsetofine}}
\begin{theorem}
Denote $r_{f}(z, z') = c^2 \delta_f(z, z')^2 - |\Var{f}{z} - \Var{f}{z'}|$. With probability $1 - \delta$, the coarse-to-fine error is at most
\begin{align*}
L_{\gamma, f}(z) &\le \frac{\sigma_f(z)}{ \sqrt{r_f(z, z') - 2 \log \gamma}} +  \mathcal{O} \Big(\Big(\frac{d \log(d / \delta)}{m_z \wedge m_{z'}}\Big)^{1/4} \Big).
\end{align*}

under the boundary condition that $r_f(z, z') - 2 \log \gamma \ge 16 \sqrt{\frac{2d \log(8d/\delta)}{m_z \wedge m_{z'}}} + \frac{2d \log(8d/\delta)}{m_z}$.
\end{theorem}
\addtocounter{theorem}{-1}
\endgroup

\begin{proof}
From~\eqref{eq:gamma_loss}, our loss function can be written as $L_{\gamma, f}(z) = \Pr_{x \sim \p_z}(f(x)^\top (W_z - W_{z'}) \le \log \gamma)$. Note that $\|f(x) - W_z \|^2 = \| f(x)\|^2 + \|W_z\|^2 - 2 f(x)^\top W_z = 1 + \|W_z\|^2 - 2f(x)^\top W_z$, which means that $f(x)^\top W_z = \frac{1}{2} \Big(1 + \|W_z\|^2 - \|f(x) - W_z\|^2 \Big)$.  We can thus write our loss as  
\begin{align}
L_{\gamma, f}(z) &= \Pr_{x \sim \p_z}( f(x)^\top (W_z - W_{z'})\le \log \gamma ) \nonumber \\
&= \Pr_{x \sim \p_z}(\|f(x) - W_z \| \ge (\|f(x) - W_{z'} \|^2 - 2 \log \gamma + \|W_z\|^2 - \|W_{z'}\|^2)^{1/2}). \label{eq:c2f_loss}
\end{align}

We bound terms in this probability individually. First, we can write 
\begin{align}
\|f(x) - W_z\| \le \|f(x) - \E{x \sim \p_z}{f(x)}\| + \xi_z, \label{eq:c2f_easy}
\end{align}

where $\xi_z$ again is constructed as $\xi_z = \|W_z - \E{x \sim \p_z}{f(x)} \| =  \big\|\frac{1}{m_z} \sum_{x \in \D_{s, z}} f_{SC}(x) - \E{x \sim \p_z}{f_{SC}(x)}\big\|$.

Next, by the reverse triangle inequality we can write 
\begin{align}
\|f(x) - W_{z'}\|^2 &\ge \big|\|f(x) - \E{x \sim \p_{z'}}{f(x)} \| - \|\E{x \sim \p_{z'}}{f(x)} - W_{z'} \| \big|^2 \nonumber \\
&\ge \|f(x) - \E{x \sim \p_{z'}}{f(x)} \|^2 - 2 \| f(x) - \E{x \sim \p_{z'}}{f(x)}\| \xi_{z'} \nonumber \\
&\ge  \|f(x) - \E{x \sim \p_{z'}}{f(x)} \|^2 - 4 \xi_{z'} \nonumber \\
&\ge c^2 \cdot \E{x \sim \p_z, x' \sim \p_{z'}}{\|f(x) - f(x')\|}^2 - 4 \xi_{z'}. \label{eq:c2f1}
\end{align}

Note the following decomposition by Jensen's inequality:
\begin{align*}
s_f(y) &= \E{h(x) = y}{\|f(x) - \E{h(x) = y}{f(x)} \|} \le p(z|y)^2 \E{x \sim \p_z}{\| f(x) - \E{x \sim \p_z}{f(x)}\|} \\
&+ p(z'|y)^2 \E{x\sim \p_{z'}}{\|f(x) - \E{x \sim \p_{z'}}{f(x)} \|} + p(z | y) p(z' | y) \E{x \sim \p_z}{\|f(x) - \E{x' \sim \p_{z'}}{f(x')} \|} \\
&+ p(z | y) p(z' | y) \E{x \sim \p_{z'}}{\|f(x) - \E{x' \sim \p_z}{f(x')} \|}\\
&\le p(z|y)^2 \sigma_f(z) + p(z' | y)^2 \sigma_f(z') + 2 p(z|y) p(z' |y) \E{x \sim \p_z, x' \sim \p_{z'}}{\|f(x) - f(x') \|},
\end{align*}

and thus, $\E{x \sim \p_z, x' \sim \p_{z'}}{\|f(x) - f(x')\|}  \ge \delta_f(z, z') = \frac{1}{p(z|y)p(z'|y)} \cdot \Big(s_f(y) - p(z|y)^2 \sigma_f(z) - p(z' | y)^2 \sigma_f(z') \Big)$. Then,~\eqref{eq:c2f1} becomes
\begin{align}
\|f(x) - W_{z'}\|^2 &\ge c^2 \cdot \delta_f(z, z')^2 - 4 \xi_{z'}.\label{eq:c2f3}
\end{align}

Finally, we bound $\|W_z\|^2 - \|W_{z'}\|^2$. Recall that $\|W_z\|^2 = \Big\|\frac{1}{m_z}\sum_{x \in \D_{s, z}} f(x) \Big\|^2 $. We can write the empirical variance $\frac{1}{m_z} \sum_{x \in \D_{s,z}}\| f(x) - \frac{1}{m_z} \sum_{x \in \D_{s,z}}f(x) \|^2$ as $1 + \|W_z \|^2  - \frac{1}{m_z} \sum_{x \in \D_{s,z}} f(x)^\top \frac{2}{m_z} \sum_{x \in \D_{s,z}} f(x) = 1 - \|W_z\|^2$, and therefore,
\begin{align}
\big|\|W_z\|^2 - \|W_{z'}\|^2\big| &\le \bigg|\frac{1}{m_z} \sum_{x \in \D_{s,z}}\| f(x) - \frac{1}{m_z} \sum_{x \in \D_{s,z}}f(x) \|^2 - \frac{1}{m_{z'}} \sum_{x \in \D_{s,z'}}\| f(x) - \frac{1}{m_{z'}} \sum_{x \in \D_{s,{z'}}}f(x) \|^2 \bigg|\nonumber \\
&\le \bigg|\frac{1}{m_{z'}} \sum_{x \in \D_{s,z'}}\bigg\| f(x) - \frac{1}{m_{z'}} \sum_{x \in \D_{s, {z'}}}f(x) \bigg\|^2 -\E{x \sim \p_{z'}}{\|f(x) - \E{x \sim \p_{z'}}{f(x)} \|^2} \bigg| \label{eq:c2f2} \\
&+ \bigg|\frac{1}{m_{z}} \sum_{x \in \D_{s,z}}\bigg\| f(x) - \frac{1}{m_{z}} \sum_{x \in \D_{s, {z}}}f(x) \bigg\|^2 -\E{x \sim \p_{z}}{\|f(x) - \E{x \sim \p_{z}}{f(x)} \|^2} \bigg| \nonumber \\
&+ \big|\Var{f}{z} - \Var{f}{z'} \big|. \nonumber 
\end{align}

where $\Var{f}{z} = \E{x \sim \p_z}{\|f(x) - \E{x \sim \p_z}{f(x)} \|^2}$, and $\Var{f}{z'}$ is similarly defined. We decompose the first term in~\eqref{eq:c2f2} and bound it by 
\begin{align*}
&\bigg|\frac{1}{m_{z'}} \sum_{x \in \D_{s,z'}}\bigg\| f(x) - \frac{1}{m_{z'}} \sum_{x \in \D_{s, {z'}}}f(x) \bigg\|^2 - \frac{1}{m_{z'}} \sum_{x \in \D_{s,z'}}\big\| f(x) - \E{x \sim \p_{z'}}{f(x)}\big\|^2    \bigg| \\
+& \bigg| \frac{1}{m_{z'}} \sum_{x \in \D_{s,z'}}\big\| f(x) - \E{x \sim \p_{z'}}{f(x)}\big\|^2 - \E{x\sim \p_{z'}}{\|f(x) - \E{x \sim \p_{z'}}{f(x)} \|^2}    \bigg| \\
=& \frac{1}{m_{z'}} \sum_{x \in \D_{s, z'}} \bigg|  \bigg\|f(x) - \frac{1}{m_{z'}} \sum_{x \in \D_{s, z'}} f(x) \bigg\|^2 - \big\|f(x) - \E{x \sim \p_{z'}}{f(x)} \big\|^2  \bigg| + \zeta_{z'} \\
\le &\frac{1}{m_{z'}} \sum_{x \in \D_{s, z'}} 4 \bigg|  \bigg\|f(x) - \frac{1}{m_{z'}} \sum_{x \in \D_{s, z'}} f(x) \bigg\| - \big\|f(x) - \E{x \sim \p_{z'}}{f(x)} \big\|  \bigg| + \zeta_{z'} \\
\le& \frac{1}{m_{z'}} \sum_{x \in \D_{s, z'}} 4 \bigg\|  \frac{1}{m_{z'}} \sum_{x \in \D_{s, z'}} f(x) - \E{x \sim \p_{z'}}{f(x)}  \bigg\| + \zeta_{z'} \\
\le& 4\xi_{z'} + \zeta_{z'}.
\end{align*}

where $\zeta_{z'} = \Big| \frac{1}{m_{z'}} \sum_{x \in \D_{s,z'}}\Big\| f(x) - \E{x \sim \p_{z'}}{f(x)}\Big\|^2 - \E{x\sim \p_{z'}}{\|f(x) - \E{x \sim \p_{z'}}{f(x)} \|^2} \Big|$ can be bounded by standard concentration inequalities. Therefore, $\big| \|W_z\|^2 - \|W_{z'}\|^2 \big|$ is bounded by
\begin{align}
\big|\|W_z \|^2 - \|W_{z'} \|^2 \big| \le |\Var{f}{z} - \Var{f}{z'}| + 4\xi_z + 4 \xi_{z'} + \zeta_{z} + \zeta_{z'}. \label{eq:w_bound}
\end{align}

Then, combining~\eqref{eq:c2f_easy},~\eqref{eq:c2f3}, and~\eqref{eq:w_bound}, the loss in~\eqref{eq:c2f_loss} becomes
\begin{align}
L_{\gamma, f}(z) \le \Pr_{x \sim \p_z}(\|f(x) - \E{x \sim \p_z}{f(x)} \| \ge \big( r_f(z, z')- 8 \xi_{z'} - 4 \xi_{z} - \zeta_{z} - \zeta_{z'} - 2 \log \gamma \big)^{1/2} - \xi_z). \label{eq:c2f_loss2}
\end{align}

Next, we note that $\xi_z \le \sqrt{\frac{2d \log (2d/\delta)}{m_z}}$ with probability at least $1 - \delta$. Applying Hoeffding's inequality on $\zeta_z$ gives us $\zeta_z \le \sqrt{\frac{8 \log(2/\delta)}{m_z}}$ with probability at least $1 - \delta$. Applying a union bound, we have that with probability $1 - \delta$,~\eqref{eq:c2f_loss2} satisfies
\ifarxiv
\begin{align*}
&L_{\gamma, f}(z) \le  \\
&\Pr_{x \sim \p_z}\Bigg(\|f(x) - \E{x \sim \p_z}{f(x)} \| \ge \bigg( r_f(z, z')- 10\sqrt{\frac{2d \log(8d /\delta)}{m_{z'}}}   - 6 \sqrt{\frac{2d \log(8d/\delta)}{m_z}} - 2 \log \gamma \bigg)^{1/2} \\
&- \sqrt{\frac{2d \log(8d / \delta)}{m_z}}\Bigg).
\end{align*}
\else
\begin{align*}
&L_{\gamma, f}(z) \le  \\
&\Pr_{x \sim \p_z}\bigg(\|f(x) - \E{x \sim \p_z}{f(x)} \| \ge \bigg( r_f(z, z')- 10\sqrt{\frac{2d \log(8d /\delta)}{m_{z'}}}   - 6 \sqrt{\frac{2d \log(8d/\delta)}{m_z}} - 2 \log \gamma \bigg)^{1/2} - \sqrt{\frac{2d \log(8d / \delta)}{m_z}}\bigg).
\end{align*}
\fi

Finally, we can apply Markov's inequality under the condition that $r_f(z, z') - 2 \log \gamma \ge \frac{2d \log(8d / \delta)}{m_z} + 6 \sqrt{\frac{2d \log(8d/\delta)}{m_z}} + 10\sqrt{\frac{2d \log(8d /\delta)}{m_{z'}}}$, or equivalently $r_f(z, z') - 2 \log \gamma \ge \frac{2d \log(8d / \delta)}{m_z} + 16 \sqrt{\frac{2d \log(8d/\delta)}{m_z \wedge m_{z'}}}$. With probability at least $1- \delta$, our loss is bounded by
\begin{align*}
L_{\gamma, f}(z) &\le \frac{\E{x \sim \p_z}{f(x) - \E{x \sim \p_z}{f(x)}}}{\sqrt{r_f(z, z') - 2 \log \gamma - 16\sqrt{\frac{2d \log(8d/\delta)}{m_z \wedge m_{z'}}}} - \sqrt{\frac{2d \log(8d/\delta)}{m_z}}} \\
&\le \frac{\sigma_f(z)}{\sqrt{r_f(z, z') - 2 \log \gamma}} + \mathcal{O}\bigg( \bigg(\frac{d \log (d/\delta)}{m_z \wedge m_{z'}} \bigg)^{1/4}\bigg).
\end{align*}

\end{proof}

\lipschitz*

\begin{proof}
Using Jensen's inequality and then Lipschitzness of $f$,
\begin{align*}
\sigma_f(z) &= \E{x \sim \p_z}{\|f(x) - \E{x \sim \p_z}{f(x)}\|} = \E{x \sim \p_z}{ \Big\|\int (f(x) -  f(x')) p(x'|z) dx' \Big\|}\\
&\le  \E{x \sim \p_z}{ \int \|f(x) -  f(x')\| p(x'|z) dx'} \le K_L \cdot \E{x \sim \p_z}{ \int \|x -  x'\| p(x'|z) dx'} \\
&= K_L \cdot \E{x, x' \sim \p_z}{\|x - x' \|} = K_L \delta_z.
\end{align*}
\end{proof}

\begingroup
\def\thelemma{\ref{lemma:autoencoder}}
\begin{lemma}
For any $g \in \G$, suppose there exists a $K_g > 0$ such that $g$ is ``reverse Lipschitz'', satisfying $\|f_{AE}(x) - f_{AE}(x') \| \le K_g \|g(f_{AE}(x)) - g(f_{AE}(x))\|$, and there exists finite $b$ such that the reconstruction loss satisfies $\max_x \|g(f_{AE}(x)) - x \|^2 \le b$. 

Then 
with probability at least $1 - \delta$,
\begin{align}
\sigma_{f_{AE}}(z)  &\le \frac{2K_g}{p(z | y)}\bigg(\hat{L}_{AE}(\D_y) + 2\rademacher^2_{n_y}(\G \circ \F_{AE}, \text{id}_{\X}) + b \sqrt{\frac{\log (1/ \delta)}{2n_y}} \bigg)^{1/2}  + K_g \sigma_z,\nonumber
\end{align}
where $\text{id}_{\X}$ is the identity function on $\X$, and $p(z | y) = \frac{p(z)}{p(y)}$ is the probability that $x$ drawn from $p(\cdot | y)$ has label $z$.
\end{lemma}
\addtocounter{lemma}{-1}
\endgroup

\begin{proof}
We can decompose $\sigma_{f_{AE}}(z)$ into the following using the assumption on the decoder $g$:
\begin{align}
\sigma_{f_{AE}}(z) &= \E{x \sim \p_z}{\|f_{AE}(x) - \E{x' \sim \p_z}{f_{AE}(x')} \|} \le \E{x\sim \p_z}{ \int \|f_{AE}(x) - f_{AE}(x')\| p(x' | z) dx'} \nonumber \\
&\le \E{x \sim \p_z}{\int K_g \cdot \|g(f_{AE}(x)) - g(f_{AE}(x'))\| p(x' | z) dx'} \nonumber \\
&\le \E{x \sim \p_z}{\int K_g \big( \|g(f_{AE}(x)) - x\| + \|x - x'\| + \|x' - g(f_{AE}(x'))\|\big) p(x' | z) dx'} \nonumber \\
&= 2 K_g \E{x \sim \p_z}{\|g(f_{AE}(x)) - x \|} + K_g \E{x, x' \sim \p_z}{\|x - x' \|}.\label{eq:autoencoder1}
\end{align}

\ifarxiv
Note that $\E{x \sim \p_y}{\|g(f_{AE}(x)) - x\|} = \sum_{k \in S_y} \E{x \sim \p_{k}}{\|g(f_{AE}(x)) - x \|} p(k | y) \ge \E{x \sim \p_z}{\|g(f_{AE}(x)) - x \|}\newline \times p(z | y)$, so~\eqref{eq:autoencoder1} becomes
\else
Note that $\E{x \sim \p_y}{\|g(f_{AE}(x)) - x\|} = \sum_{k \in S_y} \E{x \sim \p_{k}}{\|g(f_{AE}(x)) - x \|} p(k | y) \ge \E{x \sim \p_z}{\|g(f_{AE}(x)) - x \|}p(z | y)$, so~\eqref{eq:autoencoder1} becomes
\fi
\begin{align}
\sigma_{f_{AE}}(z) &\le \frac{2K_g}{p(z | y)} \E{x \sim \p_y}{\|g(f_{AE}(x)) - x \|} + K_g \E{x, x' \sim \p_z}{\|x - x' \|} \nonumber \\
&\le \frac{2K_g}{p(z | y)} \sqrt{\E{x \sim \p_y}{\|g(f_{AE}(x)) - x \|^2}} + K_g \E{x, x' \sim \p_z}{\|x - x' \|} \nonumber \\
&=\frac{2K_g}{p(z | y)}\sqrt{\hat{L}_{AE}(n_y) + \E{x \sim \p_y}{\|g(f_{AE}(x)) - x \|^2} - \hat{L}_{AE}(n_y) }  + K_g \sigma_z. \label{eq:autoencoder2}
\end{align}

where $\hat{L}_{AE}(\D_y) = \frac{1}{n_y} \sum_{x \in \D_y} \|g(f_{AE}(x)) - x \|^2$ is the reconstruction error on the training data $\D_y$. We bound the generalization error $\E{x \sim \p_y}{\|g(f_{AE}(x)) - x \|^2} - \hat{L}_{AE}(\D_y)$ using Theorem 3.3 of~\citet{mohri2018foundations} to get that with probability at least $1 - \delta$,
\begin{align*}
\sigma_{f_{AE}}(z)  \le \frac{2K_g}{p(z | y)}\bigg(\hat{L}_{AE}(\D_y) + 2\rademacher^2_{n_y}(\G \circ \F_{AE}, \text{id}_{\X}) + b \sqrt{\frac{\log (1/ \delta)}{2n_y}} \bigg)^{1/2}  + K_g \sigma_z.
\end{align*}

Finally, we compare against using a general autoencoder trained on the entire dataset of $n$ points. This yields a bound
\begin{align*}
\sigma_{f_{AE}}(z)  \le \frac{2K_g}{p(z)}\bigg(\hat{L}_{AE}(\D) + 2\rademacher^2_{n}(\G \circ \F_{AE}, \text{id}_{\X}) + b \sqrt{\frac{\log (1/ \delta)}{2n}} \bigg)^{1/2}  + K_g \sigma_z,
\end{align*}

where the only change in the result is that $p(z | y)$ is replaced with $p(z)$, the overall proportion of the subclass, and $n_y$ is replaced with $n$. This highlights a tradeoff: $p(z | y) > p(z)$, but $n_y < n$. A class-conditional autoencoder may suffer from poorer generalization due to lower sample size, but its relative worst case performance on $z$ in expectation is better. On the other hand, a general autoencoder is learned on more data, but its relative worst case performance on $z$ in expectation is worse since the subclass is more rare w.r.t. the training data.

\end{proof}

\begingroup
\def\thelemma{\ref{lemma:augmentation}}
\begin{lemma}
For $a \in \A$ and any $x, x' \in \X$, suppose that $f_{aug} \in \F_{aug}$ satisfies $\|f_{aug}(a(x)) - f_{aug}(a(x'))\| \le K_{aug} \|a(x) - a(x') \|$ for some $K_{aug}$ and that $f(a(x)) = f(x)$ for $x \in \D$. Denote $\sigma_z^{aug} = \E{x, x' \sim \p_z}{\|a(x) - a(x') \|}$. Then 
with probability at least $1 - \delta$,
\begin{align}
\sigma_{f_{aug}}(z) &\le \frac{2}{p(z)} \bigg(2 \rademacher^1_{n}(\F_{aug}, \F_{aug} \circ \A) + \sqrt{\frac{2 \log(1 /\delta)}{n}} \bigg) + K_{aug} \sigma_z^{aug}. \nonumber
\end{align}
\end{lemma}
\addtocounter{lemma}{-1}
\endgroup

\begin{proof}
We can decompose $\sigma_{f_{aug}}(z)$ into
\begin{align}
&\E{x \sim \p_z}{\|f_{aug}(x) - \E{x' \sim \p_z}{f_{aug}(x')} \|} \le \E{x\sim \p_z}{\int \|f_{aug}(x) - f_{aug}(x') \| p(x' | z) dx'} \nonumber \\
&\le \mathbb{E}_{x\sim \p_z}\bigg[\int \Big(\|f_{aug}(x) - f_{aug}(a(x)) \| + \|f_{aug}(a(x)) - f_{aug}(a(x'))\| + \|f_{aug}(a(x')) - f_{aug}(x') \|\Big) p(x' | z) dx'\bigg] \nonumber \\
&\le 2\E{x\sim \p_z}{\|f_{aug}(x) - f_{aug}(a(x))\|}  + K_{aug} \sigma^{aug}_z. \label{eq:aug1}
\end{align}

We can bound $\E{x\sim \p_z}{\|f_{aug}(x) - f_{aug}(a(x))\|} \le \frac{1}{p(z)} \E{}{\|f_{aug}(x) - f_{aug}a(x) \|}$. We assume that the encoder is able to satisfy $f_{aug}(x) = f_{aug}(a(x))$ for all training data $x \in \D$, so~\eqref{eq:aug1} becomes
\begin{align}
\sigma_{f_{aug}}(z) \le \frac{2}{p(z)}\Big(\E{x}{\|f_{aug}(x) - f_{aug}(a(x))\|}  - \frac{1}{n} \sum_{i = 1}^n \|f(x_i) - f(a(x_i))\| \Big) + K_{aug} \sigma^{aug}_z. \label{eq:aug2}
\end{align}

Then, using Theorem 3.3 from~\citet{mohri2018foundations}, with probability at least $1 - \delta$
\begin{align*}
\E{}{\|f_{aug}(x) - f_{aug}(a(x)) \|} - \frac{1}{n}\sum_{i = 1}^n \|f_{aug}(x_i) - f_{aug}(a(x_i)) \| \le 2 \rademacher^1_{n}(\F_{aug}, \F_{aug} \circ \A) + \sqrt{\frac{2\log(1/\delta)}{n}}.
\end{align*}

Therefore,~\eqref{eq:aug2} becomes
\begin{align*}
\sigma_f(z) \le \frac{2}{p(z)} \bigg(2 \rademacher^1_n(\F_{aug}, \F_{aug} \circ \A) + \sqrt{\frac{2 \log(1 /\delta)}{n}} \bigg) + K_{aug} \sigma^{aug}_z.
\end{align*}

\end{proof}

\section{Additional Theoretical Results}\label{sec:supp_additional_theory}

\subsection{Optimal $\lspread$ geometry for $K = 3$}

We provide a proof sketch that when there are $K = 3$ classes and $d \ge 3$, there exists a distribution $\bm{\mu_{\theta}}$ that obtains lower loss than the uniform or collapsed distributions. Synthetic experiments for this setting are in Appendix~\ref{sec:supp_synthetics} (see Figure~\ref{fig:sphere}).

For simplicity, let's consider when $d = 3$. Without loss of generality, denote the $3$-simplex as $v_0 = [1, 0, 0]$, $v_1 = [-1/2, 0, \sqrt{3}/2]$, $v_2 = [-1/2, 0, -\sqrt{3}/2]$. 
We will perform a rotation in the ``free'' dimension ($2$), in particular rotating $v_0$ by $\theta$ in the direction orthogonal to the subspace that the simplex is in.

In particular, we construct the rotation matrix $R_{\theta} = \begin{bmatrix}
\cos \theta & - \sin \theta & 0 \\
\sin \theta & \cos \theta & 0 \\
0 & 0 & 1
\end{bmatrix}$. Then,
\begin{align*}
R_{\theta} v_0 = \begin{bmatrix}
\cos \theta \\ \sin \theta \\ 0 
\end{bmatrix} \quad R_{\theta} v_1 = \begin{bmatrix}
-\cos\theta /2 \\ -\sin \theta/2 \\ \sqrt{3}/2
\end{bmatrix} \quad R_{\theta} v_2 = \begin{bmatrix}
-\cos \theta/2 \\ -\sin \theta /2 \\ -\sqrt{3}/2
\end{bmatrix},
\end{align*}

and we make a distribution where $\mu_{0, \theta} = \frac{1}{2} \delta_{v_0} + \frac{1}{2} \delta_{R_{\theta} v_0}$, $\mu_{1, \theta} = \frac{1}{2} \delta_{v_1} + \frac{1}{2} \delta_{R_{\theta} v_1}$, $\mu_{2, \theta} = \frac{1}{2} \delta_{v_2} + \frac{1}{2} \delta_{R_{\theta} v_2}$. That is, similar to the binary setting, we take a mixture of a simplex and that simplex rotated by $\theta$ in a dimension orthogonal to its subspace.

Now, we want to compute what the loss is. Recall that our asymptotic loss function is
\begin{align}
\lspread(\bm{\mu}, \alpha) &= (1 - \alpha) \E{x}{\log \E{x^-}{ \exp \Big(-\frac{1}{2\tau} \|f(x) - f(x^-) \|^2 \Big)}} \label{eq:k3loss} \\
&+ \alpha \E{x}{\log \E{x^+}{ \exp\Big(-\frac{1}{2\tau} \|f(x) - f(x^+) \|^2 \Big)}} + (1 - \alpha) \E{x, x^+}{\frac{1}{2\tau} \|f(x) - f(x^+) \|^2}. \nonumber 
\end{align}

For the class collapsed embeddings, note that for $K = 3$ the simplex side length is $\sqrt{3}$. Therefore,
\begin{align*}
\lspread(\bm{\delta_v}, \alpha) = (1 - \alpha) \log \exp \Big(-\frac{1}{2\tau} \cdot 3 \Big) = \frac{-3 \cdot (1 - \alpha)}{2\tau}.
\end{align*}

Next, we compute the loss for our intermediate distribution. We note the following:
\begin{align*}
&\| v_0 - R_{\theta} v_1 \|^2 = \| R_{\theta} v_0 - v_1 \|^2 = \|v_0 - R_{\theta}v_2 \|^2 = \| R_{\theta} v_0 - v_2 \|^2 = \Big(1 + \frac{\cos \theta}{2}\Big)^2 + \frac{\sin^2 \theta}{4}+ \frac{3}{4} \\
&= 2 + \cos \theta \\
&\|v_1 - R_{\theta} v_2 \| = \|R_{\theta} v_1 - v_2 \| = \Big(-\frac{1}{2} + \frac{\cos \theta}{2} \Big)^2 + \frac{\sin^2 \theta}{4} + 3 = \frac{7 - \cos \theta}{2}\\
&\| v_0 - R_{\theta} v_0 \| = (1 - \cos \theta)^2 + \sin^2 \theta = 2 - 2 \cos \theta \\
&\| v_1 - R_{\theta} v_1 \| = \| v_2 - R_{\theta} v_2 \| = \Big(-\frac{1}{2} + \frac{\cos \theta}{2} \Big)^2 + \frac{\sin^2 \theta}{4} = \frac{1 - \cos \theta}{2}
\end{align*}

and recall that $\|v_i - v_j \|^2 = \|R_{\theta}v_i - R_{\theta}v_j \|^2 = 3$. Plugging these back in, we have
\ifarxiv
\begin{align*}
\mathbb{E}_{x}[\log \mathbb{E}_{x^-}[\exp(-\|f(x) &- f(x^-)\|^2 / 2\tau)]] = \frac{1}{3} \log \bigg(\frac{1}{2}\exp\bigg(-\frac{3}{2\tau} \bigg) + \frac{1}{2} \exp\bigg(-\frac{2 + \cos \theta}{2\tau} \bigg) \bigg) \\
&+ \frac{2}{3} \log \bigg(\frac{1}{2}\exp\bigg(-\frac{3}{2\tau} \bigg) + \frac{1}{4}\exp\bigg(-\frac{7 - \cos \theta}{4\tau} \bigg) + \frac{1}{4}\exp\bigg(-\frac{2 + \cos \theta}{2\tau} \bigg) \bigg) \\
\mathbb{E}_{x}[\log \mathbb{E}_{x^+}[\exp(-\|f(x) &- f(x^+)\|^2 / 2\tau)]] = \frac{1}{3} \log \bigg(\frac{1}{2} + \frac{1}{2} \exp \bigg(-\frac{1 - \cos \theta}{\tau} \bigg) \bigg) \\
&+ \frac{2}{3}\log \bigg(\frac{1}{2} + \frac{1}{2}\exp\bigg(-\frac{1 - \cos \theta}{4\tau} \bigg) \bigg) \\
\mathbb{E}_{x, x^+}\bigg[\frac{1}{2\tau} \|f(x) - f(x^+&) \|^2\bigg] = \frac{1}{3} \cdot \frac{1 - \cos \theta}{4\tau} + \frac{1}{6} \cdot \frac{1 - \cos \theta}{\tau} = \frac{1 - \cos \theta}{4\tau}
\end{align*}
\else
\begin{align*}
\E{x}{\log \E{x^-}{\exp(-\|f(x) - f(x^-)\|^2 / 2\tau)}} &= \frac{1}{3} \log \bigg(\frac{1}{2}\exp\bigg(-\frac{3}{2\tau} \bigg) + \frac{1}{2} \exp\bigg(-\frac{2 + \cos \theta}{2\tau} \bigg) \bigg) \\
&+ \frac{2}{3} \log \bigg(\frac{1}{2}\exp\bigg(-\frac{3}{2\tau} \bigg) + \frac{1}{4}\exp\bigg(-\frac{7 - \cos \theta}{4\tau} \bigg) + \frac{1}{4}\exp\bigg(-\frac{2 + \cos \theta}{2\tau} \bigg) \bigg) \\
\E{x}{\log \E{x^+}{\exp(-\|f(x) - f(x^+)\|^2 / 2\tau)}} &= \frac{1}{3} \log \bigg(\frac{1}{2} + \frac{1}{2} \exp \bigg(-\frac{1 - \cos \theta}{\tau} \bigg) \bigg) + \frac{2}{3}\log \bigg(\frac{1}{2} + \frac{1}{2}\exp\bigg(-\frac{1 - \cos \theta}{4\tau} \bigg) \bigg) \\
\E{x, x^+}{\frac{1}{2\tau} \|f(x) - f(x^+) \|^2} &= \frac{1}{3} \cdot \frac{1 - \cos \theta}{4\tau} + \frac{1}{6} \cdot \frac{1 - \cos \theta}{\tau} = \frac{1 - \cos \theta}{4\tau}
\end{align*}
\fi

We use the above expressions to simplify~\eqref{eq:k3loss} and numerically check that there exists $\theta$ for $\alpha \gtrsim 0.6$ such that $L(\bm{\mu_{\theta}}, \alpha) \le L(\bm{\delta_v}, \alpha)$. We then numerically check there exists $(\theta, \alpha)$ that also satisfies $L(\bm{\mu_{\theta}}, \alpha) \le L(\bm{\sigma_{d-1}}, \alpha)$, where $L(\bm{\sigma_{d-1}}, \alpha)$ is defined in~\eqref{eq:sigma1}.

\subsection{Permutation Invariance} This result is a simple example of a sufficient condition under which $\lspread$ does not exhibit class-fixing permutation invariance.

\begin{lemma}
Let $\phi: \mathbb{R}^+ \rightarrow \mathbbm{R}^+$ be a monotonically increasing function. Suppose that for $x, x' \in \X$, all $f \in \F$ satisfy $\|f(x) - f(x') \| \le \phi(\|x - x'\|)$. Then, $\lspread$ is not invariant on class-fixing permutations under $\F$. 
\label{lemma:perm_invariance}
\end{lemma}

\begin{proof}
First, note that $\lspread$ in \eqref{eq:sup} has terms in the numerator of $\supcon$ of the form $\|f(x) - f(x^+)\|$, where $h(x) = h(x^+)$. We show how to break permutation invariance using this quantity.

Fix two vectors in the embedding space $\mathcal{S}^{d-1}$, $u_a$ and $u_b$. For $x_1, x_2$ in a given class, suppose that $f(x_1) = u_a$ and $f(x_2) = u_b$. We select a third point $x_3$ that is very close to $x_1$, satisfying $\phi(\|x_3 - x_1|)  < \|u_a - u_b\|$ (this property must hold for some $x_3$ since $\phi$ is monotonic).

We construct a permutation $\pi$ where $\pi(1) = 3, \pi(2) = 1, \pi(3) = 1$. We know that $\|u_a - u_b \| \le \phi(\|x_1 - x_2\|)$. Suppose that the mapping $f^\pi$ satisfies $\|f^\pi(x_3) - f^\pi(x_1) \| = \|u_a - u_b\|$. However, this implies that $\|u_a - u_b\| \le \phi(\|x_3 - x_1\|)$, which is a contradiction. Therefore, no $f^\pi \in \F$ exists that is able to map the permutation to the same value as $f$ does. As this holds for a single term in $\lspread$, it applies to $\lspread$ overall, demonstrating that such an assumption on $\F$ (which we find is true for a Lipschitz encoder, the autoencoder, and data augmentations) is able to break permutation invariance.

\end{proof}

\section{Auxiliary Lemmas}\label{sec:supp_aux_lemmas}

\begin{lemma}
Under the infinite encoder assumption, the following statement holds for $L_{\text{diff}}(f)$:
\begin{align*}
\min_{\bm{\mu} \in \{\mathcal{S}^{d-1} \}^K} \E{x}{\log \E{x^-}{\exp(- \|f(x) - f(x^-) \|^2/2\tau)}} \equiv \min_{\bm{\mu} \in \{\mathcal{S}^{d-1} \}^K} \log \E{x, x^-}{\exp(- \|f(x) - f(x^-) \|^2/2\tau)}.
\end{align*}
\label{lemma:jensen_swap}
\end{lemma}

\begin{proof}

Conditioning on the label of $x$ and using the definition of $x^-$ for $K = 2$, we can write $L_{\text{diff}}(f)$ as
\begin{align*}
L_{\text{diff}}(f) &= \E{y}{\E{x| h(x) = y}{\log \E{x^-| h(x^-) \neq y }{\exp(\sigma_f(x, x^-))}}} \\
&= \frac{1}{2} \int p(x | y = 0) \bigg(\log \int p(x^- | y = 1) \exp(\sigma_f (x, x^-)) dx^- \bigg) dx \\
&+ \frac{1}{2} \int p(x | y = 1) \bigg(\log \int p(x^- | y = 0) \exp(\sigma_f (x, x^-)) dx^- \bigg) dx .
\end{align*}

Since the encoder is assumed to be infinitely powerful, we optimize over the class-conditional measures $\mu_0$ and $\mu_{1}$ in $\mathcal{M}(\mathbb{S}^{d-1})$, the set of Borel probability measures on $\mathbb{S}^{d-1}$. The optimization problem is now
\begin{align*}
\text{minimize}_{\mu_0, \mu_1} \;\;\; & \int  \bigg(\log \int \exp(\sigma(x, x^-)) d\mu_1(x) \bigg) d\mu_0(x) +  \int  \bigg(\log \int \exp(\sigma(x, x^-)) d\mu_0(x) \bigg) d\mu_1(x).
\end{align*}

Next, define
\begin{align*}
U_{\mu}(u) = \int \exp(u^\top v / \tau) d\mu(v).
\end{align*}

The expression we want to minimize is thus
\begin{align}
&\text{minimize}_{\mu_i, \mu_{-i}} \int \log U_{\mu_1}(u) d\mu_0(u) + \int \log U_{\mu_0}(u) d\mu_1(u). \label{eq:wangisola_opt}
\end{align}

Following the approach of~\citet{wang2020understanding}, we analyze the measures $\mu_0^\star, \mu_1^\star$ that minimize this expression in two steps. First, we show that the minimum of~\eqref{eq:wangisola_opt} exists, i.e. the infimum is attained for some two measures. Second, we show that $U_{\mu_0^\star}$ is constant $\mu_1^\star$-almost surely, and vice versa. This will allow us to interchange the outer expectation over $x$ and the $\log$ in $L_{\text{diff}}(f)$.

\begin{enumerate}[leftmargin=*]
\item \textbf{Minimizers of~\eqref{eq:wangisola_opt} exist.} 

Let $m$ be a sequence such that
\begin{align*}
&\lim_{m \rightarrow \infty}  \int \log U_{\mu_1^m}(u) d\mu_0^m(u) +  \int \log U_{\mu_0^m}(u) d\mu_1^m(u) \\
= &\inf_{\mu_0, \mu_1} \int \log U_{\mu_1}(u) d\mu_0(u) + \int \log U_{\mu_0}(u) d\mu_1(u).
\end{align*}

Using Helly's Selection Theorem twice, there exists a subsequence $n$ such that $\{(\mu_0^n, \mu_1^n)\}_n$ converges to a weak cluster poinnt $(\mu_0^\star, \mu_1^\star)$. Because $\{\log U_{\mu_0^n}\}_n$ is uniformly bounded and continuously convergent to $\log U_{\mu_0^\star}$ and same for $\mu_1^n$ and $\mu_1^\star$, it holds that
\begin{align*}
& \int \log U_{\mu_1^\star}(u) d\mu_0^\star(u) + \int \log U_{\mu_0^\star}(u) d\mu_1^\star(u) \\
=  &\lim_{n \rightarrow \infty}  \int \log U_{\mu_1^n}(u) d\mu_0^n(u) + \int \log U_{\mu_0^n}(u) d\mu_1^n(u)).
\end{align*} 

and therefore $\mu_0^\star, \mu_1^\star$ achieve the infimum of~\eqref{eq:wangisola_opt}.

\item \textbf{$U_{\mu_1^\star}$ is constant $\mu_0^\star$-almost surely and $U_{\mu_0^\star}$ is constant $\mu_1^\star$-almost surely, for any minimizer $(\mu_0^\star, \mu_1^\star)$ of~\eqref{eq:wangisola_opt}.}

Formally, define $(\mu_0^\star, \mu_1^\star)$ to be a solution of~\eqref{eq:wangisola_opt}, i.e.
\begin{align*}
\mu_0^\star, \mu_1^\star \in \argmin{\mu_0, \mu_1}{\int \log U_{\mu_1}(u) d\mu_0(u) + \int \log U_{\mu_0}(u) d\mu_1(u)}.
\end{align*}

Define the Borel sets where $\mu_i^\star$ has positive measure to be $\mathcal{T}_i = \{T \in \mathcal{M}(\mathbb{S}^{d-1}): \mu_i^\star(T) > 0 \}$. Define the conditional distribution of $\mu_i^\star$ on $T$ for some $T \in \mathcal{T}_i$ as $\mu_{i, T}^\star$, where $\mu_{i, T}^\star(A) = \frac{\mu_i^\star(A \cap T)}{\mu_i^\star(T)}$.

Now we consider a mixture $(1 - \alpha)\mu_0^\star + \alpha \mu_{0, T}^\star$. The first variation of $\mu_0^\star$ states that 
\begin{align*}
0 &\texttt{=} \frac{\partial}{\partial \alpha} \bigg[   \int \log U_{\mu_1^\star}(u) d((1 - \alpha) \mu_0^\star + \alpha \mu_{0, T}^\star)(u) + \int \log U_{(1 - \alpha)\mu_0^\star + \alpha \mu_{0, T}^\star} d\mu_1^\star(u) \bigg|_{\alpha  = 0}\bigg] \\
&=  \int \log U_{\mu_1^\star}(u) d(\mu_{0, T}^\star - \mu_0^\star)(u) + \int \frac{U_{\mu_{0, T}^\star}(u) - U_{\mu_0^\star}(u)}{U_{\mu_0^\star}(u)} d\mu_1^\star(u),
\end{align*}

Where we've used the fact that $\frac{\partial}{\partial \alpha} U_{(1 - \alpha)\mu_0^\star + \alpha \mu_{0, T}^\star}(u)\big|_{\alpha = 0} = \frac{\partial}{\partial \alpha} \int \exp(u^\top v / \tau) d((1 - \alpha)\mu_0^\star + \alpha \mu_{0, T}^\star)(v) \big|_{\alpha = 0} = U_{\mu_{0, T}^\star}(u) - U_{\mu_0^\star}(u)$. Therefore, due to symmetry the optimality conditions using the first variation are
\begin{align}
\int \log U_{\mu_1^\star}(u) d(\mu_{0, T}^\star - \mu_0^\star)(u) +  \int \frac{U_{\mu_{0, T}^\star}(u)}{U_{\mu_0^\star}(u)} d\mu_1^\star(u) = 1 \label{eq:firstvar1} \\
\int \log U_{\mu_1^\star}(u) d(\mu_{1, T}^\star - \mu_1^\star)(u) +  \int \frac{U_{\mu_{1, T}^\star}(u)}{U_{\mu_1^\star}(u)} d\mu_0^\star(u) = 1 \label{eq:firstvar2}
\end{align}

Now, let $\{T_0^n \}_{n = 1}^{\infty}$ be a sequence of sets in $\mathcal{T}_0$ such that
\begin{align*}
\lim_{n \rightarrow \infty} \int U_{\mu_1^\star}(u) d \mu_{0, T_0^n}^\star (u) = \sup_{T_0 \in \mathcal{T}_0} \int U_{\mu_1^\star}(u) d \mu_{0, T_0}^\star (u) = U_{1, 0}^\star.
\end{align*}

and similarly let $\{T_1^n\}_{n = 1}^{\infty}$ be a sequence of sets in $\mathcal{T}_1$ such that
\begin{align*}
\lim_{n \rightarrow \infty} \int U_{\mu_0^\star}(u) d \mu_{1, T_1^n}^\star (u) = \sup_{T_1 \in \mathcal{T}_1} \int U_{\mu_0^\star}(u) d \mu_{1, T_1}^\star (u) = U_{0, 1}^\star.
\end{align*}

It holds that $\mu_0^\star (\{u: U_{\mu_1^\star}(u) \ge U_{1, 0}^\star \}) = 0$, $\mu_{0, T_0^n}^\star(\{u: U_{\mu_1^\star}(u) \ge U_{1, 0}^\star \}) = 0$ and similarly $\mu_1^\star(\{u: U_{\mu_0^\star}(u) \ge U_{0, 1}^\star \}) = 0$, $\mu_{1, T_1^n}^\star(\{u: U_{\mu_0^\star}(u) \ge U_{0, 1}^\star \}) = 0$. 

This implies that asymptotically $U_{\mu_0^\star}$ is constant $\mu_{1, T_1^n}^\star$-almost surely:
\begin{align*}
&\int \bigg| U_{\mu_0^\star}(u) - \int U_{\mu_0^\star}(u') d\mu_{1, T_1^n}(u') \bigg| d\mu_{1, T_1^n}^\star(u) \\
=  &2 \int \max\bigg(0, U_{\mu_0^\star}(u) - \int U_{\mu_0^\star}(u') d\mu_{1, T_1^n}^\star(u') \bigg) d\mu_{1, T_1^n}^\star(u) \\
\le & 2 \bigg(U_{0,1}^\star - \int U_{\mu_0^\star}(u) d\mu_{1, T_1^n}^\star(u)\bigg) \rightarrow 0.
\end{align*}

And the same holds that $U_{\mu_1^\star}$ is constant $\mu_{0, T_0^n}^\star$-almost surely. As a result, $\lim_{n \rightarrow \infty} \int \log U_{\mu_0^\star}(u) d\mu_{1, T_1^n}^\star(u) = \log U_{0, 1}^\star$ and $\lim_{n \rightarrow \infty} \int \log U_{\mu_1^\star}(u) d\mu_{0, T_0^n}^\star(u) = \log U_{1, 0}^\star$.

We now revisit~\eqref{eq:firstvar1} with a mixture over $\mu_0^\star$ and $\mu_{0, T_0^n}$:
\begin{align*}
1 &= \int \log U_{\mu_1^\star}(u) d(\mu_{0, T_0^n}^\star - \mu_0^\star)(u) + \int \frac{U_{\mu_{0, T_0^n}^\star}(u)}{U_{\mu_0^\star}(u)} d\mu_1^\star(u) \\
&\ge  \int \log U_{\mu_1^\star}(u) d(\mu_{0, T_0^n}^\star - \mu_0^\star)(u) + \frac{1}{U_{0,1}^\star} \int  U_{\mu_1^\star}(u) d \mu_{0, T_0^n}^\star(u).
\end{align*}

Taking the limit of both sides as $n\rightarrow \infty$, we get
\begin{align*}
1 &\ge  \log U_{1, 0}^\star -  \int \log U_{\mu_1^\star}(u) d\mu_0^\star(u) + \frac{1}{U_{0,1}^\star} U_{1, 0}^\star.
\end{align*}

and rearranging and doing the same to~\eqref{eq:firstvar2} yields
\begin{align*}
 \Big(1 - \frac{U_{1, 0}^\star}{U_{0, 1}^\star}\Big) &\ge \log U_{1, 0}^\star - \int \log U_{\mu_1^\star}(u) d\mu_0^\star(u) \\
\Big(1 - \frac{U_{0, 1}^\star}{U_{1, 0}^\star}\Big) &\ge \log U_{0, 1}^\star - \int \log U_{\mu_0^\star}(u) d\mu_1^\star(u)
\end{align*}

Note that Jensen's inequality and the definition of $U_{1, 0}^\star$ tell us that $\int \log U_{\mu_1^\star}(u) d\mu_0^\star(u) \le \log \int U_{\mu_1^\star}(u) d\mu_0^\star(u) \le \log U_{1, 0}^\star$, which means that $\Big(1 - \frac{U_{1, 0}^\star}{U_{0, 1}^\star}\Big) \ge 0$. However, applying the same logic also tells us that $ \Big(1 - \frac{U_{0, 1}^\star}{U_{1, 0}^\star}\Big) \ge 0$. The only case in which this is possible is when $U_{1,0}^\star = U_{0,1}^\star$, in which case equality is obtained. Therefore, this means that for optimal $\mu_0^\star, \mu_1^\star$, it holds that 
\begin{align}
\int \log U_{\mu_1^\star}(u) d\mu_0^\star(u) &= \log \int U_{\mu_1^\star}(u) d\mu_0^\star(u) \label{eq:jensen_thing1}\\
\int \log U_{\mu_0^\star}(u) d\mu_1^\star(u) &= \log \int U_{\mu_0^\star}(u) d\mu_1^\star(u) \label{eq:jensen_thing2}
\end{align}

\end{enumerate}

Using~\eqref{eq:jensen_thing1} and~\eqref{eq:jensen_thing2}, minimizing~\eqref{eq:wangisola_opt} is equivalent to minimizing
\begin{align*}
\log \int U_{\mu_1}(u) d \mu_0(u),
\end{align*}

where we use the fact that $\mu_0$ and $\mu_1$ are interchangable in the above expression. This expression can be written as $\log \E{x, x^-}{\exp(-\|f(x) - f(x^-)\|^2 / 2\tau)}$, which completes our proof.

\end{proof}

\begin{lemma}
Suppose $\F$ is a family of functions mapping from $\X$ to $\mathbb{R}^d$. Define $f(x)[j]$ as the $j$th element of $f(x)$ and suppose that for all $j$, $|f(x)[i]| \le b$. Define the element-wise class $\F_j = \{f(\cdot)[j] : f \in \F \}$.  Then, with probability at least $1 - \delta$ over $n$ i.i.d. samples $\{x_i\}_{i = 1}^n$, 
\begin{align*}
\bigg\|\E{}{f(x)} - \frac{1}{n} \sum_{i = 1}^n f(x_i) \bigg\|  \le 2 \rademacher_n(\F) + b d \sqrt{\frac{\log(d /\delta)}{2n}} \quad \forall f \in \F,
\end{align*}

where $\rademacher_n(\F) = \sum_{j = 1}^d \rademacher_n(\F_j)$.

\label{lemma:rademacher}

\end{lemma}

\begin{proof}
Using the triangle inequality,
\begin{align*}
\bigg\|\E{}{f(x)} - \frac{1}{n} \sum_{i = 1}^n f(x_i) \bigg\|  = \bigg(\sum_{j = 1}^d \Big(\E{}{f(x)[j]} - \frac{1}{n} \sum_{i = 1}^n f(x_i)[j]\Big)^2\bigg)^{1/2} \le \sum_{j = 1}^d \bigg|\E{}{f(x)[j]} - \frac{1}{n} \sum_{i = 1}^n f(x_i)[j] \bigg|.
\end{align*}

Using Theorem 3.3 of~\citet{mohri2018foundations}, we know that with probability at least $1 - \delta$,
\begin{align*}
\E{}{f(x)[j]} \le \frac{1}{n} \sum_{i = 1}^n f(x_i)[j] + 2\rademacher_n(\F_j) + b\sqrt{\frac{2\log(1 /\delta)}{n}}.
\end{align*}

Applying a union bound, we have that with probability at least $1 - \delta$, 
\begin{align*}
\bigg\|\E{}{f(x)} - \frac{1}{n} \sum_{i = 1}^n f(x_i) \bigg\| \le \sum_{j = 1}^d \bigg(2\rademacher_n(\F_j) + b\sqrt{\frac{\log(d /\delta)}{2n}} \bigg) = 2 \rademacher_n(\F) + b d \sqrt{\frac{2\log(d /\delta)}{n}}.
\end{align*}

\end{proof}

\section{Additional Experimental Details}
\label{sec:supp_details}
We describe details about the datasets, model architectures, and hyperparameters.

\subsection{Datasets}

We first describe all the datasets in more detail:
\begin{itemize}
\item
\textbf{CIFAR10}, \textbf{CIFAR100}, and \textbf{MNIST} are all the standard
computer vision datasets.

\item
\textbf{CIFAR10-Coarse} consists of two superclasses: animals (dog, cat, deer,
horse, frog, bird) and vehicles (car, truck, plane, boat).

\item
\textbf{CIFAR100-Coarse} consists of twenty superclasses.
We artificially imbalance subclasses to create \textbf{CIFAR100-Coarse-U}.
For each superclass, we select one subclass to keep all $500$ points, select one
subclass to subsample to $250$ points, select one subclass to subsample to $100$ points, and
select the remaining two to subsample to $50$ points.
We use the original CIFAR100 class index to select which subclasses to subsample:
the subclass with the lowest original class index keeps all $500$ points, the
next subclass keeps $250$ points, etc.

\item
\textbf{TinyImageNet-Coarse}~\citep{le2015tiny} consists of $67$ superclasses constructed from the ImageNet class hierarchy~\citep{bostock2018imagenet}.
The 67 superclasses are as follows: arachnid, armadillo, bear, bird, bug, butterfly, cat, 
coral, crocodile, crustacean, dinosaur, dog, echinoderms, ferret, fish, flower, frog, fruit, fungus, hog,
lizard, marine mammals, marsupial, mollusk, mongoose, monotreme, person, plant, primate, rabbit, rodent, 
salamander, shark, sloth, snake, trilobite, turtle, ungulate, vegetable, wild cat, wild dog, accessory, aircraft, ball, boat, 
building, clothing, container, cooking, decor, electronics, fence, food, furniture, hat, instrument, lab equipment, other, outdoor scene, 
paper, sports equipment, technology, tool, toy, train, vehicle and weapon.

\item
\textbf{MNIST-Coarse} consists of two superclasses: $<$5 and $\geq$5.

\item
\textbf{Waterbirds}~\citep{sagawa2019groupdro} is a robustness dataset designed
to evaluate the effects of spurious correlations on model performance.
The waterbirds dataset is constructed by cropping out birds from photos in the
Caltech-UCSD Birds dataset~\citep{WelinderEtal2010}, and pasting them on backgrounds
from the Places dataset~\citep{zhou2014learning}.
It consists of two categories: water birds and land birds.
The water birds are heavily correlated with water backgrounds and the land birds
with land backgrounds, but 5\% of the water birds are on land backgrounds, and
5\% of the land birds are on water backgrounds.
These form the (imbalanced) hidden strata.

\item
\textbf{ISIC} is a public skin cancer dataset for classifying skin
lesions~\citep{codella2019skin} as malignant or benign.
48\% of the benign images contain a colored patch, which form the hidden strata.

\item
\textbf{CelebA} is an image dataset commonly used as a robustness
benchmark~\citep{liu2015faceattributes,sagawa2019groupdro}.
The task is blonde/not blonde classification.
Only 6\% of blonde faces are male, which creates a rare stratum in the blonde
class.
\end{itemize}

\subsection{Model Architectures}
We use a ViT model~\citep{dosovitskiy2020image} (4 x 4 patch size, 7 multi-head attention layers with 8 attention heads and hidden MLP size of 256, final embedding size of 128) as the
encoder for the transfer learning experiments and a ResNet50 for the robustness experiments.
For the ViT models, we jointly optimize the contrastive loss with a cross-entropy 
loss head. For the ResNets, we train the contrastive loss on its own and use linear probing on the final layer.

For the autoencoder, we use the same encoder backbone as the main model, and use a ResNet18 in reverse order for the decoder.
The convolutions are replaced with resize convolutions.
We use the implementation in PyTorch Lightning Bolts\footnote{\url{https://github.com/PyTorchLightning/lightning-bolts/blob/master/pl_bolts/models/autoencoders/components.py}}~\citep{falcon2020framework}.

\subsection{Hyperparameters}

For the coarse dataset training, all models were trained for $600$ epochs with an initial learning rate of $0.0003$,
a cosine annealing learning rate scheduler with $T_{max}$ set to $100$ and the AdamW optimizer.
A dropout rate of $0.05$ was used. We did not use weight decay. For each coarse dataset, we trained $5$ separate models which jointly
optimize a cross-entropy loss head with either a contrastive loss
(InfoNCE, SupCon, SupCon + InfoNCE, SupCon + Class-conditional InfoNCE) or a reconstruction loss (mean squared error).

In the coarse-to-fine transfer experiments, we trained $5$ separate models for each of the configurations reported in Table~\ref{table:transfer} using
$5$ random seeds ($42$, $32$, $64$, $128$ and $72$). All models were trained
for $100$ epochs with an initial learning rate of $0.001$, a cosine annealing learning rate scheduler with $T_{max}$
set to $100$ and the AdamW optimizer. All transfer experiments were run using Tesla V100 machines.

All experiments were run using a batch size of $128$ for both training and evaluation.

\section{Additional Experimental Results}
\label{sec:supp_additional_results}

We present additional experimental results on end model accuracy, transfer with cross entropy, more datasets, more baselines, and full ablations.

\subsection{End Model Accuracy}

\begin{table}[t]
     \centering
     \begin{tabular}{lccclccc}
         \toprule                   & \multicolumn{3}{c}{\textbf{End Model Perf.}}  \\ \cmidrule(lr){2-4}
         \textbf{Dataset}           & InfoNCE        & $\supcon$       & $\lspread$  \\ \midrule
         \textbf{CIFAR10}           & 89.7           & 90.9           & \textbf{91.5} \\
         \textbf{CIFAR10-Coarse}    & 97.7           & 96.5           & \textbf{98.1} & \\
         \textbf{CIFAR100}          & 68.0           & 67.5           & \textbf{69.1} \\
         \textbf{CIFAR100-Coarse}   & 76.9           & 77.2           & \textbf{78.3} & \\
         \textbf{CIFAR100-Coarse-U} & 72.1           & 71.6           & \textbf{72.4} & \\
         \textbf{MNIST}             & 99.1           & \textbf{99.3}  & 99.2          \\
         \textbf{MNIST-Coarse}      & 99.1           & \textbf{99.4}  & \textbf{99.4} & \\
         \textbf{Waterbirds}        & 77.8           & 73.9           & \textbf{77.9} & \\
         \textbf{ISIC}              & 87.8           & 88.7           & \textbf{90.0}  &\\
         \bottomrule
     \end{tabular}
\caption{
    End model performance training with $\lspread$ on various datasets compared
    against contrastive baselines.
    All metrics are accuracy except for ISIC (AUROC).
    $\lspread$ produces the best performance in \num{7} out of \num{9}
    cases, and matches the best performance in \num{1} case.
}
\label{table:coarse-quality}
\end{table}

See Table~\ref{table:coarse-quality} for raw accuracy.
We confirm that using $\lspread$ instead of $\supcon$ does not degrade end model performance.

\subsection{Additional Transfer Results}

\begin{table*}[t]
  \centering
  \caption{Coarse-to-fine transfer learning performance (expanded table). Best in bold.}
  \small
  \begin{tabular}{llcccccccccc}
    \toprule
    & \textbf{Method} & \textbf{CIFAR10} & \textbf{CIFAR100} & \textbf{CIFAR100-U} & \textbf{MNIST} & \textbf{TinyImageNet} \\
    \midrule
    \multirow{4}{*}{\rotatebox[origin=c]{90}{\scriptsize{Baselines}}}
    & Cross Entropy            & 71.1 $\pm$ 0.2   & 54.2 $\pm$ 0.2    & 56.4 $\pm$ 0.4      & 98.7 $\pm$ 0.1 & 44.4 $\pm$ 0.1 \\
    & InfoNCE~\citep{chen2020simple} 
                      & 77.6 $\pm$ 0.1   & 60.5 $\pm$ 0.1    & 56.4 $\pm$ 0.3      & 98.4 $\pm$ 0.1  & 44.9 $\pm$ 0.1 \\
    & SupCon~\citep{khosla2020supervised} 
                      & 51.8 $\pm$ 1.2   & 56.1 $\pm$ 0.1    & 49.8 $\pm$ 0.3      & 95.4 $\pm$ 0.1  & 43.9 $\pm$ 0.1 \\
    & SupCon + InfoNCE~\citep{islam2021broad}
                      & 77.6 $\pm$ 0.1   & 55.7 $\pm$ 0.1    & 48.0 $\pm$ 0.2      & 98.6 $\pm$ 0.1  & 46.1 $\pm$ 0.1 \\
    \cmidrule(lr){2-7}
    \multirow{4}{*}{\rotatebox[origin=c]{90}{\scriptsize{Ours}}}
    & cAuto            & 71.4 $\pm$ 0.1   & 62.9 $\pm$ 0.1    & 58.7 $\pm$ 0.5      & 98.7 $\pm$ 0.1  & 47.1 $\pm$ 0.1 \\
    & SupCon + cNCE ($\lspread$)
                      & 77.1 $\pm$ 0.1   & 58.7 $\pm$ 0.2    & 53.5 $\pm$ 0.4      & 98.5 $\pm$ 0.1  & 45.8 $\pm$ 0.1 \\
    & SupCon + cAuto  & 71.7 $\pm$ 0.1   & 63.8 $\pm$ 0.6    & \textbf{59.8 $\pm$ 0.3} & 98.7 $\pm$ 0.1 & 49.3 $\pm$ 0.1  \\
    & SupCon + cNCE + cAuto (\textbf{\sysname})
                    & \textbf{79.1 $\pm$ 0.2}  & \textbf{65.0 $\pm$ 0.2} & 59.7 $\pm$ 0.3 & \textbf{99.0 $\pm$ 0.1} & \textbf{49.6 $\pm$ 0.1} \\                     
    \bottomrule
  \end{tabular}
  \label{table:transfer-with-cep}
\end{table*}

We reproduce Table~\ref{table:transfer} and additional report the performance of training with cross entropy loss (Table~\ref{table:transfer-with-cep}).

\subsection{Additional Datasets}

\begin{table*}[t]
    \centering
    \caption{Coarse-to-fine transfer learning performance on two additional datasets. Best in bold.}
    \small
    \begin{tabular}{lcccccccccc}
      \toprule
      \textbf{Method} & \textbf{Caltech-UCSD Birds} & \textbf{Stanford Dogs}  \\
      \midrule
      Cross Entropy   & 8.2   & 14.9    \\
      SupCon~\citep{khosla2020supervised} 
                      & 7.8   & 15.0     \\
      cAuto           & 8.8   & 17.7    \\
      SupCon + cNCE ($\lspread$)
                      & 7.5   & 16.5   \\
      SupCon + cAuto  & \textbf{9.1}   & 19.8   \\
      SupCon + cNCE + cAuto (\textbf{\sysname})
                      & 8.8   & \textbf{20.8} \\                     
      \bottomrule
    \end{tabular}
    \label{table:transfer-cub-dogs}
  \end{table*}
  
Table~\ref{table:transfer-cub-dogs} report the performance of C2F transfer on two additional datasets, Caltech-UCSD Birds~\citep{welinder2010caltech} and Stanford Dogs~\citep{khosla2011novel}.

\subsection{Additional Baselines}

\begin{table*}[t]
    \centering
    \caption{Coarse-to-fine transfer learning performance on CIFAR10 with two additional baselines. Best in bold.}
    \small
    \begin{tabular}{lcccccccccc}
      \toprule
      \textbf{Method} & \textbf{CIFAR10}   \\
      \midrule
      Clip Positives  & 35.3    \\
      Weighted Pos in Denom. 
                      & 56.3   \\
      \textbf{\sysname}
                      & \textbf{79.1} \\                     
      \bottomrule
    \end{tabular}
    \label{table:transfer-more-baselines}
  \end{table*}
  
Table~\ref{table:transfer-more-baselines} reports C2F transfer performance with two additional baselines---a) clipping the values of the positives in the numerator $\supcon$, and b) upweighting the negatives in the denominator of $\supcon$.
Both these methods underperform \sysname.

\subsection{Ablations and Sensitivity Studies}
\label{subsec:supp_ablations}
In this section, we validate our specific theoretical claims on the class-conditional autoencoder, data augmentation, and the Lipschitzness of the decoder.

\begin{table}[t]
    \centering
    \caption{Ablations on the autoencoder and data augmentation on CIFAR10 coarse-to-fine transfer.}
    \begin{tabular}{lccccccccccc}
      \toprule
      \multicolumn{2}{l}{\textbf{General vs. Class-Conditional Autoencoder}} \\
      \cmidrule(lr){1-2}
      gAuto            & 41.4 $\pm$ 0.2 \\
      cAuto            & 71.4 $\pm$ 0.1 \\
      \cmidrule(lr){1-2}
      SupCon           & 51.8 $\pm$ 1.2 \\
      SupCon + gAuto   & 55.4 $\pm$ 0.4 \\
      SupCon + cAuto   & 71.7 $\pm$ 0.1 \\
      \cmidrule(lr){1-2}
      SupCon + cNCE    & 77.1 $\pm$ 0.1 \\
      SupCon + cNCE + gAuto 
                       & 77.4 $\pm$ 0.1 \\
      SupCon + cNCE + cAuto (\textbf{\sysname})
                       & \textbf{79.1 $\pm$ 0.2} \\
      \midrule
      \multicolumn{2}{l}{\textbf{cNCE With and Without Augmentation}} \\
      \cmidrule(lr){1-2}
      SupCon + cNCE - augmentation
                      & 41.7 $\pm$ 0.2               \\
      SupCon + cNCE   & \textbf{77.1 $\pm$ 0.1}      \\               
      \bottomrule
    \end{tabular}
    \label{table:ablations}
\end{table}

We use two ablations to validate our claims that the class-conditional autoencoder outperforms a generic autoencoder, and that data augmentation in the class-conditional InfoNCE loss is critical for inducing subclass clustering.
Table~\ref{table:ablations} reports the results:
\begin{itemize}[itemsep=0.5pt,topsep=0pt,leftmargin=*]
    \item Lemma~\ref{lemma:autoencoder} claims that a class-conditional autoencoder should outperform a generic autoencoder in coarse-to-fine transfer.
    Indeed, we find that using a generic autoencoder underperforms a class-conditional autoencoder by \num{30.0} points on CIFAR10 coarse-to-fine transfer.
    Furthermore, the generic autoencoder does not improve performance of SupCon or its variants as well; we observe average lift of \num{2.0} points, compared to \num{11.0} points for the class-conditional autoencoder.

    \item Lemma~\ref{lemma:augmentation} claims that data augmentation in the class-conditional InfoNCE loss is key to break the permutation invariance.
    Removing data augmentation degrades performance by \num{35.4} points (and produces the permutation shown in Figure~\ref{fig:banner}).
\end{itemize}

\begin{figure}[t]
    \centering
    \includegraphics[width=0.45\textwidth]{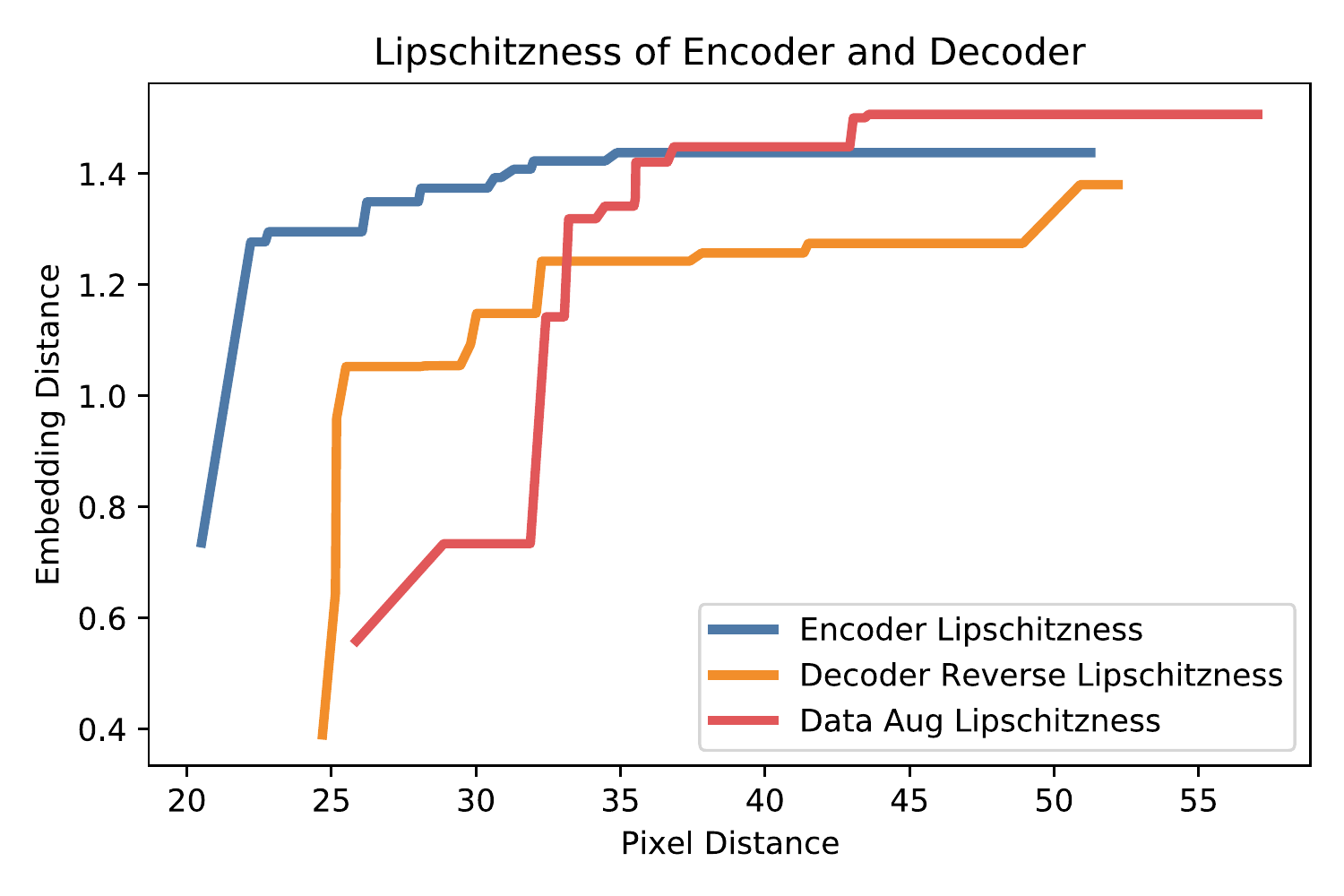}
    \caption{
        Measures of Lipschitzness for three ways to break permutation invariance.
        Encoder Lipschitzness reports pixel distance on the X axis and embedding distance on the Y axis.
        The Decoder reports pixel distance of the reconstruction.
        For augmentations, we run ten augmentations and pick the pair with the smallest ratio of pixel distance to embedding distance.
        The decoder is more Lipschitz than the encoder, and the encoder is more Lipschitz under augmentations than under traditional Lipschitzness.
    }
    \label{fig:lipschitzness}
\end{figure}

Finally, Figure~\ref{fig:lipschitzness} measures the Lipschitzness of an encoder trained with $\lspread$ and the reverse Lipschitzness of the decoder from a class-conditional autoencoder.
The encoder displays a high Lipschitzness constant (not very Lipschitz).
However, it displays a low Lipschitzness constant over augmentations.
The decoder displays a lower reverse Lipschitzness constant.
This suggests that the assumptions in Lemmas~\ref{lemma:autoencoder} and~\ref{lemma:augmentation} are reasonable.

To measure Lipschitzness of an encoder, we measure distance in embedding space of an encoder trained with $\lspread$ vs. distance in pixel space of two images (blue line).
To measure reverse Lipschitzness of a decoder, we make the same measurement, but over pixel distance of decoded images from an autoencoder (orange line).
To measure Lipschitzness under data augmentations, we measure the minimum ratio between embedding distance and pixel distance for 10 randomly-generated augmentations of two images (red line).
The Lipschitzness constants $K_L$, $K_g$, and $K_{aug}$ in Table~\ref{table:inducing_bias} are the slopes of the lines tangent to each of the curves in Figure~\ref{fig:lipschitzness} from the origin.

\section{Synthetic Experiments}
\label{sec:supp_synthetics}

We conduct synthetic experiments to understand the optimal geometry that minimizes the asymptotic loss $\lspread(\bm{\mu}, \alpha)$ as defined in Section~\ref{subsec:asymptotic}.

\paragraph{Setup} We minimize an empirical estimate of the asymptotic loss over a set of unit vectors $\{u_i\}_{i = 1}^{Kn_y}$. Denote $u$ as a unit vector, and denote $h(u)$ as its class label. The loss we minimize is 
\begin{align}
&(1 - \alpha) \frac{1}{K n_y} \sum_{i = 1}^{K n_y} \log \bigg(\frac{1}{(K-1)n_y} \quad \;\;\;\;\; \sum_{\mathclap{j: h(u_i) \neq h(u_j)}} \;\; \exp(-\|u_i - u_j \|^2 / 2\tau)\bigg) \\
&+ \alpha \cdot \frac{1}{K n_y} \sum_{i = 1}^{K n_y} \log \bigg(\frac{1}{n_y} \quad \;\;\;\;\; \sum_{\mathclap{j: h(u_i) = h(u_j)}} \;\; \exp(-\|u_i - u_j \|^2 / 2\tau)\bigg) + (1 - \alpha) \frac{1}{Kn_y^2} \sum_{h(u) = h(u')} \|u - u' \|^2 / 2\tau.
\end{align}

We use scipy.minimize and the Sequential Least Squares Programming (SLSQP) option. We report the set of vectors that obtain the lowest loss over $5$ runs with random initializations (seeds $0-4$) as the optimal geometry.

We compute an empirical estimate of $s_f(y)$ as $\frac{1}{n_y} \sum_{u: h(u) = y} \Big\|u - \frac{1}{n_y} \sum_{u': h(u') = y} u' \Big\|$, and average over all classes.

\begin{figure}
\centering
\includegraphics[width=3in]{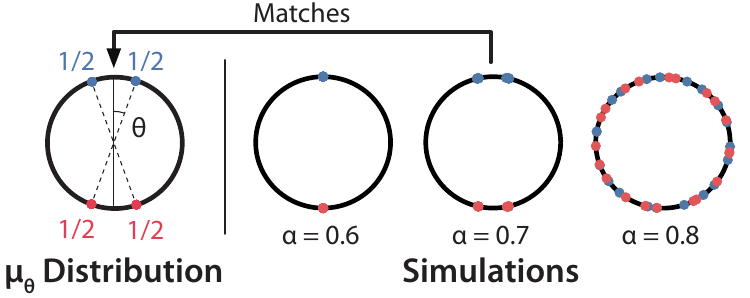}
\caption{Left: Distribution family $\bm{\mu_{\theta}}$. Right: Simulations of optimal geometry for binary setting on $S^1$.}
\label{fig:synthetic_fig}
\end{figure}

\paragraph{Matching $\bm{\mu}$ and the optimal geometry}

Figure~\ref{fig:synthetic_fig} displays our constructed distribution $\bm{\mu}_\theta$ as well as simulations on $S^1$ for $K = 2$. In particular, the right figure consists of the optimal geometry for $n_y = 20$, $\tau = 0.5$. We see that for $\alpha = 0.6$ (which is below our threshold in Theorem~\ref{thm:geometry}), that the optimal geometry is collapsed. For $\alpha = 0.7$, the optimal geometry appears to closely match the parametrization of $\bm{\mu_{\theta}}$. For $\alpha = 0.8$, which is above the theoretical threshold, the optimal geometry is uniform per class.

\begin{figure}[h]
\centering
\includegraphics[scale=0.5]{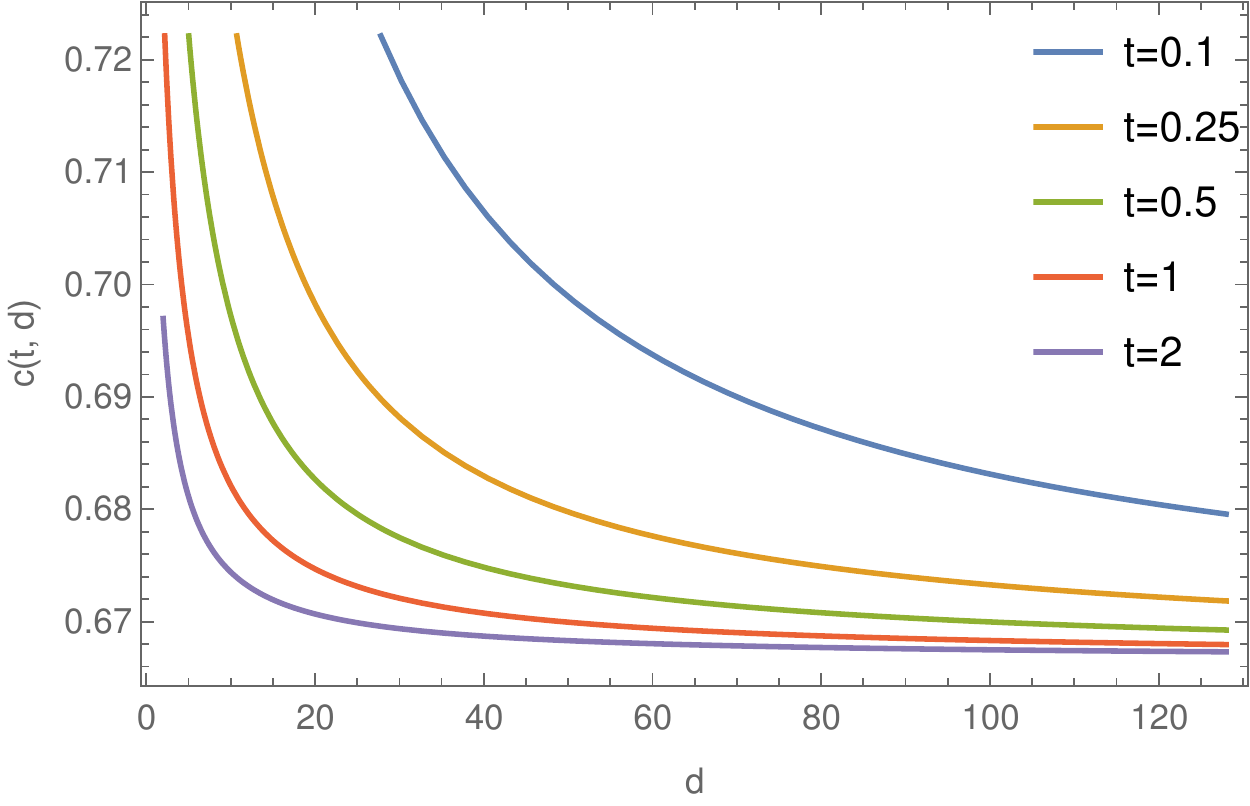}
\caption{The value of $c(\tau, d)$, which determines the range of $\alpha$ for which $\bm{\mu_{\theta}}$ obtains lower loss than $\bm{\delta_v}$ and $\bm{\sigma_{d-1}}$.}
\label{fig:alpha_ub}
\end{figure}

\paragraph{Computing $c(\tau, d)$} Next, we compute the value of $c(\tau, d)$, the constant in Theorem~\ref{thm:geometry}, over values of $\tau$ and $d$ to verify that $\alpha \in (2/3, c(\tau, d)$ is a valid range for which the optimal geometry is neither collapsed nor uniform. This quantity depends on the Wiener constant of the Gaussian $\frac{1}{2\tau}$-energy on $\mathcal{S}^{d-1}$, which does not have a closed form expression (see~\eqref{eq:wiener}). Figure~\ref{fig:alpha_ub} shows that $c(\tau, d) > 2/3$ for $d$ up to $128$ (which is the dimension of our embedding space) and for $\tau = 0.1, 0.25, 0.5, 1, 2$.

\paragraph{Varying $\tau$ and $d$ for $K = 2$}

\begin{figure}[h]
\centering
\includegraphics[scale=0.5]{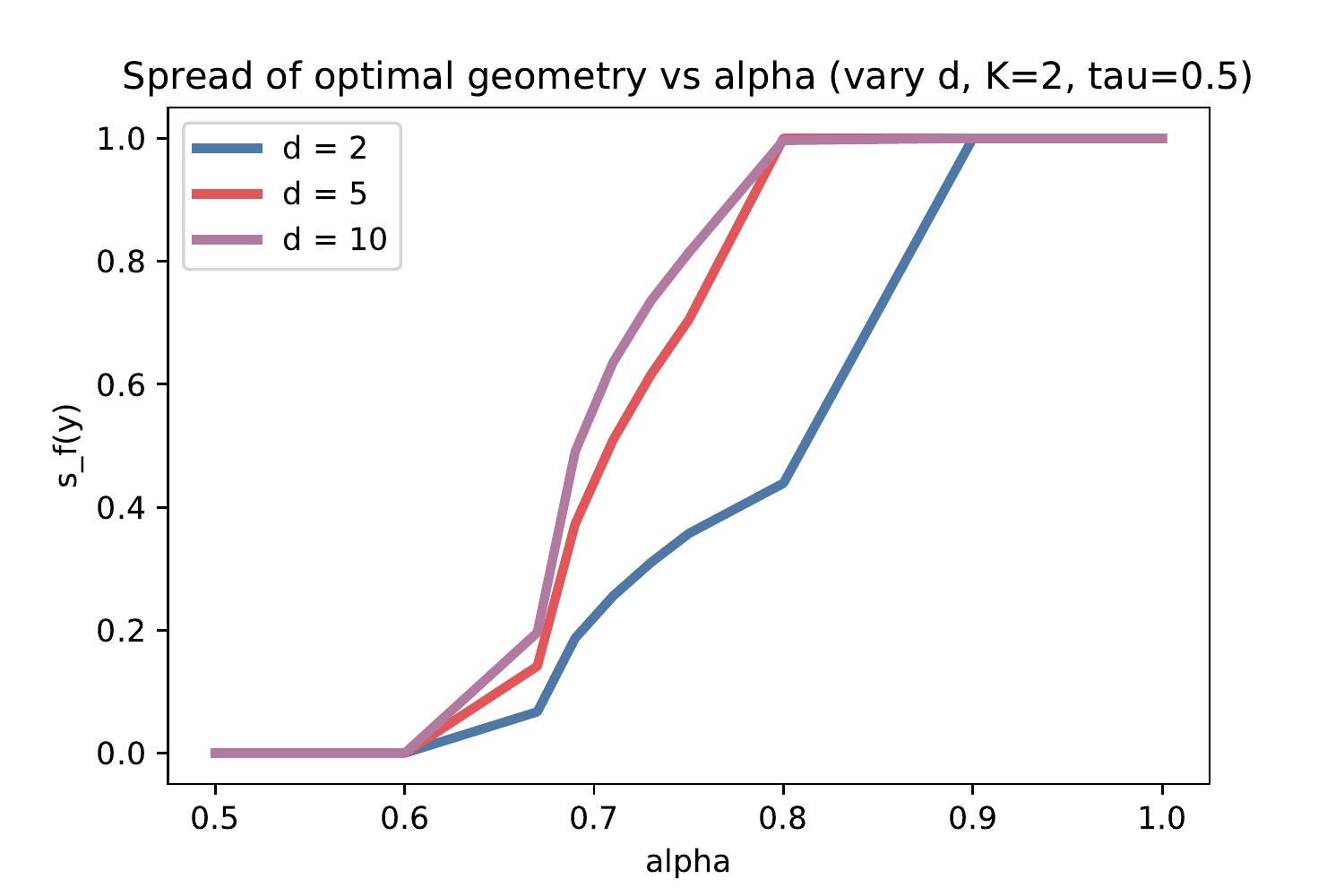}
\includegraphics[scale=0.5]{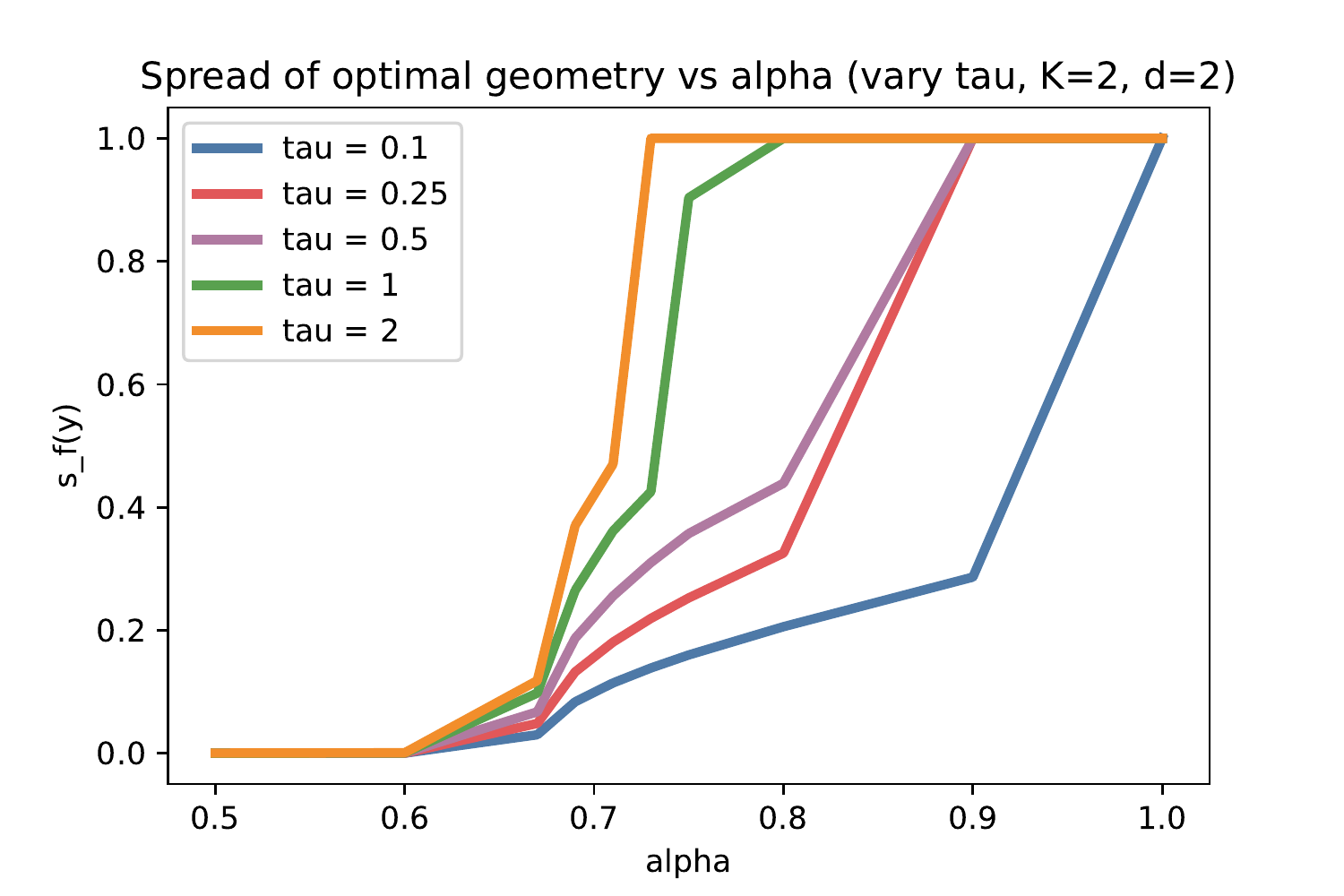}
\caption{The spread $s_f(y)$ of the optimal geometry for a given $\alpha$ in the binary setting, $n_y = 8$. Left: how optimal spread changes based on dimension $d$ of the embedding space. Right: how optimal spread changes based on the temperature hyperparameter $\tau$.}
\label{fig:k2}
\end{figure}

We plot $\alpha$ versus spread $s_f(y)$ for the optimal geometry in the binary setting, and in particular we vary $\tau$ and $d$ in Figure~\ref{fig:k2}. We compute the optimal geometries over $\alpha = 0.5, 0.6, 0.67, 0.69, 0.71, 0.73, 0.75,$ $0.8, 0.9$. Figure~\ref{fig:k2} left shows how the spread changes as $\alpha$ increases for dimensions $d = 2, 5, 10$ and $n_y = 8$ samples with $\tau = 0.25$, and the right shows how the spread changes as $\alpha$ increases for $\tau = 0.1, 0.25, 0.5, 1, 2$ with $d = 2$.
Note that for all dimensions and all $\tau$, the optimal geometry has nonzero spread starting at $\alpha = 0.67$, matching our theoretical findings. The point at which the uniform distribution becomes optimal is less clear but follows the trends we note from Figure~\ref{fig:alpha_ub} for $c(\tau, d)$. This figure matches our findings that an $\alpha$ that induces appropriate spread exists over a certain range, outside of which behavior is strictly collapsed or uniform.

\paragraph{Multiclass Analysis for $K = 3$}

\begin{figure}[h]
\centering
\includegraphics[scale=0.5]{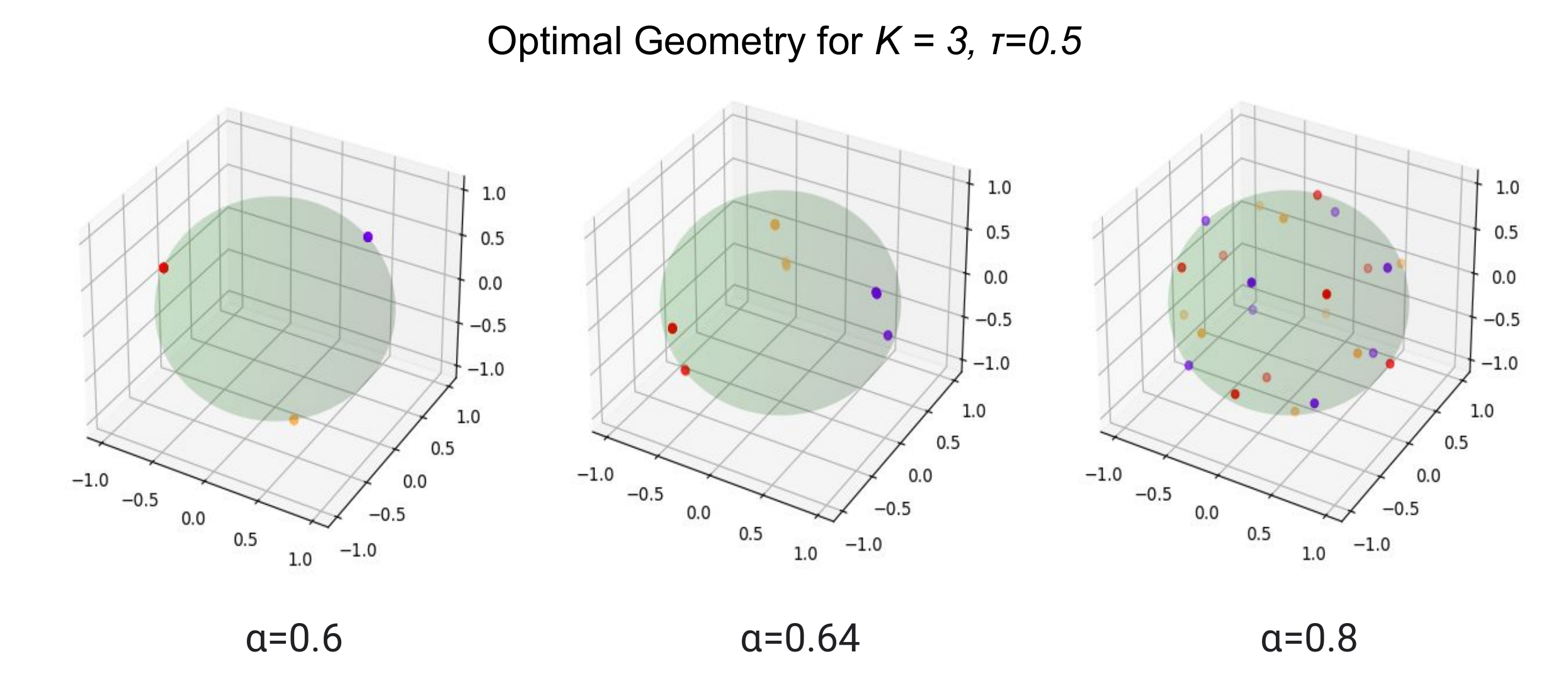}
\caption{Visualizations of the optimal geometry for $n_y = 8, K = 3$ in $\mathcal{S}^2$ across various $\alpha$.}
\label{fig:sphere}
\end{figure}

Similar to Figure~\ref{fig:synthetic_fig}, we show for $K = 3$ and $\mathcal{S}^2$ that the optimal geometry is collapsed for low $\alpha$, sufficiently spread for a particular range, and uniform for high $\alpha$.  Figure~\ref{fig:sphere} displays the optimal geometry for $n_y = 8$, $\tau = 0.5$ across $\alpha = 0.6, 0.64, 0.8$, suggesting that the multiclass case exhibits similar behavior as $\alpha$ varies.

\begin{figure}[h]
\centering
\includegraphics[scale=0.5]{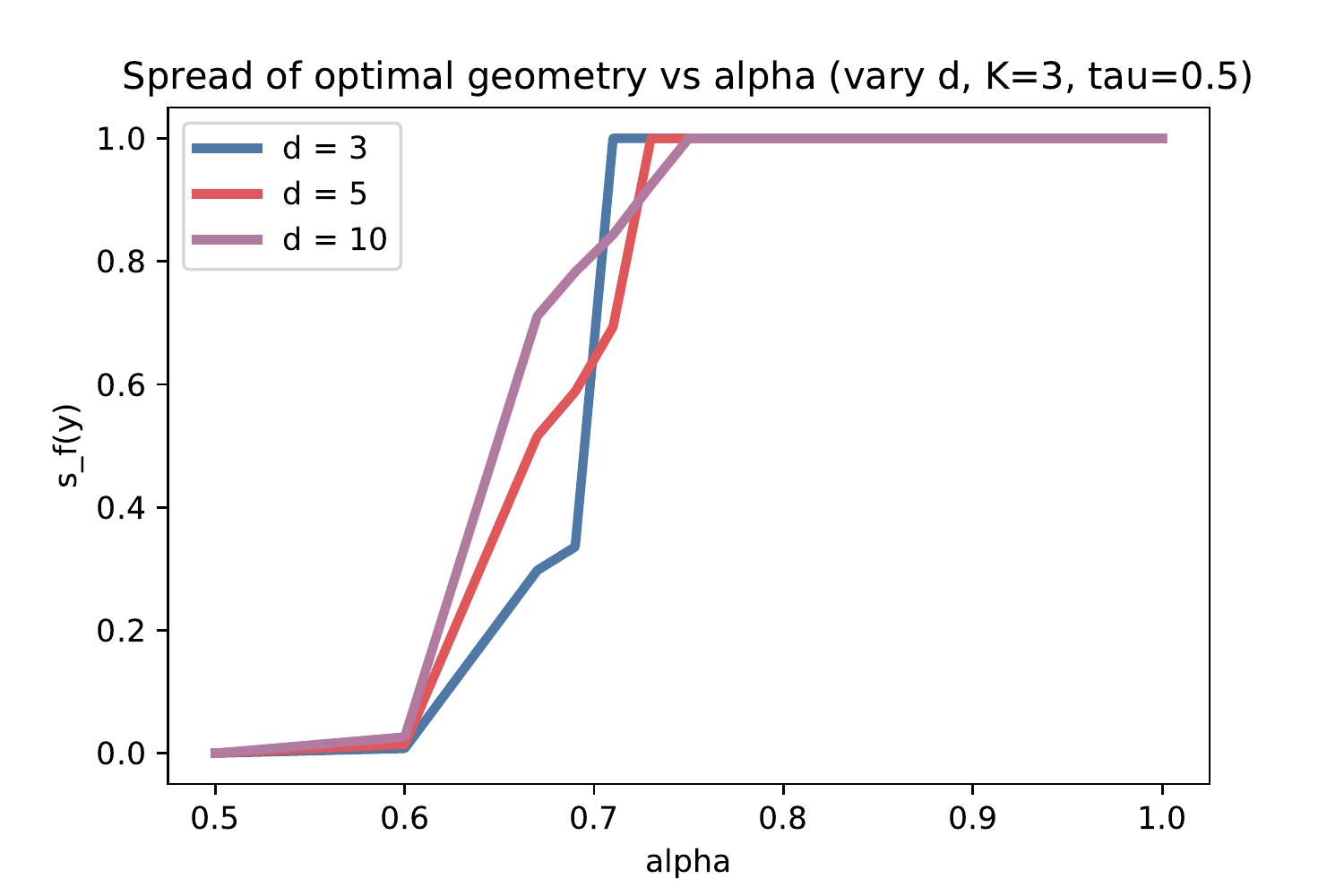}
\includegraphics[scale=0.5]{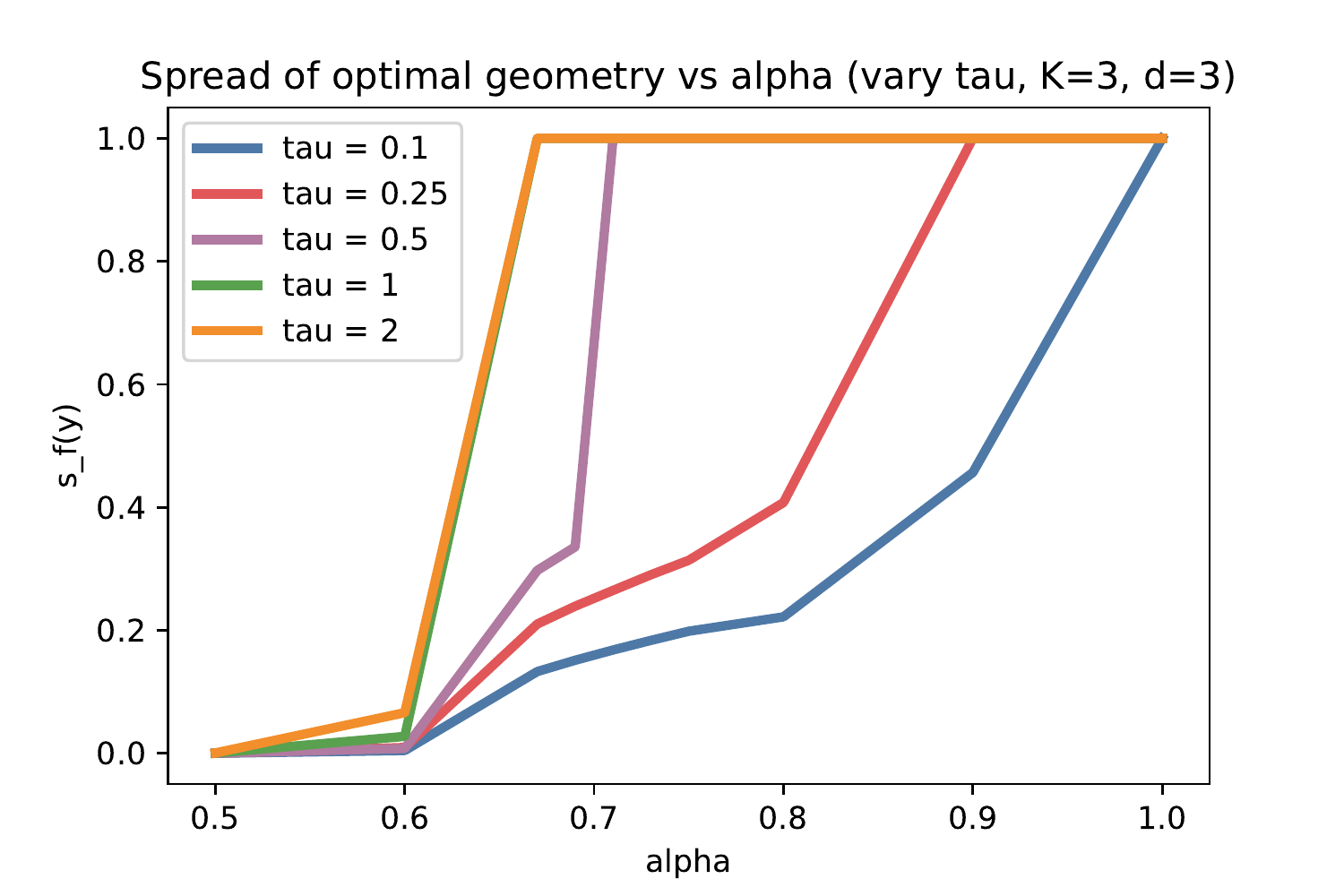}
\caption{The spread $s_f(y)$ of the optimal geometry for a given $\alpha$ when $K = 3, n_y = 8$. Left: how optimal spread changes based on dimension $d$ of the embedding space. Right: how optimal spread changes based on the temperature hyperparameter $\tau$.}
\label{fig:k3}
\end{figure}

We plot $\alpha$ versus spread $s_f(y)$ for the optimal geometry in the multiclass setting, and again we vary $\tau$ and $d$ in Figure~\ref{fig:k3}. We see that the behavior of the optimal geometry's spread $s_f(y)$ across $\alpha$ is roughly similar to that of $K = 2$ in Figure~\ref{fig:k2}.

\end{document}